\documentclass[letterpaper]{article} 

\usepackage{aaai2027}
\nocopyright

\usepackage[hyphens]{url}  
\usepackage{graphicx} 
\usepackage{natbib}  
\usepackage{caption} 
\usepackage{newfloat}
\usepackage{listings}
\DeclareCaptionStyle{ruled}{labelfont=normalfont,labelsep=colon,strut=off} 
\usepackage{booktabs}
\usepackage{array}
\title{
Optimized Certainty Equivalent Risk Minimization Using Samples: \\Algorithms, Convergence Rates, and Applications
}
\author{
    Sumedh Gupte \textsuperscript{\rm 1,\rm 2},
    Prashanth L. A. \textsuperscript{\rm 2},
    Sanjay P. Bhat \textsuperscript{\rm 1}
}
\affiliations{
    \textsuperscript{\rm 1}TCS Research, IIIT-H Research Park, Hyderabad, 500032, India\\
    \textsuperscript{\rm 2}Department of Computer Science and Engineering, Indian Institute of Technology Madras, Chennai, 600036, India\\
    sumedh.gupte@tcs.com, prashla@cse.iitm.ac.in, sanjay.bhat@tcs.com
}

\usepackage{amsmath}
\usepackage{amssymb}
\usepackage{amsthm}
\usepackage{relsize}
\usepackage{comment}
\usepackage[ruled,vlined,noend]{algorithm2e}
\usepackage{graphicx}
\usepackage{caption}
\usepackage{subcaption}
\newtheorem{remark}{Remark}
\newtheorem{theorem}{Theorem}
\newtheorem{lemma}[theorem]{Lemma}

\newtheorem{proposition}{Proposition}
\newtheorem{assumption}{Assumption}
\newtheorem{definition}{Definition}
\newtheorem{example}{Example}

\usepackage{multirow}
\usepackage{pifont}
\newcommand{\cmark}{\ding{51}}%
\newcommand{\xmark}{\ding{55}}%

\usepackage{todonotes}
\usepackage{xspace}

\newcommand{\Leb}[1]{\mathit{L}_{#1}}
\newcommand{\Rel}{\mathbb{R}}
\newcommand{\N}{\mathbb{N}}
\newcommand{\Exp}{\mathbb{E}}

\newcommand{\xu}{\mathcal{X}_u}
\newcommand{\xhu}{\mathcal{\bar{X}}_u}

\newcommand{\B}{\mathcal{T}}
\newcommand{\Z}{\mathbf{Z}}

\newcommand\norm[1]{\left\lVert#1\right\rVert} 
\def\argmin{\mathop{\rm arg\,min}}

\newcommand\numberthis{\addtocounter{equation}{1}\tag{\theequation}}
\newcommand{\order}[1]{$\mathcal{O}\left(#1\right)$}
\newcommand{\horder}[1]{$\tilde{\mathcal{O}}\left(#1\right)$}
\DeclareMathOperator{\srt}{SR}
\newcommand{\sr}[3]{\srt_{#1,#2}(#3)}
\newcommand{\ubsr}{\srt_{l,\lambda}(X)}
\newcommand{\srm}[1]{\srt_{m}(#1)}

\newcommand{\srth}[1]{\srt(\theta_{#1})}
\DeclareMathOperator{\ocet}{OCE}
\DeclareMathSymbol{\shortminus}{\mathbin}{AMSa}{"39}
\newcommand{\oce}{\ocet_{u}(X)}
\newcommand{\oc}[1]{\ocet_{u}(#1)}
\newcommand{\ocm}[1]{\ocet_{m}(#1)}
\newcommand{\octh}[1]{\ocet(\theta_{#1})}

\usepackage{cleveref}
\Crefname{checksubsection}{Section}{Sections}
\crefformat{checksubsection}{Section~#2#1#3}
\Crefformat{checksubsection}{Section~#2#1#3}

\Crefname{checksection}{Section}{Sections}
\crefformat{checksection}{Section~#2#1#3}
\Crefformat{checksection}{Section~#2#1#3}

\crefname{lemma}{lemma}{lemmas}
\crefname{figure}{figure}{figures}
\crefname{assumption}{assumption}{assumptions}
\begin{document}

\maketitle

\begin{abstract}
We consider the optimization of the Optimized Certainty Equivalent (OCE) risk, with applications including portfolio optimization in finance, and uncertainty quantification, classification, and regression in machine learning. Our contributions cover popular special cases of OCE, such as entropic risk, mean-variance risk, and smooth variants of Conditional Value-at-Risk. Our treatment sets out the conditions that facilitate the extension of OCE to unbounded r.v.s.. We provide a useful characterization of OCE that links OCE to utility-based shortfall risk (UBSR). Our characterization enables us to form an OCE estimator from the classic sample-average approximation (SAA) of UBSR. We derive mean-squared error (MSE) bounds for our proposed OCE estimator. For OCE optimization, we first derive an expression for the OCE gradient using the characterization linking OCE to UBSR. This expression serves as the basis for a gradient estimator for the OCE. We derive non-asymptotic bounds on the MSE for the proposed OCE gradient estimator. We incorporate the aforementioned gradient estimator into a stochastic gradient (SG) algorithm to optimize OCE and quantify its convergence rate using non-asymptotic bounds that we derive. Finally, we present three experiments that use our OCE optimization algorithm to solve portfolio optimization and uncertainty quantification problems. 
\end{abstract}


\section{Introduction}
\label{sec:intro}
Accurate risk quantification is of prime importance in applications where assessment and mitigation of risk are critical both from a business perspective as well as a regulatory perspective. Examples of such application areas include the finance, healthcare, and insurance industries. A \textit{risk measure} provides a quantitative assessment of a risky financial position formulated as a random variable (r.v.). A financially relevant risk measure should align with human perception and intuitive understanding of risk. For instance, when a financial position is evaluated purely based on its expected value, it downplays the risk associated with extreme but rare events. While the Value at Risk (VaR) \citep{jorion1997value,basak-shapiro-var} highlights the risk from tail events, it fails to be sub-additive \citep{artzner-coherent-risk-measures} and thus militates against the intuition that that diversification should reduce risk. These considerations motivated a formal look at the properties that a financial risk measure should possess \citep{ACERBI20021487} and led to the notion of a convex risk measure \citep{FollmerSchied2004}. A convex risk measure possesses the financially relevant properties of monotonicity (higher losses represent greater risks), cash equivariance (more cash in hand means less risk) and subadditivity (diversification cannot increase risk). 
\begin{table*}[t]
\centering
\caption{Comparison of related works w.r.t. OCE estimation and optimization. Here MSE, MAE denote mean-squared and mean-absolute errors, and SC, C, NC denotes strongly-convex, convex and non-convex objectives, respectively. OCE estimation includes Lipschitz utility in CVaR (see \cref{example:cvar}) as well as several non-Lipschitz utilites (see \cref{example:entropic,example:mean-variance,example:monotone-mean-variance,example:quartic}).
    OCE optimization covers the linear case in portfolio optimization (see \cref{example:portfolio-optimization}) as well as non-linear cases in  
\cref{example:classification,example:regression,example:uncertainty-quantification}. 
}
    \label{tab:summary_main}
    \begin{subtable}[b]{0.38\textwidth}
        \centering
        \caption{Comparison of OCE risk estimation guarantees.}
        \label{tab:summary_a}
        \begin{tabular}{|c|c|c|}
            \hline
            \multirow{2}{*}{\textbf{Reference}} & \textbf{Bound}  & \textbf{Non-Lipschitz} \\ 
            & \textbf{type}  & \textbf{utility} \\ \hline
            L.A. et al. 2022 & MAE & \xmark \\ \hline
            Ghosh et al. 2024 & MSE & \xmark \\ \hline
            Hamm et al. 2013 & Asymptotic & \xmark \\ \hline
            Our work & MSE, MAE & \cmark \\ \hline
        \end{tabular}
    \end{subtable}
    \hfill
    \begin{subtable}[b]{0.57\textwidth}
        \centering
        \caption{Comparison of OCE risk optimization guarantees.}
        \label{tab:summary_b}
        \begin{tabular}{|c|c|c|c|}
            \hline
            \multirow{2}{*}{\textbf{Reference}} & \textbf{Risk} & \textbf{Guarantee} & \multirow{2}{*}{\textbf{Objective}} \\ 
            & \textbf{measure} & \textbf{type} &  \\ \hline
            Natarajan et al. 2010 & OCE & Empirical & Linear \\ \hline
            Shapiro et al. 2021 & Convex risk & Asymptotic & C \\ \hline
            Tamtalini et al. 2022  & OCE & Asymptotic & Linear \\ \hline
            Our work & OCE & Non-asymptotic & SC/C/NC \\ \hline
        \end{tabular}
    \end{subtable}
\end{table*}

Convex risk measures also share a deep connection with the framework of \textit{Distributionally Robust Optimization} (DRO) \cite{kuhn2025distributionally}, which has received increased research attention in the operations research and the machine learning (ML) communities. In the context of ML, DRO provides a seemingly better alternative to the classical empirical risk minimization (ERM) framework. Since optimizing a convex risk measure is equivalent to solving a DRO problem \cite{rahimian2022distributionally}, developing risk minimization algorithms is meaningful and warrants research attention. 

OCE is a class of convex risk measures that generalizes CVaR and includes several popular risk measures, such as entropic risk, monotone mean-variance, and quartic risk, as special cases. The introduction of OCE in literature \citep{ben-tal-1986}, however, predates the emergence of risk measures \citep{artzner-coherent-risk-measures} and associated properties like convexity or coherence. OCE is associated with the idea of a \textit{preference order} 
 ($\succeq$) that is commonly used in the expected utility theory \citep{VonNeumann+Morgenstern:1944}, where for a utility function $u$ and any two r.v.s. $X,Y$, we have $X \succeq Y$ ($X$ is preferred over $Y$) if and only if $\Exp\left[u(X)\right] \geq \Exp\left[u(Y)\right]$. In a financial application, $X$ and $Y$ could denote the random returns associated with two different investment strategies. A \textit{certainty equivalent}, say $C(X)$, for a decision maker is a sure amount that is equivalent to the uncertain quantity $X$, and imposes the following \textit{preference order}: $X \succeq Y$ if and only if $C(X) \geq C(Y)$. OCE risk measure \citep{ben-tal-1986} is a type of \textit{certainty equivalent} that is based on utility functions. OCE was later reintroduced as a convex risk measure by \citet{ben-tal_old-new_2007}.
 OCE risk measures are also connected to the information-theoretic concept of $\phi$-divergence, and the reader is referred to \citet{ben-tal_old-new_2007} and Section 4.9 of \citet{follmer2016stochastic} for a precise statement of the aforementioned connection.


\paragraph{Our contributions.}
We consider OCE estimation and optimization under uncertainty, i.e., using samples from the underlying model, which is not known. 
Our contributions cover popular special cases of OCE, such as entropic risk, mean-variance risk, and two smooth variants of Conditional Value-at-Risk that we propose. \Cref{tab:summary_main} provides a summary comparison of our contributions against closely related works.

We first extend the OCE formalism to unbounded r.v.s. while imposing an integrability condition. For this class, we establish the convexity of OCE.

Our main contribution is to solve the OCE optimization problem under a smooth parameterization of the underlying r.v. We adopt a gradient-based approach and derive an expression for the OCE gradient. Next, we form a gradient estimate using samples, and establish MSE and MAE bounds for this estimator. Finally, we propose a SG algorithm for OCE optimization and provide non-asymptotic bounds that quantify the convergence rate of this algorithm. We derive these bounds for three cases: when the underlying parameterization yields a strongly convex, a convex, or a non-convex OCE objective.

As a secondary contribution, we study the problem of OCE estimation from independent and identically distributed (i.i.d.) samples using an SAA approach. In particular, we exploit the connection between OCE and UBSR to derive an estimator of OCE from UBSR. For this estimator, we derive bounds on the MSE and MAE under a moment assumption on the underlying r.v. These bounds for OCE estimation are of independent interest.

Finally, we conduct simulation experiments on OCE minimization for the following settings: classification, portfolio optimization, and uncertainty quantification. Additional experiments on CVaR estimation and optimization, and entropic risk estimation and optimization are available in \Cref{sec:suppl-experiments} of the supplementary material.

\paragraph{Related work.}

 The OCE measure was first introduced by \citet{ben-tal-1986} as a decision-making criterion. \citet{ben-tal_old-new_2007} reformulated the OCE criterion and presented it as a risk measure. In particular, they derived useful properties, such as convexity and coherence, under the assumption that the r.v.s. are bounded and the utility function is sublinear. \citet{hamm_2013_stochastic_root-finding_oce} provided a stochastic approximation scheme for OCE estimation, with an asymptotic convergence guarantee. \citet{tamtalini_multivariate_2022} analyzed a multivariate form of OCE and proposed a stochastic approximation scheme for OCE estimation, wherein they showed asymptotic convergence and asymptotic normality of their estimator. \citet{JMLR-LAP-SPB,ghosh} studied an SAA scheme for OCE estimation; the former gave MAE bounds for a Lipschitz utility function, while the latter gave MSE bounds for a smooth and strongly-convex utility function. In comparison to these works, we would like to the note the following aspects: i) Unlike \cite{hamm_2013_stochastic_root-finding_oce}, our results for OCE estimation/optimization apply to possibly unbounded r.v.s.;  ii) Unlike \citet{ben-tal_old-new_2007}, we neither assume that the utility function is sub-linear, nor assume that $u(0) =0$ holds; iii) Unlike \citet{shapiro_book,tamtalini_multivariate_2022,hamm_2013_stochastic_root-finding_oce}, we provide non-asymptotic error bounds on the proposed SAA estimator; iv) Unlike \citet{JMLR-LAP-SPB}, we provide MSE bounds, and our bounds allow for utility functions that may not be Lipschitz; v) Unlike \citet{ghosh}, our MSE bounds do not require a strongly-convex utility function.

Previous works on OCE risk optimization focus specifically on portfolio optimization and are limited to a few specific utility functions. These works use convex optimization routines and are either empirical \cite{natarajan2010tractable} or provide only asymptotic guarantees \cite{tamtalini_multivariate_2022,shapiro_book}. To the best of our knowledge, non-asymptotic bounds for OCE optimization using a stochastic gradient scheme are not available in the literature. 

\paragraph{Preliminaries}
We use boldface ($\mathbf{v}$), uppercase ($X$), and a combination of boldface and uppercase ($\mathbf{Z}$) to denote vectors, random variables, and random vectors, respectively. We use 'Var' as an abbreviation for variance. The terms $x^+$ and $x^-$ indicate $\max\left\{x,0\right\}$ and $\max\left\{-x,0\right\}$, respectively. For $p\in[1,\infty)$, $\norm{\mathbf{v}}_p$ denotes the vector $p$-norm. 


Let $(\Omega, \mathcal{F}, \mathit{P})$ be a probability space, $\Leb{0}$ denote the space of $\mathcal{F}$-measurable, real r.v.s. on $\Omega$, and $\Exp(\cdot)$ denote the expectation under $\mathit{P}$. For $p \in [1,\infty)$, $\big(\Leb{p}, \norm{\cdot}_{\Leb{p}}\big)$ denotes the normed vector space of r.v.s. with finite $p^{th}$ moment. Let $\mathbf{Z}$ be a random vector such that each $Z_i$ is $\mathcal{F}$-measurable and has finite $p^{th}$ moment. Then the $\Leb{p}$-norm of $\Z$ is defined by $\norm{\mathbf{Z}}_{\Leb{p}} \triangleq \left( \Exp \Big[ \norm{\mathbf{Z}}_p^p\Big] \right)^\frac{1}{p}$. 
Let $\mu_X$ and $\mu_Y$ denote the marginal distributions of r.v.s. $X$ and $Y$, respectively. Let $\mathcal{H}(\mu_{{X}},\mu_{{Y}})$ denote the set of all joint distributions having $\mu_{{X}}$ and $\mu_{{Y}}$ as the marginals. Then, for every $p\geq 1,\mathcal{W}_p(\mu_{{X}},\mu_{{Y}}) \triangleq \inf \left(\left\{ \int \norm{x-y}^p \eta(dx,dy) : \eta \in \mathcal{H}(\mu_{{X}},\mu_{{Y}}) \right\}\right)^{1/p}$ denotes the $p^\textrm{th}$ Wasserstein distance \citep{panaretos_invitation_2020}. 

Given a real-valued function $f:\Rel \to \Rel$, $\mathcal{X}_f \subseteq \Leb{0}$ denotes the space of r.v.s. $X$ for which $f(X-t)$ is integrable for every $t \in \Rel$. The risk measures that we consider in this paper are well-defined when $X \in \mathcal{X}_f$, i.e., when $\Exp\left[f(X-t)\right]$ is finite for all $t \in \Rel$. When the random variable $X$ is unbounded, the finiteness of the above expectation not only depends on $X$, but also on $f$. This dependency motivates the use of the $\mathcal{X}_f$ notation above. For the remainder of the paper, the r.v. $X$ denotes losses; therefore, a lower value is more preferable. 

\section{OCE Risk for Unbounded Random Variables}
\label{sec:pb}
In this section, we characterize the OCE risk measure using an expression that relates OCE risk to UBSR. In addition, our analysis covers a class of unbounded r.v.s and shows that properties such as convexity continue to hold. In the financial domain, if a risk measure is convex, then diversification cannot increase this risk. We define convex risk measures below. 
\begin{definition}
A  risk measure $\rho: \mathcal{X} \to \Rel$ is convex, if $\mathcal{X}$ is convex and for every $X_1,X_2 \in \mathcal{X}$ and $\alpha \in [0,1]$, we have $ \rho(\alpha X_1 + (1-\alpha)X_2) \leq \alpha \rho(X_1) + (1-\alpha)  \rho(X_2)$.
\end{definition}

OCE is a convex risk measure \citep{ben-tal_old-new_2007},
and is formally defined below.
\begin{definition}\label{def:oce}
Let $u:\Rel \to \Rel$ be a convex and increasing function. Let the r.v. $X \in \xu$. Then the OCE risk of $X$ under the utility function $u$ is defined as follows:
\begin{equation*}
    \oce \triangleq \inf_{ t \in \Rel}\{ t + \Exp\left[u(X-t)\right] \}.
\end{equation*}
\end{definition}

\subsection{Popular OCE Examples}
\label{subsection:oce-examples}

\begin{example}[Entropic risk]\label{example:entropic}
    Let $\beta>0$ and define the utility function as $u(x) = \beta^{-1}\left(e^{\beta x}-1\right), \forall x \in \Rel$. Then, the OCE risk measure coincides with the entropic risk measure \citep[Example 4.13]{follmer2016stochastic}, i.e.,
\begin{equation}\label{eq:oce-as-entropic-risk}
    \oce = \frac{1}{\beta}\log\left(\Exp[e^{\beta X}]\right).
\end{equation}
\end{example}
\begin{example}[Mean variance risk]\label{example:mean-variance}
Let the utility function be given by $u(x) = x + \beta x^2$ for all $x \in R$. Then, the OCE risk measure coincides with the mean variance risk measure, i.e., $\oce  = \Exp[X]+\beta \text{Var}(X)$.
\end{example}
\begin{example}[Monotone mean variance]\label{example:monotone-mean-variance}
Let $a\geq 2$ and define the utility function as $u(x) = a^{-1}{([x+1]^+)}^a - a^{-1}$ for all $x \in R$. For the choice of $a=2$, the OCE risk measure coincides with the monotone mean-variance risk measure, see \cite{tamtalini_multivariate_2022} and \cite[eq 1.10]{cerny_computation_2012}.
\end{example}
\begin{example}[Quartic risk]\label{example:quartic}
Define the utility function as $u(x) = \left([x+1]^+\right)^4 - 1$ for all $x \in R$. The resulting OCE risk measure for the above utility function satisfies several useful properties. See \citet{hamm_2013_stochastic_root-finding_oce} for more details.
\end{example}
\begin{example}[CVaR]\label{example:cvar}
Let $\alpha \in (0,1)$ and define the utility function as $u(x) = \alpha^{-1}x^+, \forall x \in \Rel$. Then, $\oce$ coincides with $\textrm{CVaR}_\alpha(X)$. The Conditional Value-at-Risk (CVaR) at level $\alpha \in (0,1)$ for a r.v. $X$ is given by
$
    \textrm{CVaR}_\alpha(X) \triangleq \frac{1}{\alpha} \int_0^{\alpha} \textrm{VaR}_\gamma(X) d\gamma.
$
\end{example}
\begin{example}[Leaky-CVaR (L-CVaR)]\label{example:lcvar}
Let $\alpha \in (0,1)$ and define the utility function as $u(x) = \alpha^{-1}x^+ - \arctan\left(\alpha^{-1}x^-\right)$ for all $x \in R$. Then, $\oce$ is a variant of CVaR that is sensitive to tail losses.
\end{example}
\begin{example}[Smooth-CVaR (S-CVaR)]\label{example:scvar}
Let $\alpha \in (0,1)$ and define the utility function as $u(x) = \alpha \ln\left(1+e^{\frac{x}{\alpha}}\right), \forall x \in \Rel$. Then, $\oce$ is a smooth variant of CVaR and is sensitive to tail losses.
\end{example}

\paragraph{Characterization of OCE risk.}
For the main result characterizing OCE, we make the following assumption, which is satisfied by all the utility functions in \cref{example:entropic,example:mean-variance,example:monotone-mean-variance,example:quartic,example:cvar,example:lcvar,example:scvar}.
\begin{assumption}\label{assumption:u-main}The convex and increasing utility function $u$ is continuously differentiable a.e., such that the range of $u'$ contains $1$ in its interior.
\end{assumption}
We characterize the OCE risk of an r.v. $X$ using shortfall risk, a popular risk measure in finance.
The shortfall risk, also known as `utility-based shortfall risk' (UBSR), is specified using a loss function $l$ and a risk threshold $\lambda$. The UBSR is formally defined below.
\begin{definition}\label{def:ubsr}
    Let $l:\Rel\to \Rel$ be increasing and continuous a.e., and let $\lambda$ be a scalar value chosen within the range of $l$. Then, the UBSR for an r.v. $X$ is given as
    \begin{equation*}
    \ubsr \triangleq \{ \inf_{ t \in \Rel} \left| \Exp\left[l(X-t)\right] \leq \lambda \right.\}.
\end{equation*}
\end{definition}
Consider the function $G_X(\cdot): \Rel \to \Rel$ defined below.
\begin{equation*}
    G_X(t) \triangleq t + \Exp\left[u(X-t)\right].
\end{equation*}
The expectation above is finite if $X \in \xu$. If $X \in \mathcal{X}_{u'}$ also holds, then it is easy to see that $G_X$ is convex and differentiable, and the  derivative of $G_X$ is given by
\begin{equation}\label{eq:G-prime-def}
    G^{'}_X(t) = 1 - \Exp\left[u'(X-t)\right].  
\end{equation}
Suppose $G_X(\cdot)$ attains a minimum at some $t^* \in \Rel$, then we have
\begin{equation}\label{eq:oce-t-star}
    \oce = G_X(t^*) = t^* + \Exp\left[u(X-t^*)\right].
\end{equation}
For some pathological cases, $G_X(\cdot)$ may not attain a minimum in $\Rel$. To avoid such cases, we assume existence of $t_X^\mathrm{l}$ and $t_X^\mathrm{u}$ such that $t^* \in \left[t_X^\mathrm{l}, t_X^\mathrm{u}\right]$.
Furthermore, $G_X$ is convex, and therefore, finding $t^*$, the minimizer of $G_X(\cdot)$, is equivalent to finding the root of \eqref{eq:G-prime-def}, and this problem can be solved by associating it with the UBSR, under suitable assumptions on $X$ and $u$. Precisely, we equate $u'$ and $1$ with $l$ and $\lambda$ respectively, and hence, $u$ being convex and continuously differentiable a.e. is analogous to $l$ being increasing and continuous a.e., respectively, which conforms to \Cref{def:ubsr}. We denote $\xhu \triangleq \xu \cap \mathcal{X}_{u'}$, and present the following proposition, which associates the OCE risk with the UBSR.
\begin{proposition}\label{proposition:oce-sr-coincides}
Suppose the utility function satisfies \Cref{assumption:u-main} and let $X \in \xhu$. Suppose there exist $t_X^{\mathrm{u}}, t_X^{\mathrm{l}} \in \Rel$ such that $G_X^{'}(t_X^{\mathrm{u}})\leq 0 < G_X^{'}(t_X^{\mathrm{l}})$. Then, $\sr{u'}{1}{X}$ is a root of $G'_X(\cdot)$ as well as a minimizer of $G_X(\cdot)$. Furthermore, the OCE of $X$ is given as
    \begin{equation*}
        \oce = \sr{u'}{1}{X} + \Exp\left[u(X-\sr{u'}{1}{X}\right],
    \end{equation*}
    and $\oc{\cdot}$ is a convex risk measure.
\end{proposition}


\begin{remark}
We provide a Wasserstein-distance bound on the difference in OCE between two distributions, in the spirit of \cite{JMLR-LAP-SPB}, albeit for the non-Lipschitz utility case. In particular, 
we obtain\\ $\left| \oce - \oc{Y} \right|
\le \mathcal{W}_2(\mu_X,\mu_Y)\sqrt{\sigma_1^2+1}$, \\
where we assume $\textrm{Var}\left(u'\left(X-\sr{u'}{1}{X}\right)\right) \leq \sigma_1^2$.
In Lemma 12 of \citet{JMLR-LAP-SPB}, the authors obtained a bound on $\left| \oce - \oc{Y} \right|$ in terms of the 1-Wasserstein distance (between marginals $\mu_X$ and $\mu_Y$) under the assumption that the utility function is Lipschitz. Our extension covers \cref{example:entropic,example:mean-variance,example:monotone-mean-variance,example:quartic}, which includes popular OCE instances with non-Lipschitz utility functions. See \Cref{supp:Wasserstein} for more details.
\end{remark}



\section{OCE Risk Estimation}\label{sec:estimation}


In this section, we consider the problem of estimating the OCE risk of a r.v. $X \in \xhu$ using samples from the distribution of $X$.  
For $m \in \mathcal{N}$, we obtain i.i.d samples $\{Z_i\}_{i=1}^m$ (also indicated as a random vector $\mathbf{Z}$) and use them construct $\srm{\mathbf{Z}}$, an estimator of $\sr{u'}{1}{X}$. We use the same samples to construct $\ocm{\mathbf{Z}}$, the estimator of $\oce$. The estimators $\srt_m:\Rel^m \to \Rel$ and $\ocet_m:\Rel^m \to \Rel$ are defined as
\begin{align}
    \label{eq:sr-m-definition-oce}
    \srm{\mathbf{z}} &\triangleq \min \left\{ t \in \Rel \left| \frac{1}{m} \sum_{j=1}^m u'(\textrm{z}_j-t) \leq 1 \right. \right\}, \\
    \label{eq:oce-m-definition}
    \ocm{\mathbf{z}} &\triangleq \srm{\mathbf{z}} + \frac{1}{m}\sum_{j=1}^m\left[u(\textrm{z}_j-\srm{\mathbf{z}})\right].
\end{align}
    
    In the following lemma, we bound the OCE estimation error in the case where the utility function is Lipschitz.
\begin{lemma}\label{lemma:oce-saa-lipschitz-bounds}
    Suppose \Cref{assumption:u-main} and the assumptions of \Cref{proposition:oce-sr-coincides} are satisfied, and the utility function $u$ is $K$-Lipschitz. If there exists $q>2$ and $T>0$ such that $\norm{X}_{\Leb{q}} \leq T$, then
    \begin{align*}
        \Exp\left[\left| \ocm{\mathbf{Z}} - \oce \right| \right] &\le \frac{39K T}{\sqrt{m}}.
    \end{align*}
    If there exists $q>4$ and $T>0$ such that $\norm{X}_{\Leb{q}} \leq T$, then
    \begin{align*}
        \Exp\left[\left| \ocm{\mathbf{Z}} - \oce \right|^2 \right]  &\le \frac{108 K^2T^2}{\sqrt{m}}.
    \end{align*}
\end{lemma}

In Corollary 20 of \citet{JMLR-LAP-SPB}, the authors derived an MAE bound of the order \order{1/\sqrt{m}} for OCE estimation under the assumption that the utility function is Lipschitz. Their result does not include MSE bounds, which will be useful for deriving convergence rates of our OCE optimization algorithm. Furthermore, several popular instances of OCE risk are not Lipschitz, e.g., see \cref{example:entropic,example:mean-variance,example:monotone-mean-variance,example:quartic}. We fill these gaps under the following additional assumption.
\begin{assumption}\label{assumption:u-prime-variance}There exists $\sigma_1>0$ such that the utility function satisfies $\textrm{Var}\left(u'\left(X-\sr{u'}{1}{X}\right)\right) \leq \sigma_1^2$.
\end{assumption}
This assumption is satisfied for \cref{example:mean-variance,example:cvar,example:monotone-mean-variance,example:quartic,example:lcvar,example:scvar} when the underlying distribution has bounded higher moments, and for \Cref{example:entropic}, when $X$ is sub-Gaussian.

\begin{lemma}\label{lemma:oce-saa-wasserstein-bound}
    Suppose \cref{assumption:u-main,assumption:u-prime-variance}, and the assumptions of \Cref{proposition:oce-sr-coincides} are satisfied. Let there exist $q>4$ and $T>0$ such that $\norm{X}_{\Leb{q}}\leq T$. Then, we have
    \begin{align*}
        \Exp\left[\left| \ocm{\mathbf{Z}} - \oce \right|^2 \right]  &\le \frac{108(\sigma_1^2+1)T^2}{\sqrt{m}},
    \end{align*}
    where $\sigma_1$ is as given in \cref{assumption:u-prime-variance} respectively.
\end{lemma}

\paragraph{Efficient procedure for OCE Estimation.}
We note that the UBSR estimation problem in \cref{eq:sr-m-definition-oce} is a root-finding problem of a deterministic and monotone function $g_z(t) = \frac{1}{m}\sum_{j=1}^m u'(-z_j-t) - 1$, and can be solved efficiently using bisection search. We now provide a simple procedure for estimating the quantity in \cref{eq:sr-m-definition-oce} using samples $\mathbf{z}$.
\begin{enumerate}
    \item Find the interval $[\textrm{low,high}]$ such that $g_\mathbf{z}(\textrm{high}) \leq 0 \leq g_z(\textrm{low})$.
    \item Choose $\epsilon,\delta>0$ and use bisection on this interval to find a solution that satisfies the following:
    \begin{align}\label{eq:oce-delta-epsilon}
        \left| \hat{t}_m - \srm{\mathbf{z}} \right| \leq \delta, \textrm{ and }
        \left| g_\mathbf{z}(\hat{t}_m) \right| \leq \epsilon,   
    \end{align}
\end{enumerate}

For the sake of completeness, we include \Cref{alg:oce-saa-bisect} in the supplementary, whose output satisfies \cref{eq:oce-delta-epsilon}. Next, we present the following scheme to approximate $\ocm{\mathbf{z}}$ using the samples $\mathbf{z}$ and the UBSR estimate $\hat{t}_m$:
\begin{equation}\label{alg:oce_estimation}
    \hat{s}_m \leftarrow \hat{t}_m + \frac{1}{m} \sum_{i=1}^m u(Z_i - \hat{t}_m) 
\end{equation}


Choosing smaller values for $\delta$ and $\epsilon$ ensures that the output $\hat{s}_m$ lies closer to $\ocm{\mathbf{Z}}$.
Estimation bounds on $\hat{s}_m$ can be established in a manner similar to those available in lemmas \ref{lemma:oce-saa-lipschitz-bounds} and \ref{lemma:oce-saa-wasserstein-bound}, and a detailed exposition is provided in \Cref{sec:oce-algorithm} in the supplementary.


\section{OCE Risk Optimization}
\label{sec:oce-opt}
 We consider the following problem:
\begin{align}
    \textrm{Find } \theta_* \in \argmin_{\theta\in \Theta}\oc{F(\theta,\xi)}.
    \label{eq:oce-opt-pb}
\end{align}
In the above, given any $\theta \in \Theta$, $F(\theta,\xi)$ is the r.v. associated with the parameter $\theta$ and 
$\xi$ denotes the noise factor, independent of $\theta$. The decision set $\Theta$ is a subset of $\B\subseteq \Rel^d$, where $\B $ is an open and convex set. The OCE risk and its properties are well-defined for any element in $\B$, whereas $\Theta$ is a possibly constrained and problem-dependent construct. For instance, in a portfolio optimization problem, $\Theta$ is a $d$-dimensional simplex denoting portfolio weights. We present several choices of $F$ that yield well-known problem instances in machine learning and finance.
\begin{example}[Portfolio Optimization]\label{example:portfolio-optimization}
    Consider a financial market with $d$ assets, and let the random vector $\mathbf{\xi} \in \Rel^d$ denote asset-wise market returns. Let $\theta \in \Theta$ denote an asset allocation or portfolio weight, then the r.v. $F(\theta, \mathbf{\xi})\triangleq -\xi^T\theta$ denotes portfolio losses associated with $\theta$.
\end{example}
\begin{example}[Classification]\label{example:classification} 
Consider a classification problem with $K$ classes, where $\xi$ denotes the random vector corresponding to the inputs and targets. Every data point $z \triangleq \langle x,y \rangle$ is thus sampled from the distribution of $\xi$, where $y$ is a $K$-dimensional one-hot encoded vector. Each parameter vector $\theta$ in the decision set $\Theta$ corresponds to a model $\hat{f}(x;\theta)$ that outputs the raw logits for a given input $x$. The categorical cross-entropy loss $F(\theta,z)$ is then defined as: $F(\theta, z) = - \langle y, \log \operatorname{softmax}\big(\hat{f}(x; \theta)\big) \rangle$.
\end{example}
\begin{example}[Regression]\label{example:regression} 
Consider a regression problem where $\xi$ denotes the random vector corresponding to the inputs and continuous targets. Every data point $z \triangleq \langle x,y \rangle$ is thus sampled from this distribution of $\xi$, where $y \in \mathbb{R}$ is a scalar target value. Each parameter vector $\theta$ in the decision set $\Theta$ corresponds to a model $\hat{f}(x;\theta)$ that outputs a scalar prediction for a given input $x$. The squared residual $F(\theta,z)$ is then defined as:$F(\theta, z) = \left( \hat{f}(x; \theta) - y \right)^2$.

\end{example}
\begin{example}[Uncertainty Quantification]\label{example:uncertainty-quantification}
We consider the Mean Variance Estimation (MVE) framework for uncertainty quantification, where every data point $z \triangleq \langle x,y \rangle$ is drawn from a random vector $\xi$, that has the joint distribution  $y \mid x \sim \mathcal{N}\big(\mu(x), \sigma^2(x)\big)$. The model parameters $\theta$ are learned by minimizing $F(\theta,z)$, the negative log-likelihood (NLL) of a Gaussian distribution (omitting constant terms) given by:
\begin{equation}
    F(\theta, z) = \frac{1}{2}\log\left(\sigma^2_\theta(x)\right) + \frac{1}{2}\frac{\left(y-\mu_\theta(x)\right)^2}{\sigma_\theta^2(x)}.
\end{equation}
\end{example}

\paragraph{OCE Gradient and properties.}
We use the association between the OCE and the UBSR given by \Cref{proposition:oce-sr-coincides} to derive the expression for the gradient of OCE. For notational convenience, let $\octh{} \triangleq \oc{F(\theta,\xi)}$, and  $\srth{} \triangleq \sr{u'}{1}{F(\theta,\xi)}$. We first state an assumption made in \Cref{proposition:oce-sr-coincides} using notation that involves $\theta$ and $F$.
\begin{assumption}\label{assumption:oce-tu-tl-theta}
For every $\theta \in \B, F(\theta,\xi) \in \xhu$ and there exist $t_\textrm{u}(\theta), t_\textrm{l}(\theta) \in \Rel$ such that $G_{F(\theta,\xi)}'(t_\textrm{u}(\theta), \theta) \leq 0$ and $G_{F(\theta,\xi)}'(t_\textrm{l}(\theta),\theta)>0$. 
\end{assumption}
A similar assumption has been made in \citet{zhaolin2016ubsrest,Hegde2024}. Under \cref{assumption:u-main,assumption:oce-tu-tl-theta}, \Cref{proposition:oce-sr-coincides} implies that $\forall \theta \in \B$, $\octh{}$ is expressed as
\begin{equation}\label{eq:oce-expression-compact}
    \octh{} = \srth{} + \Exp\left[u\left(F(\theta,\xi)-\srth{}\right)\right].
\end{equation}
In the following result, we derive the expression for the gradient of OCE using the expression in \cref{eq:oce-expression-compact}.
\begin{theorem}[Gradient of OCE]\label{theorem-oce-gradient}
    Suppose the utility function $u$ is twice differentiable, $F(\cdot,\xi)$ is continuously differentiable a.s., and \cref{assumption:u-main,assumption:oce-tu-tl-theta} hold. Then $\ocet(\cdot)$ is continuously differentiable and, for every $\theta \in \B$, the gradient of OCE is given by
    \begin{equation*}
    \nabla \octh{} = \Exp\left[u'\left(F(\theta,\xi)-\srth{}\right)\nabla F(\theta,\xi) \right].
\end{equation*}
\end{theorem}
The result above applies to \cref{example:entropic,example:mean-variance,example:monotone-mean-variance,example:quartic,example:lcvar,example:scvar}.
Using the above gradient expression, we shall establish the smoothness of $\octh{}$, and in that spirit, we make the following assumptions. 
\begin{assumption}\label{assumption:F-gradient-bound-II}
There exists $M>0$ such that for every $\theta \in \B$, $\norm{\nabla F(\theta, \xi)}_2 \leq M$ a.s.
\end{assumption}
\begin{assumption}\label{as:F-smooth}
    There exists  $L>0$ such that, $\forall \theta_1,\theta_2 \in \B$, $
        \norm{\nabla F(\theta_1, \xi)- \nabla F(\theta_2,\xi)}_{2} \leq L \norm{\theta_1-\theta_2}_2 \;\;\text{a.s.}
    $.
\end{assumption}

For classification and regression applications (see \cref{example:classification,example:regression}) as well as the uncertainty quantification application in \Cref{example:uncertainty-quantification}, \cref{assumption:F-gradient-bound-II} is satisfied with a bounded input space ($\norm{x}\leq M_1,\forall x \in X$) and bounded parameter space ($\norm{\theta}\leq M_2,\forall \theta \in \Theta$). Further, \cref{as:F-smooth} is satisfied in classification/regression examples under an affine parameterization, and a neural network parameterization with smooth activation functions.
Finally, for \Cref{example:uncertainty-quantification}, we note that in addition to the boundedness assumptions of input space and parameter space, a lower bound on the variance, i.e., $\sigma^2_\theta \geq \sigma^2_\text{min}>0, \forall \theta \in \Theta$, ensures  \Cref{as:F-smooth} is satisfied.  

The following result establishes smoothness of $\ocet(\cdot)$. 
\begin{lemma}\label{lemma:oce-smooth}
    Suppose \Cref{as:F-smooth} and the assumptions of \Cref{theorem-oce-gradient} are satisfied. Then $\ocet(\cdot)$ is $L$-smooth.
\end{lemma}

The following result shows that $\ocet$ is strongly convex.
\begin{lemma}\label{lemma:oce-strong-convexity}
    Suppose $F(\cdot,\xi)$ is $\mu$-strongly convex w.p. $1$, and the assumptions of \Cref{theorem-oce-gradient} are satisfied. Then $\ocet(\cdot)$ is $\mu$-strongly convex. 
\end{lemma}
\paragraph{OCE Gradient estimator.} From the gradient expression in \Cref{theorem-oce-gradient}, it is evident that to estimate the gradient of OCE, we also need to estimate $\srth{}$. We define the functions $SR^m_{\theta}:\Rel^m \to \Rel$ and $Q^m_\theta:\Rel^m \to \Rel^d$ for a given $\theta \in \B$, as follows:
For any $\mathbf{z} \in \Rel^m$,
\begin{align}
    \label{eq:SR-definition-oce}
    &SR^m_{\theta}(\mathbf{{z}}) \triangleq \min \left\{ t \in \Rel \left| \frac{1}{m} \sum_{j=1}^m u'(F(\theta,{\mathrm{z}}_j)-t) \leq 1 \right. \right\}, \\
    \label{eq:oce-grad-estimator}
    &Q^m_\theta(\mathbf{z}) \triangleq \frac{1}{m}\sum_{j=1}^m u'\left(F(\theta,\textrm{z}_j)-SR^m_\theta(\mathbf{z})\right)\nabla F(\theta,\textrm{z}_j).
\end{align}
Let $\mathbf{Z}$ be an $m$-dimensional vector such that each $Z_j$ is an i.i.d. copy of $\xi$. Then, $SR^m_\theta(\mathbf{{Z}})$ and $ Q^m_\theta(\mathbf{Z})$ are our proposed estimators of $\srth{}$ and $\nabla \octh{}$ respectively. Next, we restate \Cref{assumption:u-prime-variance} using $F$ and $\theta$. 
\begin{assumption}\label{assumption:u-prime-variance-theta}
    There exists $\sigma_1>0$ such that $\textrm{Var}\left(u'\left(F(\theta,\xi)-\srth{}\right)\right) \leq \sigma_1^2$, for every $\theta \in \B$.
\end{assumption}

The following assumption is made to bound the variance of the gradient, and it is common to the non-asymptotic analysis of stochastic gradient algorithms, cf. \citep{robust-stoc-appx-nemirovski-shapiro,ghadimi-lan-strongly-convex-stoc-opt}.
\begin{assumption}\label{assumption:u-F-variance-bound}
There exists $T > 0$ such that  
    \begin{align*}
        \norm{ {\mathbf{Y}_\theta} - \Exp\left[\mathbf{Y}_{\theta}\right] }_{\Leb{2}} \leq T, \;\; \text{ for every } \theta \in \B,
    \end{align*}    
    where $\mathbf{Y}_\theta \triangleq u'\left(F(\theta,\xi)-\srth{}\right)\nabla F(\theta,\xi), \forall \theta \in \Theta$.
\end{assumption}
The result below provides bounds on the MSE and MAE of the OCE gradient estimator defined by (\ref{eq:oce-grad-estimator}).
\begin{lemma}\label{lemma:oce-gradient-estimator-I}
    Suppose the assumptions of \Cref{theorem-oce-gradient} and \cref{assumption:F-gradient-bound-II,assumption:u-F-variance-bound,assumption:u-prime-variance-theta} are satisfied. Then for every $m \in \N$ and $\theta \in \B \subseteq \Rel^d$, we have 
    \begin{align*}
        \Exp\left[\norm{Q^m_\theta(\mathbf{Z}) - \nabla \octh{}}_2\right] &\leq \frac{M\sigma_1+T}{\sqrt{m}}, \\
        \Exp\left[\norm{Q^m_\theta(\mathbf{Z}) - \nabla \octh{}}_2^2\right] &\leq \frac{2(M\sigma_1^2+T^2)}{m},
    \end{align*}
    where $M, T$ and $\sigma_1$ are given in \cref{assumption:F-gradient-bound-II,assumption:u-F-variance-bound,assumption:u-prime-variance-theta}.
\end{lemma}

\paragraph{SG Algorithm for OCE optimization.}
We now describe our proposed OCE-SG algorithm for solving \cref{eq:oce-opt-pb}. The pseudocode of this algorithm is provided in \Cref{supp:algorithm-and-convergence-rates} of the supplementary material. In each iteration $k$ of our algorithm, we obtain $m_k$ samples from $\xi$ to form $\mathbf{Z}^k$ and perform the following update: 
\begin{equation}\label{eq:sg-update}
    \theta_{k} = \Pi_\Theta\left(\theta_{k-1} - \alpha_k {Q}_{\theta_{k-1}}^{m_k}(\mathbf{Z}^k)\right), k \geq 1,
\end{equation}
where $\Pi_\Theta:\Rel^d \to \Theta$ is a non-expansive projection operator, ${Q}_{\theta_{k-1}}^{m_k}$ is as defined in \cref{eq:oce-grad-estimator} and $\alpha_k$ is the step size at iteration $k$. The following result establishes a non-asymptotic bound on the OCE-SG iterates across three regimes: strongly convex, convex, and non-convex.


\begin{theorem}\label{theorem:oce-main}
Suppose the assumptions of \Cref{theorem-oce-gradient} and \cref{assumption:u-prime-variance-theta,assumption:F-gradient-bound-II,as:F-smooth,assumption:u-F-variance-bound} are satisfied. Suppose the minimizer $\theta_*$ defined in \cref{eq:oce-opt-pb} satisfies $\nabla \ocet(\theta_*)=0$.

\noindent\textbf{\textit{Strongly-convex case:}} Suppose the OCE objective $\ocet(\cdot)$ is $\mu$-strongly convex on $\Theta$ for some $\mu>0$.  Set $\alpha_k = \frac{c}{k},m_k=k,\,\forall k$, where $c>\frac{1}{\mu}$. Then, $\theta_n$, governed by \cref{eq:sg-update}, satisfies
    \begin{align*}
        \Exp&\left[\norm{\theta_n - \theta_*}_2^2\right] \leq \frac{C_1}{n+1}.
    \end{align*}
\noindent\textbf{\textit{Convex case:}}
 Suppose $\ocet(\cdot)$ is convex on $\Theta$. Suppose iterates $(\theta_1,\ldots,\theta_n)$ are governed by \cref{eq:sg-update} with $\alpha_k = \frac{1}{K\sqrt{k}},m_k=k,\,\forall k$. Let $\overline{\theta}_n=\frac{1}{n}\sum_{k=1}^n \theta_k$. Then,
    \begin{align*}
        \Exp\left[\ocet(\overline{\theta}_n)-\octh{*}\right] \leq\frac{C_2}{\sqrt{n}}.
    \end{align*}
\noindent\textbf{\textit{Non-convex case:}}    
 Suppose the iterates $(\theta_1,\ldots,\theta_n)$ are governed by \cref{eq:sg-update} with $\alpha_k = \frac{1}{L\sqrt{k}},m_k=k,\,\forall k$, then 
    \begin{align*}
        \Exp&\left[\min_{k \in [1,n]}\norm{\nabla \octh{k-1}}_2^2\right] \leq\frac{C_3}{\sqrt{n}}.
    \end{align*}
    Here, $K$ is as defined in Lemma \ref{lemma:oce-smooth}, and the constants $C_1,C_2$ and $C_3$ are given in \Cref{supp:algorithm-and-convergence-rates} of the supplementary material.
\end{theorem}
Recall that the OCE gradient is biased, and this bias is controlled using the batch sizes $m_k$. To produce an any-time algorithm using \cref{eq:sg-update}, we need an increasing batch size which diminishes this bias and leads to convergence rates that decrease with each iteration $k$. If the horizon $N$ is known, then a fixed batch size of $m_k=N$ suffices. When working with a fixed dataset of size $M$, a fixed batch size of $m_k=m$ leads to convergence rates of \order{1/m}, which follow using arguments that are identical to those made in the proof of \Cref{theorem:oce-main}, and we avoid stating a separate result. 
\begin{remark}
     Portfolio optimization in \cref{example:portfolio-optimization} does not satisfy \Cref{as:F-smooth} and \Cref{assumption:F-gradient-bound-II}. However, results similar to Lemmas \ref{lemma:oce-smooth} and \ref{lemma:oce-gradient-estimator-I} can be obtained using the linearity of $F$ and a sub-Gaussianity assumption for $\norm{\nabla F(\theta, \xi)}_{2}^2$ to infer a bound of the following form: $\Exp\left[\norm{Q^m_\theta(\mathbf{Z}) - \nabla \octh{}}_2^2\right]\leq$ \order{1/\sqrt{m}}. Thus, the bounds in \Cref{theorem:oce-main} apply to \Cref{example:portfolio-optimization}, see Appendix \ref{proof:lemma-oce-gradient-estimator-II} for the details.
\end{remark}

\section{Simulation Experiments}
\label{sec:experiments}
We conduct several experiments to test the performance of the OCE estimation and optimization algorithms proposed in this work. Due to space limitations, the detailed experimental setup and additional results are provided in \Cref{sec:suppl-experiments} of the supplementary material.

\begin{table}[ht]
\centering
\caption{Performance of NN models minimized on three OCE optimization criteria: entropic risk, CVaR, and Monotone Mean Variance (MMV), against benchmark algorithms on a classification task with the \textit{UCI Heart Disease Dataset}.}
\label{table:uci-heart-disease}
\begin{tabular}{lcccc}
\toprule
\textbf{Model} & \textbf{Acc.} & \textbf{F1} & \textbf{ECE} & \textbf{AUROC} \\
\midrule
Logistic Reg.  & 0.8444 & 0.8250 & 0.1050 & 0.9439 \\
Random Forest        & 0.8556 & 0.8395 & 0.1198 & 0.9293 \\
NN BCE Loss          & 0.8556 & 0.8395 & \textbf{0.0980} & 0.9415 \\
\midrule
Entropic Risk   & 0.8556 & 0.8395 & 0.1089 & \textbf{0.9504} \\
CVaR        & \textbf{0.8778} & \textbf{0.8736} & 0.3721 & 0.9444 \\
MMV Risk & 0.8667 & 0.8500 & 0.1192 & 0.9479 \\

\bottomrule
\end{tabular}%
\end{table}

\begin{table}[ht]
\centering
\caption{Performance of NN models minimized on three OCE optimization criteria: entropic risk, SCVaR, and MMV risk, against benchmark algorithms on a classification task with the \textit{Breast Cancer Detection dataset}.}
\label{table:breast-cancer}
\begin{tabular}{lcccc}
\toprule
\textbf{Model} & \textbf{Acc.} & \textbf{F1} & \textbf{ECE} & \textbf{AUROC} \\
\midrule
Logistic Reg.  & \textbf{0.9883} & \textbf{0.9907} & 0.0324 & \textbf{0.9981} \\
Random Forest        & 0.9357 & 0.9488 & 0.0430 & 0.9913 \\
NN BCE Loss          & 0.9532 & 0.9619 & 0.0437 & 0.9908 \\
\midrule
Entropic risk    & 0.9591 & 0.9668 & \textbf{0.0213} & 0.9950 \\
Smooth CVaR       & \textbf{0.9883} & \textbf{0.9907} & \textbf{0.0239} & 0.9963 \\
MMV Risk & 0.9649 & 0.9717 & \textbf{0.0263} & 0.9956 \\
\bottomrule
\end{tabular}%
\end{table}

\paragraph{Classification (\Cref{example:classification})} We consider a two-layer neural network (NN) parameterization for the \textit{UCI Heart Disease (UCI-HD)} dataset \cite{heart_disease}. We learn OCE-optimal NN parameters using the OCE-SG algorithm. We benchmark them against logistic regression, random forest, and an NN optimized for the binary cross-entropy (BCE) loss. \Cref{table:uci-heart-disease} shows the performance of these algorithms, evaluated across four well-known classification metrics: accuracy, F1-score, ECE (Expected Calibration Error), and AUROC (Area Under the Receiver Operating Characteristic). From \Cref{table:uci-heart-disease}, we observe that the OCE models achieve higher accuracy, F1, and AUROC than the rest on UCI-HD dataset. We conduct a similar experiment on the \textit{Breast Cancer Detection (BCD)} dataset \cite{breast_cancer_wisconsin}. From \Cref{table:breast-cancer}, we observe that the OCE models achieve the lowest ECE, and show comparable performance on the remaining metrics. 

\paragraph{Portfolio optimization (\Cref{example:portfolio-optimization}).}
With $\octh{} \triangleq \oc{F(\theta,\xi)}$, we obtain OCE-optimal portfolio allocations under three choices of asset returns $\xi$, corresponding to the stock markets S\&P, NASDAQ and the FTSE. Using data sourced from these stock markets, we obtain portfolio allocations for each OCE instance given in \cref{example:entropic,example:mean-variance,example:monotone-mean-variance,example:cvar,example:lcvar,example:scvar,example:quartic} and compare them against three classical benchmark portfolios: equal weighted, maximum Sharpe ratio, and Hierarchical Risk Parity (HRP). A comprehensive analysis is provided in \Cref{sec:suppl-experiments} of the supplementary material. \Cref{figure:port-opt-oce} captures the performance on the S\&P dataset and we observe that the OCE-optimal portfolios outperform the benchmarks, demonstrating that OCE is a viable optimization metric and that our gradient-based optimization scheme is a compelling candidate for optimizing OCE.
\begin{figure}[t] 
  \centering
    \includegraphics[width=\linewidth]{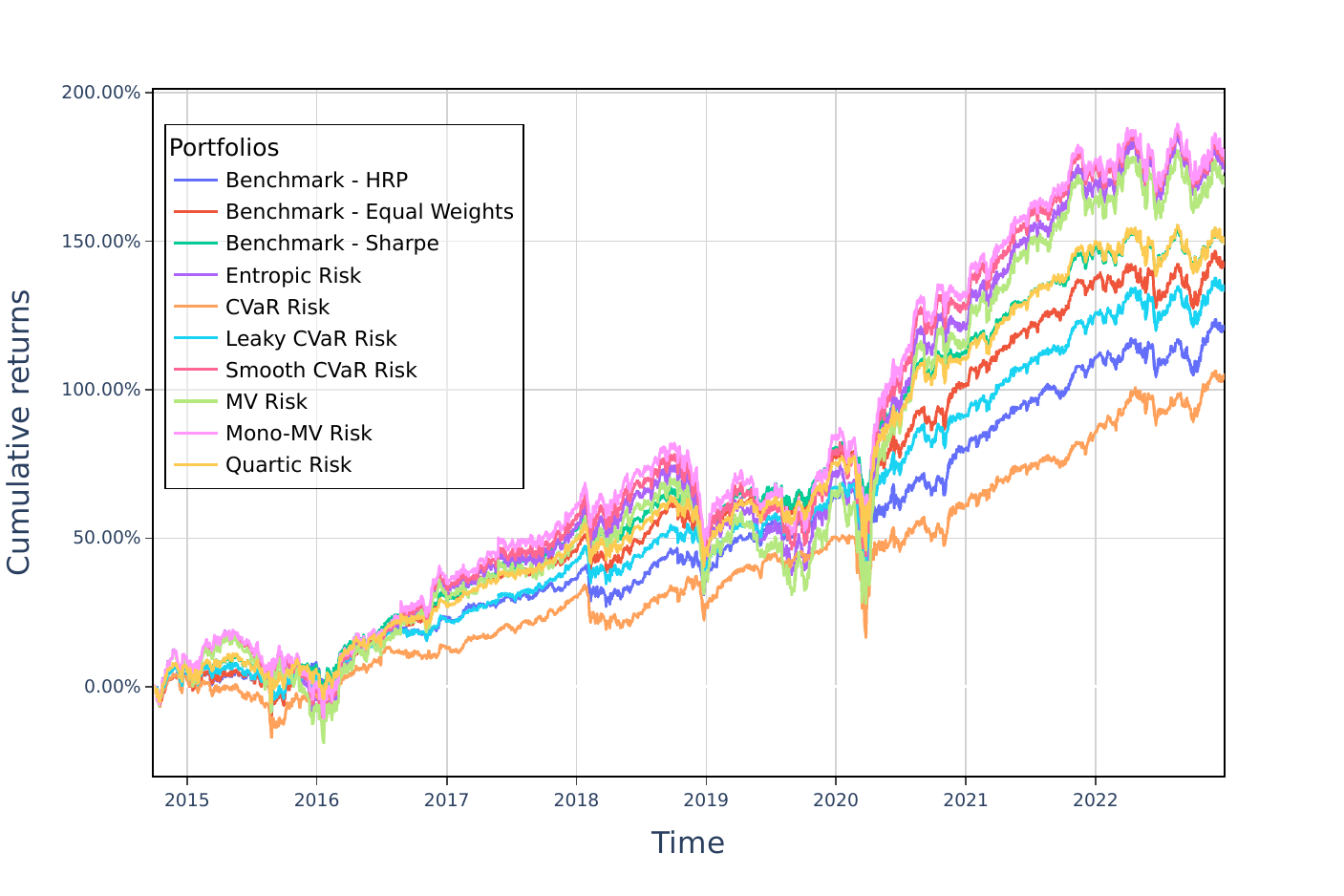}
  \caption{Performance of OCE-SG algorithm for a variety of OCE risk measures in a portfolio optimization application sourced from the S\&P stock market data.}\label{figure:port-opt-oce}
\end{figure}

\paragraph{CVaR Optimization.}
Borrowing the OCE-optimal portfolio allocations from the previous experiment, we focus on the three OCE variants related to the CVaR risk, i.e., \cref{example:cvar,example:lcvar,example:scvar}. We compare their performance with a CVaR-optimal portfolio generated by a solver in the \textit{skfolio} library. We evaluate the performance of these portfolios using four well-known risk metrics: annualized returns, CVaR, Conditional Drawdown-at-Risk (CDaR), and the Sharpe ratio. We observe from \Cref{table:cvar-comparision} that i) for the utility function in \Cref{example:cvar}, the OCE-SG algorithm gives a portfolio that has identical score on the CVaR metric, but marginally better scores on all other metrics, in comparison to the CVaR-optimal portfolio given by the solver, and ii) our proposed CVaR variants serve as viable alternatives, taking into account all the performance metrics. 
\begin{table}[t]
\caption{Comparison of portfolios given by \Cref{alg:oce-minimization} for CVaR variants in \cref{example:cvar,example:lcvar,example:scvar} against a CVaR solver of \textit{skfolio} library.}
\label{table:cvar-comparision}
\centering
    \small 
    \begin{tabular}{l *{4}{r}}
      \toprule
      \textbf{Portfolio} & \textbf{Ann. Ret.} & \textbf{CVaR} & \textbf{CDaR} & \textbf{Sharpe} \\
      \midrule
      CVaR (solver) & 0.1250 & \textbf{0.0239} & 0.1561 & 0.7639 \\
      CVaR (OCE)    & 0.1261 & \textbf{0.0239} & 0.1544 & 0.7713 \\
      LCVaR (OCE)   & 0.1619 & 0.0258 & \textbf{0.1350} & \textbf{0.9209} \\
      SCVaR (OCE)  & \textbf{0.2135} & 0.0348 & 0.2681 & 0.8466 \\
      \bottomrule
    \end{tabular}
    \captionof{table}{Risk and Return Metrics}
    \label{tab:cvar-comparision}
\end{table}

\paragraph{Uncertainty quantification.}
Consider the objective $F(\theta,\xi)$ for the MVE framework given in \Cref{example:uncertainty-quantification}. We modify the MVE's loss $\Exp\left[F(\theta,\xi)\right]$ by replacing the expectation with the OCE risk criterion $\ocet(\cdot)$, where $u$ is the exponential utility function. Using a simple heterogeneous regression dataset, we compare the model obtained by minimizing the OCE against two benchmarks: the original MVE solution and the state-of-the-art Deep ensembles \citep{deep_ensembles}. The results are presented in \Cref{tab:mve-metrics}, and we observe that our approach significantly outperforms both the benchmarks and is a viable alternative for quantifying model uncertainty. 
\begin{table}[t]
\caption{Comparison of our MVE with OCE approach against vanilla MVE and deep ensembles across the metrics: Mean Absolute Calibration Error (MACE), Root Mean Squared Calibration Error (RMSCE), and Miscalibration Area (MA).}
\label{tab:mve-metrics}
\centering
\renewcommand{\arraystretch}{1.5}
\begin{tabular}{ l | c | c | c }
\specialrule{1.5pt}{0pt}{0pt}
\textbf{Algorithm} & \textbf{MACE} & \textbf{RMSCE} & \textbf{MA} \\ \midrule
 Deep Ensembles & 0.0772 & 0.0883 3 & 0.0780 \\ \hline
MVE & 0.0718 & 0.0815 & 0.0724 \\ \hline
MVE with OCE & \textbf{0.0285} & \textbf{0.0367} & \textbf{0.0283} \\ \bottomrule
\end{tabular}
\end{table}


\section{Conclusions}
\label{sec:conclusions}
We proposed estimators for OCE and its gradient and provided non-asymptotic bounds on the MAE and MSE of these estimators for both Lipschitz and non-Lipschitz utility functions. Our proposed OCE gradient estimator is based on the OCE gradient expression we derive. We used this estimator to devise an SG algorithm to optimize the OCE of a parametrized loss r.v. and to provide non-asymptotic convergence bounds. Finally, we presented simulation experiments to demonstrate the performance of our proposed algorithms. 

\newpage
\bibliography{references}

\clearpage

\onecolumn

\title{Optimized Certainty Equivalent Risk Minimization Using Samples: \\Algorithms, Convergence Rates, and Applications\\(Supplementary Material)}
\maketitle
\appendix
\setcounter{secnumdepth}{2}
\section*{Organization of the supplementary material}
In \Cref{appendix:oce}, we provide proofs related to the results that characterize the OCE risk for unbounded r.v.s. In \Cref{appendix:estimation}, we provide proofs for results related to OCE estimation, while in \Cref{appendix:optimization}, we provide proofs for the results related to OCE optimization. In \Cref{supp:algorithm-and-convergence-rates}, we provide the pseudocode for the OCE-SG Algorithm and subsequently provide complete proofs on the convergence guarantees of the algorithm in all three regimes, strongly-convex, convex, and non-convex. Finally, in \Cref{sec:suppl-experiments}, we provide a comprehensive analysis details of the simulation experiments reported in the main paper, and also include more simulation experiments.
\section{Characterization of the OCE risk measure}\label{appendix:oce}
\subsection{Proof of Proposition \ref{proposition:oce-sr-coincides}}\label{proof-proposition-sr-coincides}
\begin{proof}
    Putting $l=u'$ and $\lambda=1$ in \Cref{def:ubsr}, we observe that $\sr{u'}{1}{X}$ is a solution to \cref{eq:G-prime-def}, and therefore, $\sr{u'}{1}{X}$ is a root of $G'_X(\cdot)$. This implies that $\ubsr$ is also a minimizer of $G_X(\cdot)$ and therefore, we substitute $t^*$ in \cref{eq:oce-t-star} with $\sr{u'}{1}{X}$, and we have
    \begin{equation}\label{eq:claim-oce}
        \oce = \sr{u'}{1}{X} + \Exp\left[u(X-\sr{u'}{1}{X}\right].
    \end{equation}
    Next, we show that $\oc{\cdot}$ is a convex risk measure.
    \paragraph{Monotonicity.}Let $X_1,X_2 \in \xhu$ such that $X_1\leq X_2$ a.s. From the claim in (\ref{eq:claim-oce}), $\sr{u'}{1}{X_1}$ is a minimizer of $G_{X_1}(\cdot)$, and therefore it follows that $G_{X_1}\left(\sr{u'}{1}{X_1}\right) \leq G_{X_1}\left(t\right),\forall t \in \Rel$. Then substituting $t = \sr{u'}{1}{X_2}$, we have
    \begin{align*}
        \oc{X_1} = G_{X_1}\left(\sr{u'}{1}{X_1}\right) &\leq G_{X_1}\left(\sr{u'}{1}{X_2}\right) = \sr{u'}{1}{X_2} + \Exp\left[u(X_1-\sr{u'}{1}{X_2})\right] \\
        &\leq \sr{u'}{1}{X_2} + \Exp\left[u(X_2-\sr{u'}{1}{X_2})\right] = \oc{X_2},
    \end{align*}
    where the final inequality follows because $X_1 \leq X_2$ a.s. and $u$ is increasing. This proves monotonicity of $\oc{\cdot}$.
    \paragraph{Cash-invariance.}Let $X \in \xhu$ and $m \in \Rel$. Then,
    \begin{align*}
        \oc{X+m} &= \sr{u'}{1}{X+m} + \Exp\left[u(X-m-\sr{u'}{1}{X+m})\right] \\
        &= \sr{u'}{1}{X} + m + \Exp\left[u(X-\sr{u'}{1}{X})\right] = \oc{X} + m,
    \end{align*}
    where the second equality holds due to cash-invariance property of $\sr{u'}{1}{\cdot}$, cf. \cite[Proposition 2]{gupte2023optimization}. This proves cash-invariance of the $\oc{\cdot}$.
    \paragraph{Convexity.}Let $\alpha \in [0,1]$ and denote $X_\alpha \triangleq \alpha X_1 + (1-\alpha)X_2$. Then by \Cref{proposition:oce-sr-coincides}, $\sr{u'}{1}{X_\alpha}$ is a minimizer of $G_{X_\alpha}(\cdot)$, and hence $G_{X_\alpha}(\sr{u'}{1}{X_\alpha}) \leq G_{X_\alpha}(t),\forall t \in \Rel$. \\ Then, with $t=\alpha\sr{u'}{1}{X_1} + (1-\alpha)\sr{u'}{1}{X_2}$, we have
    \begin{align*}
        &\oc{X_\alpha} = G_{X_\alpha}(\sr{u'}{1}{X_\alpha}) \leq G_{X_\alpha}(\alpha\sr{u'}{1}{X_1} + (1-\alpha)\sr{u'}{1}{X_2}) \\
        &= \alpha\sr{u'}{1}{X_1} + (1-\alpha)\sr{u'}{1}{X_2} + \Exp\left[u(X_\alpha-\alpha\sr{u'}{1}{X_1} - (1-\alpha)\sr{u'}{1}{X_2})\right] \\
        &\leq \alpha\sr{u'}{1}{X_1} + (1-\alpha)\sr{u'}{1}{X_2} \\
        &+ \alpha \Exp\left[u(X_1-\sr{u'}{1}{X_1})\right] + (1-\alpha)\Exp\left[u(X_2-\sr{u'}{1}{X_2})\right] \\
        &= \alpha\oc{X_1} + (1-\alpha) \oc{X_2},
    \end{align*}
    where the first and second inequalities follow from the convexity of $G_{X_\alpha}$ and $u$, respectively. This proves the convexity of $\oc{\cdot}$.
\end{proof}

\subsection{Wasserstein Distance Bound on OCE.}\label{supp:Wasserstein}
The following result provides a bound on the difference between the OCE risk value of two r.v.s. $X$ and $Y$ in terms of the $2$-Wasserstein distance between the corresponding marginal distributions $\mu_X$ and $\mu_Y$. 




\begin{lemma}\label{lemma:oce-Wassertein-bound-I}
    Let $X,Y \in \xhu \cap \Leb{2}$, and suppose that \cref{assumption:u-main,assumption:u-prime-variance}, and the assumptions of \Cref{proposition:oce-sr-coincides} are satisfied. Then, 
    \begin{align*}
        \left| \oce - \oc{Y} \right|
         \le \mathcal{W}_2(\mu_X,\mu_Y)\sqrt{\sigma_1^2+1}.
    \end{align*}
\end{lemma}
\begin{proof}
Recall that $\mathcal{H}(\mu_X,\mu_Y)$ denotes the set of all joint distributions whose marginals are $\mu_X$ and $\mu_Y$. Note from the assumptions of the lemma that the r.v.s. $X,Y \in \xhu$ have finite $2^{\mathrm{nd}}$ moment. Then by definition of $\mathcal{W}_2$ as the infimum, it follows that for every $\epsilon>0$, there exists $\eta'(\epsilon) \in \mathcal{H}(\mu_X,\mu_Y)$ such that following holds.
    \begin{align}\label{eq:temporary-wasserstein}
        \mathcal{W}_2^2(\mu_X,\mu_Y) > \Exp_{\eta'(\epsilon)}\left[\left|X-Y\right|_2^2\right] - \epsilon.
    \end{align}
    Fix $\epsilon >0 $. W.L.O.G., consider the case where $\oce \geq \oc{Y}$. Then, we have 
    \begin{align*}
        &\left| \oce - \oc{Y} \right| =  \sr{u'}{1}{X} - \sr{u'}{1}{Y} + \Exp\left[ u\left(X-\sr{u'}{1}{X}\right) \right] - \Exp\left[u\left(Y-\sr{u'}{1}{Y}\right)  \right] \\
        &= \sr{u'}{1}{X} - \sr{u'}{1}{Y} + \Exp_{\eta'(\epsilon)}\left[ u\left(X-\sr{u'}{1}{X}\right) - u\left(Y-\sr{u'}{1}{Y}\right)  \right] \\
        &\leq  \sr{u'}{1}{X} - \sr{u'}{1}{Y}  +  \Exp_{\eta'(\epsilon)}\left[u'(X-\sr{u'}{1}{X})\left(\sr{u'}{1}{Y} -\sr{u'}{1}{X} + X - Y \right) \right] \\
        &\leq \Exp_{\eta'(\epsilon)}\left[u'(X-\sr{u'}{1}{X}) \left(X-Y\right) \right] \\
        &\leq \sqrt{\Exp\left[u'(X-\sr{u'}{1}{X})^2\right]\Exp_{\eta'(\epsilon)}\left[\left|X-Y\right|^2\right]} < \sqrt{\sigma_1^2+1}\sqrt{\mathcal{W}_2^2(\mu_X,\mu_Y)+\epsilon},
    \end{align*}    
    where the first inequality is due to the convexity of $u$. The second inequality follows from the fact that $\Exp\left[u'\left(X-\sr{u'}{1}{X}\right)\right]=1$ and the second inequality holds for any arbitrary choice of $\epsilon$. The third inequality follows from the Cauchy-Schwartz inequality. The final inequality follows from the variance assumption of the lemma and \cref{eq:temporary-wasserstein}. Since $\epsilon$ was chosen to be arbitrary, the claim of the lemma follows.
\end{proof}
In Lemma 12 of \citet{JMLR-LAP-SPB}, the authors obtained a bound similar to Lemma \ref{lemma:oce-Wassertein-bound-I} in terms of the 1-Wasserstein distance (between marginals $\mu_X$ and $\mu_Y$) under the assumption that the utility function is Lipschitz. In Lemma \ref{lemma:oce-Wassertein-bound-I}, we extend the result to possibly non-Lipschitz utility functions by replacing the Lipschitz assumption with a variance assumption and provide a bound in terms of the $2$-Wasserstein distance between the marginal distributions $\mu_X$ and $\mu_Y$. This extension covers  \cref{example:entropic,example:mean-variance,example:monotone-mean-variance,example:quartic}, which includes popular OCE instances with non-Lipschitz utility functions.
\section{OCE estimation}\label{appendix:estimation}




\subsection{Proof for Lemma \ref{lemma:oce-saa-lipschitz-bounds}}\label{proof:lemma-oce-saa-lipschitz-bounds}
\begin{proof}
    Let $\mu$ denote the distribution of $X$. Choose $\mathbf{z} \in \Rel^m$ and let $\mu_m(\mathbf{z})$ denote the measure having a mass $1/m$ at each of the points $\mathrm{z}_1,\mathrm{z}_1,\ldots,\mathrm{z}_m$. Further, let $\hat{X}_m(\mathbf{z})$ denote the r.v. which takes values $\mathrm{z}_1,\mathrm{z}_1,\ldots,\mathrm{z}_m$ with probability $1/m$ each. Then it is clear that $\mathbf{z}$ is a collection of samples of $\hat{X}_m(\mathbf{z})$ and $\mu_m(\mathbf{z})$ is the associated empirical measure. 
    
    Recall $\srm{\mathbf{z}}$ defined in (\ref{eq:sr-m-definition-oce}) and $G_X(t) \triangleq t + \Exp\left[u(X-t)\right]$. Then by \Cref{proposition:oce-sr-coincides}, $\sr{u'}{1}{X}$ is a minimizer of $G_X(\cdot)$ and therefore, $G_X(\sr{u'}{1}{X}) \leq G_X(\srm{\mathbf{z}})$, i.e.,  
    \begin{equation}\label{eq:G_X_inequality}
        \sr{u'}{1}{X} + \Exp\left[u(X-\sr{u'}{1}{X})\right] \leq \srm{\mathbf{z}} + \Exp\left[u(X-\srm{\mathbf{z}})\right].
    \end{equation}
    Consider a convex function $\hat{G}_\mathbf{z}(t) \triangleq t + \frac{1}{m}\sum_{j=1}^m u(\mathrm{z}_j-t)$. It is easy to see from (\ref{eq:sr-m-definition-oce}) that $\srm{\mathbf{z}}$ is a minimizer of $\hat{G}_\mathbf{z}(\cdot)$, and we have
    \begin{equation}\label{eq:G_X_hat_inequality}
        \srm{\mathbf{z}} +\frac{1}{m}\sum_{j=1}^m u(\mathrm{z}_j-\srm{\mathbf{z}}) \leq \sr{u'}{1}{X} + \frac{1}{m}\sum_{j=1}^m u(\mathrm{z}_j-\sr{u'}{1}{X}).
    \end{equation}
    We now derive the error bounds. Recall the definition of $\oce$ given by \Cref{proposition:oce-sr-coincides} and the definition of $\mathrm{OCE}_m$ given in (\ref{eq:oce-m-definition}). Consider the case when $\ocm{\mathbf{z}} \geq \oce$. Then, we have
    \begin{align*}
        &\ocm{\mathbf{z}}  - \oce = \srm{\mathbf{z}} -\sr{u'}{1}{X} + \frac{1}{m}\sum_{j=1}^m u\left(\mathrm{z}_j - \srm{\mathbf{z}}\right) - \Exp\left[u\left(X-\sr{u'}{1}{X}\right)\right] \\
        \numberthis \label{eq:oce-Lipschitz-hook}
        &\leq \frac{1}{m}\sum_{j=1}^m u(\mathrm{z}_j-\sr{u'}{1}{X}) - \Exp\left[ u\left(X - \sr{u'}{1}{X}\right)\right] \\
        &= \Exp\left[ u\left(-\hat{X}_m(\mathbf{z})-\sr{u'}{1}{X}\right) \right] - \Exp\left[u\left(X-\sr{u'}{1}{X}\right) \right] = K \left( \Exp\left[ u_1\left(\hat{X}_m(\mathbf{z})\right)\right] - \Exp\left[u_1\left(X\right) \right] \right),
    \end{align*}
    where $u_1(t) \triangleq K^{-1} u(-t-\sr{u'}{1}{X}), \forall t \in \Rel$, and the first inequality above follows from (\ref{eq:G_X_hat_inequality}). Since $u_1$ is $1$-Lipschitz, by Kantorovich–Rubinstein Theorem (cf. Section 1.8.2 of \citet{panaretos_invitation_2020}), $\ocm{\mathbf{z}}  - \oce$ is bounded above by $K \mathcal{W}_1(\mu_m(\mathbf{z}), \mu)$. For the other case: $\ocm{\mathbf{z}}  < \oce$, we have
    \begin{align*}
        &\oce  - \ocm{\mathbf{z}} =  \sr{u'}{1}{X} - \srm{\mathbf{z}} + \Exp\left[u\left(X-\sr{u'}{1}{X}\right)\right] - \frac{1}{m}\sum_{j=1}^m u\left(\mathrm{z}_j - \srm{\mathbf{z}}\right) \\
        &\leq \Exp\left[ u\left(X-\srm{\mathbf{z}}\right)\right] - \frac{1}{m}\sum_{j=1}^m u\left(\mathrm{z}_j - \srm{\mathbf{z}}\right) \\
        &= \Exp\left[ u\left(X-\srm{\mathbf{z}}\right) \right] - \Exp\left[u\left(-\hat{X}_m(\mathbf{z})-\srm{\mathbf{z}}\right) \right] = K \left( \Exp\left[u_2\left(X\right) \right] - \Exp\left[ u_2\left(\hat{X}_m(\mathbf{z})\right)\right] \right),
    \end{align*}
    where $u_2(t) \triangleq K^{-1} u(-t-\srm{\mathbf{z}}), \forall t \in \Rel$, and the first inequality follows from (\ref{eq:G_X_inequality}). Since $u_2$ is $1$-Lipschitz, by Kantorovich–Rubinstein Theorem (cf. Section 1.8.2 of \citet{panaretos_invitation_2020}), $\oce  - \ocm{\mathbf{z}}$ is bounded above by $K \mathcal{W}_1(\mu, \mu_m(\mathbf{z}))$.
    Combining the two cases, we have 
    \begin{equation*}
        \left| \ocm{\mathbf{z}}  - \oce \right| \leq K \mathcal{W}_1(\mu_m(\mathbf{z}), \mu).
    \end{equation*}
    Since $\mathbf{z}$ was chosen to be arbitrary, it follows that the above holds w.p. $1$ with $\mathbf{z}$ replaced by the random vector $\mathbf{Z}$. Then, 
    \begin{equation}\label{eq:oce-raw-bound-lipschitz}
        \left| \ocm{\mathbf{Z}}  - \oce \right| \leq K \mathcal{W}_1(\mu_m(\mathbf{Z}), \mu).
    \end{equation}
    Taking expectation on both sides of (\ref{eq:oce-raw-bound-lipschitz}), we have
    \begin{align*}
        \Exp\left[\left| \ocm{\mathbf{Z}}  - \oce \right| \right] \leq K \Exp\left[ \mathcal{W}_1(\mu_m(\mathbf{Z}), \mu)\right] \leq \frac{39 K T}{\sqrt{m}}.
    \end{align*}
    where the last inequality follows by invoking the Theorem 2.1 of \citet{fournier-2023}. With $p,d,m$ and $q$ from the notation of the Theorem 2.1 in \cite{fournier-2023}, we invoke the theorem with $p=1,d=1,m=1,q=3$ to get the following bound: $\Exp\left[\mathcal{W}_1(\mu_m(\mathbf{Z}), \mu)\right] = \Exp\left[\mathcal{T}_1(\mu_m(\mathbf{Z}), \mu)\right] \leq \frac{39T}{\sqrt{m}}$. In a similar manner, first squaring both sides in (\ref{eq:oce-raw-bound-lipschitz}) and then taking expectation on both sides, we have
    \begin{align*}
        \Exp\left[\left| \ocm{\mathbf{Z}}  - \oce \right|^2 \right]  ] &\leq  K^2 \Exp\left[\mathcal{W}_1^2(\mu_m(\mathbf{Z}), \mu)\right] \leq \Exp\left[\mathcal{W}_2^2(\mu_m(\mathbf{Z}), \mu)\right] \leq \frac{108K^2T^2}{\sqrt{m}}.
    \end{align*}
    The second inequality above follows from the monotonicity of the Wasserstein distance \citep[eq 2.1]{panaretos_invitation_2020}, while the last inequality follows by invoking the Theorem 2.1 of \cite{fournier-2023} with $p=2,d=1,m=1,q=5$ to get $\Exp\left[\mathcal{W}_2^2(\mu_m(\mathbf{Z}), \mu)\right] = \Exp\left[\mathcal{T}_2(\mu_m(\mathbf{Z}), \mu)\right] \leq \frac{108T^2}{\sqrt{m}}$. The bound on $\mathcal{T}_p$ applies if the higher-moment bound $\norm{X}_{\Leb{q}}\leq T$ is satisfied for some $q> 2p$. Therefore, for the MAE bound ($p=1)$, we assume that the higher-moment bound is satisfied for some $q>2$, and for the MSE bound ($p=2$), we assume that the higher-moment bound is satisfied for some $q>4$. This concludes the proof of \Cref{lemma:oce-saa-lipschitz-bounds}.
\end{proof}

\subsection{Proof for Lemma \ref{lemma:oce-saa-wasserstein-bound}}\label{proof:lemma-oce-saa-wasserstein-bound}
\begin{proof}Choose $\mathbf{z} \in \Rel^m$. Recall $\srm{\mathbf{z}}$ defined in (\ref{eq:sr-m-definition-oce}). Recall the definitions $\hat{X}_m(\mathbf{z})$ from the proof of \Cref{lemma:oce-saa-lipschitz-bounds} in \Cref{proof:lemma-oce-saa-lipschitz-bounds}. 
Consider the case: $\ocm{\mathbf{z}}  < \oce$, we have
    \begin{align*}
    \oce  - \ocm{\mathbf{z}}  &=  \sr{u'}{1}{X} - \srm{\mathbf{z}} + \Exp\left[u\left(X-\sr{u'}{1}{X}\right)\right] - \frac{1}{m}\sum_{j=1}^m u\left(\mathrm{z}_j - \srm{\mathbf{z}}\right)
    \end{align*}
Recall that $\mu$ denotes the distribution of $X$, and $\mu_m(\mathbf{z})$ is the empirical measure associated with r.v. $\hat{X}_m(\mathbf{z})$. Suppose $\mathcal{H}(\mu_m(\mathbf{z}),\mu)$ denotes the set of all joint distributions whose marginals are $\mu_m(\mathbf{z})$ and $\mu$. Then for every $\eta \in \mathcal{H}(\mu_m(\mathbf{z}),\mu)$, we have
\begin{align}\label{eq:oce-general-eta}
    \nonumber \ocm{\mathbf{z}}  - \oce \leq \sr{u'}{1}{X} - \srm{\mathbf{z}} + \Exp_{\eta}\left[u\left(X-\sr{u'}{1}{X}\right) - u\left(\hat{X}_m(\mathbf{z}) - \srm{\mathbf{z}}\right)\right].
\end{align}
From the higher moment assumption on $X$ and the definition of $\hat{X}_m(\mathbf{z})$, it is easy to see that the r.v.s. $X,\hat{X}_m(\mathbf{z}) \in \xhu$ and have finite $2^{\mathrm{nd}}$ moment. Then by definition of $\mathcal{W}_2$ as the infimum, it follows that for every $\epsilon>0$, there exists $\eta(\epsilon) \in \mathcal{H}(\mu_m(\mathbf{z}),\mu)$ such that following holds.
\begin{align}
    \mathcal{W}_2^2(\mu_m(\mathbf{z}),\mu) > \Exp_{\eta(\epsilon)}\left[\left|X-\hat{X}_m(\mathbf{z})\right|_2^2\right] - \epsilon.
\end{align}
Fix $\epsilon > 0$. Then, from (\ref{eq:oce-general-eta}) we have
    \begin{align*}
        &\oce  - \ocm{\mathbf{z}} =  \sr{u'}{1}{X} - \srm{\mathbf{z}} + \Exp_{\eta(\epsilon)}\left[u\left(X-\sr{u'}{1}{X}\right) - u\left(\hat{X}_m(\mathbf{z}) - \srm{\mathbf{z}}\right)\right] \\
        &\leq \sr{u'}{1}{X} - \srm{\mathbf{z}} + \Exp_{\eta(\epsilon)}\left[u'\left(X-\sr{u'}{1}{X}\right)\left(X - \sr{u'}{1}{X} - \hat{X}_m(\mathbf{z}) + \srm{\mathbf{z}}\right)\right] \\
        &= \Exp_{\eta(\epsilon)}\left[u'\left(X-\sr{u'}{1}{X}\right)\left(\hat{X}_m(\mathbf{z}) - X\right)\right]\\
        &\leq \sqrt{\Exp\left[ u'\left(X-\sr{u'}{1}{X}\right)^2 \right]} \sqrt{\Exp_{\eta(\epsilon)}\left[\left|X - \hat{X}_m(\mathbf{z})\right|^2 \right]} < \sqrt{\sigma_1^2+1}\sqrt{\mathcal{W}_2^2(\mu_m(\mathbf{z}),\mu) + \epsilon}.
\end{align*}
Since $\epsilon$ was chosen arbitrarily, we have
\begin{equation*}
    \ocm{\mathbf{z}}  - \oce \leq  \sqrt{\sigma_1^2+1} \sqrt{\mathcal{W}_2^2(\mu_m(\mathbf{z}),\mu)}.
\end{equation*}
For the case: $\ocm{\mathbf{z}}  \geq \oce$, using parallel arguments to those above, we get the same bound as above. Combining the two cases, we have
\begin{align}\label{eq:oce-w2-temp}
    \left| \ocm{\mathbf{z}}  - \oce \right| &\leq \sqrt{\sigma_1^2+1} \sqrt{\mathcal{W}_2^2(\mu_m(\mathbf{z}),\mu)}
\end{align}
Since $\mathbf{z}$ was chosen to be arbitrary, it follows that the above holds w.p. $1$ with $\mathbf{z}$ replaced by the random vector $\mathbf{Z}$. Then, replacing $\mathbf{z}$ with $\mathbf{Z}$ and taking expectation on both sides, we have
\begin{align*}
    \Exp\left[\left| \ocm{\mathbf{Z}}  - \oce \right|\right] \leq \left(\sqrt{\sigma_1^2+1}\right) \Exp\left[\sqrt{\mathcal{W}_2^2(\mu_m(\mathbf{Z}),\mu)}\right] \leq \sqrt{\sigma_1^2+1} \sqrt{\Exp\left[\mathcal{W}_2^2(\mu_m(\mathbf{Z}),\mu)\right]}.
\end{align*}
The final inequality follows from the Cauchy-Schwarz inequality. Using $\Exp\left[\mathcal{W}_2^2(\mu_m(\mathbf{Z}),\mu)\right] \leq \frac{108T}{\sqrt{m}}$ by the Theorem 2.1 of \cite{fournier-2023} yields the first bound of the lemma. The second bound of the lemma follows by a similar argument, by first squaring on both sides of \cref{eq:oce-w2-temp} and then taking the expectation.
\end{proof}

\subsection{Algorithm for OCE estimation}\label{sec:oce-algorithm}

\begin{algorithm}[H]
\SetKwInOut{Input}{Input}\SetKwInOut{Output}{Output}
\SetKwInOut{Define}{Define}
\SetKw{KwOutput}{Output} 
\SetAlgoLined
\caption{OCE-SB (Search and Bisect)}\label{alg:oce-saa-bisect}
\Input{thresholds $\delta,\epsilon > 0, \text{ i.i.d. samples } \{Z_i\}_{i=1}^m$}
\Define{$\hat{g}(t) \triangleq \frac{1}{m} \sum_{i=1}^m u'(-Z_i - t) - 1$}
 \leIf{$\hat{g}(0)>0$}{$low, high \leftarrow -1, 0 $}{$low, high \leftarrow 0, 1$}
 \lWhile{$\hat{g}(high) > 0$}{$high \leftarrow 2 * high$}
 \lWhile{$\hat{g}(low) < 0$}{$low \leftarrow 2 * low$}
 $T \leftarrow high$ - $low, \hat{t}_m \leftarrow (low+high)/2$\;
 \While{$T > 2\delta$ or $\left|\hat{g}(\hat{t}_m)\right|> \epsilon$}{
  \leIf{$\hat{g}(\hat{t}_m )>0$}{$low \leftarrow \hat{t}_m $}{$high \leftarrow \hat{t}_m$}
  $T \leftarrow high$ - $low, \hat{t}_m \leftarrow (low+high)/2$\;
 }
\KwOutput \textbf{ :} $\hat{t}_m$ 
\end{algorithm}

\begin{proposition}\label{proposition:oce-estimator-bounds} Suppose the utility function $u$ is continuously differentiable and satisfies the assumptions of \Cref{proposition:oce-sr-coincides}. Let $\hat{s}_m$ be an approximate solution to \cref{eq:sr-m-definition-oce} given by  \cref{alg:oce_estimation} with the inputs $\mathbf{Z}, \delta>0$  and $\epsilon>0$. , we have
    \begin{align*}
        \Exp[|\hat{s}_m(\mathbf{Z}) - \oce|] &\leq \delta \epsilon + \Exp\left[\left|\ocm{\mathbf{Z}} - \oce\right|\right], \;\; \text{and} \\ \Exp[\left(\hat{s}_m(\mathbf{Z}) - \oce\right)^2] &\leq 2\delta ^2 \epsilon^2 + 2\Exp\left[\left|\ocm{\mathbf{Z}} - \oce\right|^2\right].
    \end{align*}
\end{proposition}
\begin{proof}
    Choose $\mathbf{z} \in \Rel^m$, let $\hat{t}_m(\mathbf{z})$ denote the approximation of $\srm{\mathbf{z}}$ obtained by \Cref{alg:oce-saa-bisect} or any. Then the following holds.
    \begin{align}\label{eq:oce-estimation-algorithm-conditions}
        \left| \hat{t}_m(\mathbf{z}) - \srm{\mathbf{z}} \right| \leq \delta, \;\; \mathrm{ and } \;\; \left| \frac{1}{m} \sum_{j=1}^m u'\left(\mathrm{z}_j - \hat{t}_m(\mathbf{z}) \right) - 1 \right| \leq \epsilon.   
    \end{align}
    Using same $\mathbf{z}$, \cref{alg:oce_estimation} returns $\hat{s}_m(\mathbf{z})$ as an estimate for $\oce$, where
    \begin{align*}
        \hat{s}_m(\mathbf{z}) \triangleq \hat{t}_m(\mathbf{z}) + \frac{1}{m} \sum_{j=1}^m u'\left(\mathrm{z}_j - \hat{t}_m(\mathbf{z}) \right).
    \end{align*} 
    Suppose $\hat{s}_m(\mathbf{z}) \geq \oce$. Then we have
    \begin{align*}
        &\left|\hat{s}_m(\mathbf{z}) - \oce \right| 
        = \hat{t}_m(\mathbf{z}) - \sr{u'}{1}{X} + \frac{1}{m}\sum_{j=1}^m\left[u(\mathrm{z}_j-\hat{t}_m(\mathbf{z}))\right] - \Exp\left[u\left(X-\sr{u'}{1}{X}\right)\right] \\
        &=  \hat{t}_m(\mathbf{z}) - \srm{\mathbf{z}} + \srm{\mathbf{z}}  - \sr{u'}{1}{X} + \frac{1}{m}\sum_{j=1}^m\left[u(\mathrm{z}_j-\hat{t}_m(\mathbf{z}))\right] - \frac{1}{m}\sum_{j=1}^m\left[u(\mathrm{z}_j-\srm{\mathbf{z}} \right] \\
        &+ \frac{1}{m}\sum_{j=1}^m\left[u(\mathrm{z}_j-\srm{\mathbf{z}} \right] -\Exp\left[u\left(X-\sr{u'}{1}{X}\right)\right] \\
        &\leq \left[\srm{\mathbf{z}}  - \hat{t}_m(\mathbf{z})\right]\left[\frac{1}{m}\sum_{j=1}^m u'\left(\mathrm{z}_j - \hat{t}_m(\mathbf{z})\right) - 1\right] + \ocm{\mathbf{z}}  - \oce \leq \delta \epsilon + \ocm{\mathbf{z}}  - \oce.
    \end{align*}
    Here, the first inequality follows from the convexity of $u$, and the second inequality follows from \cref{eq:oce-estimation-algorithm-conditions}. For the other case: $\hat{s}_m(\mathbf{z}) < \oce$, we have 
    \begin{align*}
        &\left|\hat{s}_m(\mathbf{z}) - \oce\right| = \sr{u'}{1}{X} - \hat{t}_m(\mathbf{z}) + \Exp\left[u\left(X-\sr{u'}{1}{X}\right)\right] - \frac{1}{m}\sum_{j=1}^m\left[u(\mathrm{z}_j-\hat{t}_m(\mathbf{z}))\right] \\
        &= \sr{u'}{1}{X} - \srm{\mathbf{z}}  + \srm{\mathbf{z}}  - \hat{t}_m(\mathbf{z}) + \Exp\left[u\left(X-\sr{u'}{1}{X}\right)\right] \\
        &- \frac{1}{m}\sum_{j=1}^m\left[u(\mathrm{z}_j-\srm{\mathbf{z}} \right] + \frac{1}{m}\sum_{j=1}^m\left[u(\mathrm{z}_j-\srm{\mathbf{z}} \right] - \frac{1}{m}\sum_{j=1}^m\left[u(\mathrm{z}_j-\hat{t}_m(\mathbf{z}))\right] \\
        &\leq \oce - \ocm{\mathbf{z}}  + \srm{\mathbf{z}}  - \hat{t}_m(\mathbf{z})  + \frac{1}{m}\sum_{j=1}^m\left[u'\left(\mathrm{z}_j-\srm{\mathbf{z}} \right)\left(\hat{t}_m(\mathbf{z}) - \srm{\mathbf{z}} \right) \right] = \oce - \ocm{\mathbf{z}},
    \end{align*}
    where the last equality follows from the continuous differentiability of $u$. Indeed, if $u'$ is continuous, then \cref{eq:sr-m-definition-oce} implies that $ \frac{1}{m} \sum_{j=1}^m u'(\textrm{z}_j-\srm{\mathbf{z}}) = 1$. Combining the two cases, we have
    \begin{align*}
        \left|\hat{s}_m(\mathbf{z}) - \oce \right| &\leq \delta\epsilon + \left|\oc{\mathbf{z}}  - \oce \right| \\
        \left|\hat{s}_m(\mathbf{z}) - \oce \right|^2 &\leq 2\delta^2\epsilon^2 + 2\left|\oc{\mathbf{z}} - \oce \right|^2.
    \end{align*}
    Since $\mathbf{z}$ was chosen to be arbitrary, it follows that the above holds w.p. $1$ with $\mathbf{z}$ replaced by the random vector $\mathbf{Z}$. Then, taking the expectation on both sides, the proposition's claims follow.
\end{proof}
\Cref{proposition:oce-estimator-bounds} extends the bounds from \cref{lemma:oce-saa-wasserstein-bound,lemma:oce-saa-lipschitz-bounds} to the solution given by \cref{alg:oce_estimation}. Depending on the choice of the utility function $u$, \Cref{proposition:oce-estimator-bounds} may be invoked in tandem with one of the lemmas from \Cref{lemma:oce-saa-wasserstein-bound,lemma:oce-saa-lipschitz-bounds} and the values for $\delta$ and $\epsilon$ can be chosen to match the respective error rates of the lemma. For example, suppose the assumptions of \Cref{lemma:oce-saa-lipschitz-bounds} hold for some $T>0$ and $K>0$. Then invoking \Cref{proposition:oce-estimator-bounds} with $\delta=1/\sqrt{m}$ and $\epsilon=d_2$ would yield the following MAE and MSE bounds on the estimator $\hat{s}_m$ given by \cref{alg:oce_estimation}:
\begin{align*}
        \Exp\left[\left| \hat{s}_m(\mathbf{Z}) - \oce \right| \right] \le \frac{d_2+39K T}{\sqrt{m}}, \,\, \text{ and } \;\;
        \Exp\left[\left| \hat{s}_m(\mathbf{Z}) - \oce \right|^2 \right]  \le \frac{2d_2^2+216K^2 T^2}{\sqrt{m}}.
\end{align*}
\begin{remark}
\Cref{proposition:oce-estimator-bounds} employs the assumption that $u$ is continuously differentiable, which is satisfied by all examples except CVaR in \Cref{example:cvar}. Notice that the continuous differentiability of $u$ is only employed in case $2$, i.e., $\hat{s}_m(\mathbf{z}) < \oce$. In such cases, we can avoid relying on the aforementioned assumption by using an alternative proof technique. We partition case $2$ further into three sub-cases: i) $\hat{s}_m(\mathbf{z}) < \oce \leq \srm{\mathbf{z}}$ ii) $\hat{s}_m(\mathbf{z}) <  \srm{\mathbf{z}} < \oce$, and iii) $\srm{\mathbf{z}} \leq \hat{s}_m(\mathbf{z}) < \oce$. It is easy to see that this leads to an upper bound of $\epsilon + \left|\oc{\mathbf{z}}  - \oce \right|$ for case $2$, and combined with case 1, gives the bound: $(\delta+1)\epsilon + \left|\oc{\mathbf{z}}  - \oce \right|$. Thus, even when $u$ is not continuously differentiable, we obtain a bound similar to \Cref{proposition:oce-estimator-bounds}, with the exception that $\delta$ is replaced with $\delta+1$.
\end{remark}
    
\section{OCE Optimization}\label{appendix:optimization}

\subsection{Proof for Theorem \ref{theorem-oce-gradient}}\label{proof:theorem-oce-gradient}
\begin{proof}
     Recall that $\srth{} \triangleq \sr{u'}{1}{F(\theta,\xi)}$. Define $h:\Theta \times \Rel \to \Rel$ as $h(\theta, t) = \Exp\left[u'\left(F(\theta,\xi)-t\right)\right] - 1$. Note that i) $h(\theta, \srth{}) = \Exp\left[u'\left(F(\theta,\xi)-\srth{}\right)\right] - 1 = 0$ for every $\theta \in \Theta$. Since $u$ is twice differentiable and $F$ is continuously differentiable, it implies that ii) $h$ is differentiable on $\Theta \times \Rel$. Furthermore, by the theorem assumption of $\forall \theta \in \Theta, P(u''\left(F(\theta,\xi)-\srth{}\right)>0)>0$, we have $\Exp\left[u''\left(F(\theta,\xi)-\srth{}\right)\right]$ for every $\theta \in \Theta$. This implies that iii) $\partial h / \partial t < 0$ is satisfied at every $(\theta, \srth{})$. Then (i),(ii), and (iii) satisfy the conditions of Theorem 1 of \cite{hurwicz2003implicit} and therefore, $\srth{}$ is differentiable, and the partial derivatives are given by 
     \begin{equation*}
         \frac{\partial \srth{}}{\partial \theta_i} = \left.\frac{\partial h / \partial \theta_i}{\partial h / \partial t}\right|_{(\theta,\srth{})} \quad \forall i \in 1,\ldots,d.
     \end{equation*}
     We do not expand the r.h.s. of the above equation as the partial derivatives of $\srth{}$ get canceled in the following passage. Thus, we require only their existence. Recall the expression for $\octh{}$ in \cref{eq:oce-expression-compact}. Taking partial derivatives w.r.t. $\theta_i$ on both sides of \cref{eq:oce-expression-compact}, we have     
    \begin{align*}
        \frac{\partial \octh{}}{\partial \theta_i} &= \frac{\partial \srth{}}{\partial \theta_i} - \Exp\left[u'\left(F(\theta,\xi)+\srth{}\right)\left(\frac{\partial F(\theta,\xi)}{\partial \theta_i}  - \frac{\partial \srth{}}{\partial \theta_i}\right)\right] \\
        &= \frac{\partial \srth{}}{\partial \theta_i}\left[1 - \Exp\left[u'\left(F(\theta,\xi)-\srth{}\right)\right]\right] + \Exp\left[u'\left(F(\theta,\xi)-\srth{}\right)\frac{\partial F(\theta,\xi)}{\partial \theta_i}\right] \\
        &= \Exp\left[u'\left(F(\theta,\xi)-\srth{}\right)\frac{\partial F(\theta,\xi)}{\partial \theta_i}\right],
    \end{align*}
where the final equality follows (i). Since $u$ and $F$ are continuously differentiable, the r.h.s. of the final equality is continuously differentiable, implying that the gradient of $\octh{}$ exists and the claim of the theorem follows.
\end{proof}

\subsection{Proof of Lemma \ref{lemma:oce-smooth}}\label{proof:oce-smooth}
\Cref{as:F-smooth} implies that $F$ and therefore $-F$ is $L$-smooth a.s. Then by the Theorem 2.1.5 of \cite{nesterov_introductory_2004}, for every $\theta_1,\theta_2 \in \Theta$, we have
\begin{equation}\label{eq:f-smoothness-inequality}
    \left(-F(\theta_2,\xi)\right) - \left(-F(\theta_1,\xi)\right) \leq \left(-\nabla F(\theta_1,\xi)\right)^T\left(\theta_2-\theta_1\right) + \frac{L}{2}\norm{\theta_1-\theta_2}^2 \;\; \text{a.s.}
\end{equation} 
Recall that $\octh{} = \srth{} + \Exp\left[u\left(F\left(\theta,\xi\right) - \srth{}\right)\right]$. Then, for every $\theta_1,\theta_2 \in \B$, we have
\begin{align*}
    \octh{1} - \octh{2} &= \srth{1} - \srth{2} + \Exp\left[u\left(F\left(\theta_1,\xi\right) - \srth{1}\right)\right] - \Exp\left[u\left(F\left(\theta_2,\xi\right) - \srth{2}\right)\right] \\
    &\leq \srth{1} - \srth{2} + \Exp\left[u'\left(F\left(\theta_1,\xi\right)- \srth{1}\right) \left(\srth{2} - \srth{1} + F\left(\theta_1,\xi\right) - F\left(\theta_2,\xi\right)\right)\right] \\
    &= \Exp\left[u'\left(F\left(\theta_1,\xi\right)- \srth{1}\right) \left(F\left(\theta_2,\xi\right) - F\left(\theta_1,\xi\right)\right)\right] \\
    &\leq \Exp\left[u'\left(F\left(\theta_1,\xi\right)- \srth{1}\right) \left(\nabla F(\theta_1,\xi)^T\left(\theta_1-\theta_2\right)+\frac{L}{2}\norm{\theta_1-\theta_2}^2\right)\right] \\
    &= \nabla \octh{1}^T(\theta_1-\theta_2)+\frac{L}{2}\norm{\theta_1-\theta_2}^2.
\end{align*}
Here, the first inequality follows from the convexity of $u$, while the second equality follows by definition of $\srth{1}$, which for a continuously differentiable $u$ implies that $\Exp\left[u'\left(F\left(\theta_1,\xi\right)- \srth{1}\right)\right]=1$. The last inequality follows from \cref{eq:f-smoothness-inequality}. Then by the Theorem 2.1.5 of \cite{nesterov_introductory_2004}, $\ocet(\cdot)$ is $L$-smooth.

\subsection{Proof of Lemma \ref{lemma:oce-strong-convexity}}\label{proof:oce-strong-convexity}
From the assumptions of the lemma, we note that $F$ is continuously differentiable and $\mu$-strongly convex w.p. 1. Then by \cite[Definition 2.1.2]{nesterov_introductory_2004}, and $F$ is $\mu$-strongly convex w.p. $1$, i.e., for every $\theta_1,\theta_2 \in \Theta$, we have
\begin{equation}\label{eq:f-strong-convexity}
    F(\theta_1,\xi) \geq F(\theta_2,\xi) + \nabla F(\theta_2,\xi)^T\left(\theta_1-\theta_2\right) + \frac{\mu}{2}\norm{\theta_1-\theta_2}^2, \qquad \text{w.p. 1}.
\end{equation} 
Recall that $\octh{} = \srth{} + \Exp\left[u\left(F\left(\theta,\xi\right) - \srth{}\right)\right]$. Then, for every $\theta_1,\theta_2 \in \B$, we have
\begin{align*}
    \octh{1} - \octh{2} &= \srth{1} - \srth{2} + \Exp\left[u\left(F\left(\theta_1,\xi\right) - \srth{1}\right)\right] - \Exp\left[u\left(F\left(\theta_2,\xi\right) - \srth{2}\right)\right] \\
    &\geq \srth{1} - \srth{2} + \Exp\left[u'\left(F\left(\theta_2,\xi\right)- \srth{2}\right) \left(\srth{2} - \srth{1} + F\left(\theta_1,\xi\right) - F\left(\theta_2,\xi\right)\right)\right] \\
    &= \Exp\left[u'\left(F\left(\theta_2,\xi\right)- \srth{2}\right) \left(F\left(\theta_2,\xi\right) - F\left(\theta_1,\xi\right)\right)\right] \\
    &\geq \Exp\left[u'\left(F\left(\theta_2,\xi\right)- \srth{2}\right) \left(\nabla F(\theta_2,\xi)^T\left(\theta_1-\theta_2\right) + \frac{\mu}{2}\norm{\theta_1-\theta_2}^2\right)\right] \\
    &= \nabla \octh{2}^T\left(\theta_1-\theta_2\right) + \frac{\mu}{2}\norm{\theta_1-\theta_2}^2.
\end{align*}
The first inequality follows from convexity of $u$, while the second equality follows because by the definition of $\srth{2}$  which implies that $\Exp\left[u'\left(F\left(\theta_2,\xi\right) - \srth{2}\right)\right]=1$ holds for a continuously differentiable utility function $u$. The second inequality follows from \cref{eq:f-strong-convexity} while the final equality follows from the gradient expression of $\ocet$. Then by \cite[Definition 2.1.2]{nesterov_introductory_2004}, $\ocet$ is $\mu$-strongly convex.

\subsection{Proof for Lemma \ref{lemma:oce-gradient-estimator-I}}\label{proof:lemma-oce-gradient-estimator-I}
\begin{proof}
     Fix $m \in \mathcal{N}, \mathbf{z} \in \Rel^m$ and $\theta \in \B$. Recall that $SR^m_\theta(\mathrm{z})$ and $Q^m_\theta(\mathbf{z})$ are estimators of $\srth{}$ and $\nabla \octh{}$ respectively. Then, the following holds for $q \in \{1,2\}$.
     \begin{align*}
        &\norm{Q^m_\theta(\mathbf{z}) - \nabla \octh{}}_2^q = \norm{\frac{1}{m}\sum_{j=1}^mu'\left(F(\theta,\mathrm{z}_j) - SR^m_\theta({\mathrm{z}})\right)\nabla F(\theta,\mathrm{z}_j) - \Exp\left[u'\left(F(\theta,\xi)-\srth{}\right)\nabla F(\theta,\xi)\right]}_2^q \\
        &\leq q\norm{\frac{1}{m}\sum_{j=1}^mu'\left(F(\theta,\mathrm{z}_j) - SR^m_\theta({\mathrm{z}})\right)\nabla F(\theta,\mathrm{z}_j) - \frac{1}{m}\sum_{j=1}^mu'\left(F(\theta,\mathrm{z}_j) - \srth{}\right)\nabla F(\theta,\mathrm{z}_j)}_2^q \\
        &+ q\norm{\frac{1}{m}\sum_{j=1}^mu'\left(F(\theta,\mathrm{z}_j) - \srth{}\right)\nabla F(\theta,\mathrm{z}_j) - \Exp\left[u'\left(F(\theta,\xi)-\srth{}\right)\nabla F(\theta,\xi)\right]}_2^q \\
        \numberthis \label{eq:Q-grad-temp-1}
        &= qI_1 + q\norm{\frac{1}{m}\sum_{j=1}^mu'\left(F(\theta,\mathrm{z}_j) - \srth{}\right)\nabla F(\theta,\mathrm{z}_j) - \Exp\left[u'\left(F(\theta,\xi)-\srth{}\right)\nabla F(\theta,\xi)\right]}_2^q.
    \end{align*}
    For bounding $I_1$, we note the following.
    \begin{align*}
        I_1 &= \norm{\frac{1}{m}\sum_{j=1}^m \left[u'\left(F(\theta,\mathrm{z}_j) - \srth{}\right) - u'\left(F(\theta,\mathrm{z}_j) - SR^m_\theta({\mathrm{z}})\right)\right]\nabla F(\theta,\mathrm{z}_j)}_2^q \\
        &\leq \max_j \norm{\nabla F(\theta,\mathrm{z}_j)}_2^q \left[\frac{1}{m}\sum_{j=1}^m \left|u'\left(F(\theta,\mathrm{z}_j) - \srth{}\right) - u'\left(F(\theta,\mathrm{z}_j) - SR^m_\theta({\mathrm{z}})\right)\right|\right]^q \\
        \numberthis \label{eq:two-case-equality}
        &= \max_j \norm{\nabla F(\theta,\mathrm{z}_j)}_2^q \left|\frac{1}{m}\sum_{j=1}^m u'\left(F(\theta,\mathrm{z}_j) - \srth{}\right) - u'\left(F(\theta,\mathrm{z}_j) - SR^m_\theta({\mathrm{z}})\right)\right|^q \\
        &= \max_j \norm{\nabla F(\theta,\mathrm{z}_j)}_2^q \left|\frac{1}{m}\sum_{j=1}^m u'\left(F(\theta,\mathrm{z}_j) - \srth{}\right) - 1\right|^q.
    \end{align*}
    For obtaining the first inequality, we first apply the Cauchy Schwartz inequality, then replace $\norm{\nabla F(\theta,\mathrm{z}_j)}_2^q$ with its maximum over $j$ and finally take the max term outside the summation. The last equality follows from \cref{eq:SR-definition-oce}. The equality in \cref{eq:two-case-equality} is unusual, and we provide the following justification. Consider the case: $SR^m_\theta({\mathrm{z}}) \geq \srth{}$. Then, 
    \begin{equation*}
        \sum_{j=1}^m \left|u'\left(F(\theta,\mathrm{z}_j) - \srth{}\right) - u'\left(F(\theta,\mathrm{z}_j) - SR^m_\theta({\mathrm{z}})\right)\right| = \sum_{j=1}^m u'\left(F(\theta,\mathrm{z}_j) - \srth{}\right) - u'\left(F(\theta,\mathrm{z}_j) - SR^m_\theta({\mathrm{z}})\right),
    \end{equation*} 
    where we use the fact that $u'$ is increasing, which follows from convexity of $u$. For the other case: $SR^m_\theta({\mathrm{z}}) < \srth{}$, we note that following holds. 
    \begin{equation*}
        \sum_{j=1}^m \left|u'\left(F(\theta,\mathrm{z}_j) - \srth{}\right) - u'\left(F(\theta,\mathrm{z}_j) - SR^m_\theta({\mathrm{z}})\right)\right| = \sum_{j=1}^m u'\left(F(\theta,\mathrm{z}_j) - SR^m_\theta({\mathrm{z}}) - u'\left(F(\theta,\mathrm{z}_j) - \srth{}\right)\right).
    \end{equation*} 
    Combining the two cases conclude the proof of the equality in \cref{eq:two-case-equality}. Substituting the bound on $I_1$ back into \cref{eq:Q-grad-temp-1}, we have
    \begin{align}
        \nonumber
        \norm{Q^m_\theta(\mathbf{z}) - \nabla \octh{}}_2^q &\leq q \max_j \norm{\nabla F(\theta,\mathrm{z}_j)}_2^q \left|\frac{1}{m}\sum_{j=1}^m u'\left(F(\theta,\mathrm{z}_j) - \srth{}\right) - 1\right|^q \\
        \label{eq:oce-grad-common}
        &+ q\norm{\frac{1}{m}\sum_{j=1}^mu'\left(F(\theta,\mathrm{z}_j) - \srth{}\right)\nabla F(\theta,\mathrm{z}_j) - \Exp\left[u'\left(F(\theta,\xi)-\srth{}\right)\nabla F(\theta,\xi)\right]}_2^q.
    \end{align}
    The above holds for all $z \in \Rel^m$, and therefore it holds almost surely with $\mathrm{z}$ replaced with $\Z$. Applying \Cref{assumption:F-gradient-bound-II} and taking expectation on both sides, we have
    \begin{align*}
        \Exp\left[\norm{Q^m_\theta(\Z) - \nabla \octh{}}_2^q\right] &\leq q M_2^q \Exp\left[\left|\frac{1}{m}\sum_{j=1}^m u'\left(F(\theta,\Z_j) - \srth{}\right) - 1\right|^q\right] \\
        &+ q \Exp\left[\norm{\frac{1}{m}\sum_{j=1}^mu'\left(F(\theta,Z_j) - \srth{}\right)\nabla F(\theta,\Z_j) - \Exp\left[u'\left(F(\theta,\xi)-\srth{}\right)\nabla F(\theta,\xi)\right]}_2^q\right].
    \end{align*}
    Next, we substitute $1$ in the above inequality with $\Exp\left[\frac{1}{m}\sum_{j=1}^m u'\left(F(\theta,\Z_j) - \srth{}\right)\right]$ because $\Z_j$ are i.i.d. and by definition of $\srth{}=\sr{u'}{1}{F(\theta,\xi)}$, we have $\Exp\left[ u'\left(F(\theta,\xi) - \srth{}\right)\right]=1$. Then, we have
    \begin{align*}
        \Exp&\left[\norm{Q^m_\theta(\Z) - \nabla \octh{}}_2^q\right] \leq q M_2^q \Exp\left[\left|\frac{1}{m}\sum_{j=1}^m u'\left(F(\theta,\Z_j) - \srth{}\right) - 1\right|^q\right] \\
        &+ q \Exp\left[\norm{\frac{1}{m}\sum_{j=1}^mu'\left(F(\theta,Z_j) - \srth{}\right)\nabla F(\theta,\Z_j) - \Exp\left[\frac{1}{m}\sum_{j=1}^mu'\left(F(\theta,Z_j) - \srth{}\right)\nabla F(\theta,\Z_j)\right]}_2^q\right].
    \end{align*}
    We note that both the terms on the r.h.s. of the above inequality are variances of summations of $m$ i.i.d. terms. Then, replacing variance of a sum with the sum of variances, we have 
    \begin{align*}
        \Exp&\left[\norm{Q^m_\theta(\Z) - \nabla \octh{}}_2^q\right] \leq \frac{q M_2^q}{m^{q/2}}  Var\left(u'\left(F(\theta,\xi) - \srth{}\right)\right) + \frac{q}{m^{q/2}} \Exp\left[\norm{\mathbf{Y}_\theta- \Exp\left[\mathbf{Y}_{\theta}\right]}_2^q\right],
    \end{align*}
    where $\mathbf{Y}_\theta \triangleq u'\left(F(\theta,\xi)-\srth{}\right)\nabla F(\theta,\xi)$. Applying \cref{assumption:u-prime-variance,assumption:u-F-variance-bound}, the claim of the lemma follows.
\end{proof}

\subsection{An alternative to Lemma \ref{lemma:oce-gradient-estimator-I}}\label{proof:lemma-oce-gradient-estimator-II}
Some optimization objectives (see \Cref{example:portfolio-optimization}) do not satisfy \Cref{assumption:F-gradient-bound-II}. We derive similar rates as given in Lemma \ref{lemma:oce-gradient-estimator-I} using the following relaxed assumption.
\begin{assumption}\label{assumption:F-gradient-bound}
There exists $S_1>0,S_2>0$ such that for every $\theta \in \Theta$, $\norm{\nabla F(\theta, \xi)}_{2}^2$ is sub-Gaussian and satisfies the following for $m$ i.i.d. copies $\{\xi_j\}_{j=1}^m$of $\xi$:
\begin{align*}
    \Exp\left[\max_j \norm{\nabla F(\theta,\xi_j)}_2^2\right] \leq S_1 \sqrt{\log(m)} \;\; \text{and}\;\;
    \Exp\left[\max_j \norm{\nabla F(\theta,\xi_j)}_2^4\right] \leq S_2 \sqrt{\log(m)}.
\end{align*}
\end{assumption}
The inequalities above follow from the sub-Gaussianity of $\norm{\nabla F(\theta, \xi)}_{2}^2$, which also implies that $\norm{\nabla F(\theta, \xi)}_{2}^4$ is sub-exponential. For a precise value of the constants $S_1,S_2$, see \cite[Chapter 2]{wainwright_high-dimensional_2019} or \cite[Chapter 3]{vershynin_high-dimensional_2018}. A variant of Lemma \ref{lemma:oce-gradient-estimator-I} with \Cref{assumption:F-gradient-bound} is presented below. 
\begin{lemma}\label{lemma:oce-gradient-estimator-II}
    Suppose the assumptions of \Cref{theorem-oce-gradient} and \cref{assumption:u-prime-variance-theta,assumption:u-F-variance-bound,assumption:F-gradient-bound} are satisfied. Then for every $m \in \N$ and $\theta \in \B \subseteq \Rel^d$, we have
    \begin{align*}
\Exp\left[\norm{Q^m_\theta(\mathbf{Z}) - \nabla \octh{}}_2\right] &\leq \frac{\sigma_1\sqrt{S_1 \log(m)}+T}{\sqrt{m}}
    \end{align*}
    In addition, if there exists $L_1,M_1>0$ such that $u'$ is $L_1$-Lipschitz and for every $\theta \in \B, \Exp\left[\left|F(\theta,\xi)-\Exp\left[F(\theta,\xi)\right]\right|^4\right] \leq M_1$, then
    \begin{align*}
        \Exp\left[\norm{Q^m_\theta(\mathbf{Z}) - \nabla \octh{}}_2^2\right] &\leq \left(\sigma_1S_1+2\sqrt{2}S_2^2L_1^3 M_1^3\right) \frac{\log(m)}{\sqrt{m}}+\frac{2T^2}{m}.
    \end{align*}
\end{lemma}

The above MAE and MSE bounds can be shown to hold when $Q^m_\theta(\mathbf{Z})$ is constructed by replacing $SR^m_{\theta}(\mathbf{{z}})$ with an approximation of $SR^m_{\theta}(\mathbf{\hat{z}})$.
\begin{proof}
    Fix $m \in \mathcal{N}, \mathrm{z} \in \Rel^m$ and $\theta \in \B$. The proof of this lemma is identical to the proof of \Cref{lemma:oce-gradient-estimator-I} up to \cref{eq:oce-grad-common}. The treatment of the second term on the r.h.s. of \cref{eq:oce-grad-common} is also identical, and we omit the details. We focus on bounding the first term on the r.h.s. of \cref{eq:oce-grad-common}, i.e. $I_1 \triangleq q \max_j \norm{\nabla F(\theta,\mathrm{z}_j)}_2^q \left|\frac{1}{m}\sum_{j=1}^m u'\left(F(\theta,\mathrm{z}_j) - \srth{}\right) - 1\right|^q$. Replacing $\mathrm{z}$ in $I_1$ with $\mathbf{Z}$, putting $q=1$, and then taking expectation, we have,
    \begin{align*}
        \Exp&\left[\max_j \norm{\nabla F(\theta,\mathrm{Z}_j)}_2 \left|\frac{1}{m}\sum_{j=1}^m u'\left(F(\theta,\mathrm{Z}_j) - \srth{}\right) - 1\right|\right] \\
        &\leq \sqrt{\Exp\left[\max_j \norm{\nabla F(\theta,\mathrm{Z}_j)}_2^2\right]}\sqrt{\Exp\left[\left|\frac{1}{m}\sum_{j=1}^m u'\left(F(\theta,\mathrm{Z}_j) - \srth{}\right) - 1\right|^2\right]} \leq \sqrt{S_1 \sqrt{\log(m)}} \frac{\sigma_1}{\sqrt{m}}.
    \end{align*}
    For the second claim, we replace $\mathrm{z}$ in $I_1$ with $\mathbf{Z}$, put $q=2$ and then take expectation to have
    \begin{align*}
        \Exp&\left[\max_j \norm{\nabla F(\theta,\mathrm{Z}_j)}_2^2 \left|\frac{1}{m}\sum_{j=1}^m u'\left(F(\theta,\mathrm{z}_j) - \srth{}\right) - 1\right|^2\right] \\
        &\leq \Exp\left[\max_j \norm{\nabla F(\theta,\mathrm{Z}_j)}_2^4\right]\Exp\left[\left|\frac{1}{m}\sum_{j=1}^m u'\left(F(\theta,\mathrm{Z}_j) - \srth{}\right) - 1\right|^4\right] \\
        &\leq S_2^2 \sqrt{\log(m)} \Exp\left[\left|\frac{1}{m}\sum_{j=1}^m u'\left(F(\theta,\mathrm{Z}_j) - \srth{}\right) - 1\right|\left|\frac{1}{m}\sum_{j=1}^m u'\left(F(\theta,\mathrm{Z}_j) - \srth{}\right) - 1\right|^3\right] \\ 
        &\leq S_2^2 \sqrt{\log(m)} \sqrt{\Exp\left[\left|\frac{1}{m}\sum_{j=1}^m u'\left(F(\theta,\mathrm{Z}_j) - \srth{}\right) - 1\right|^2\right]}\sqrt{\Exp\left[\left|\frac{1}{m}\sum_{j=1}^m u'\left(F(\theta,\mathrm{Z}_j) - \srth{}\right) - 1\right|^6\right]} \\
        & \leq S_2^2 \sqrt{\log(m)}\frac{\sigma_1}{\sqrt{m}}\sqrt{\Exp\left[\left|\frac{1}{m}\sum_{j=1}^m u'\left(F(\theta,\mathrm{Z}_j) - \srth{}\right) - \Exp\left[u'\left(F(\theta,\xi)-\srth{}\right)\nabla F(\theta,\xi)\right]\right|^6\right]} \\
        & \leq S_2^2 \sqrt{\log(m)}\frac{\sigma_1}{\sqrt{m}} \sqrt{2} L_1^3 2 M_1^3.
    \end{align*}
\end{proof}

\clearpage
\section{Algorithm for OCE optimization and its convergence rates}\label{supp:algorithm-and-convergence-rates}
\begin{algorithm}[h]
\caption{OCE-SG}\label{alg:oce-minimization}
\SetKwInOut{Input}{Input}\SetKwInOut{Output}{Output}
\SetAlgoLined
\Input{$\theta_0 \in \Theta$, batch sizes $\{m_k\}_{k \ge 1}$, step sizes $\{\alpha_k\}_{k \ge 1}$, and number of iterations $n$.}
\For{$k= 1, 2, \ldots, n$}{
    sample $\mathbf{Z}^k = [{Z}^k_1, {Z}^k_2, \ldots, {Z}^k_{m_k} ]$ from $\xi$\; 
    compute $SR^{m_k}_{\theta}(\mathbf{{Z^k}})$ using \cref{eq:SR-definition-oce}\;
    compute ${Q}_{\theta_{k-1}}^{m_k}(\mathbf{Z}^k)$ using \cref{eq:oce-m-definition}\;
    update $\theta_k \leftarrow \Pi_\Theta\left(\theta_{k-1} - \alpha_k {Q}_{\theta_{k-1}}^{m_k}(\mathbf{Z}^k)\right)$\;
}
\end{algorithm}

\subsection{Strongly-convex case }
The following result establishes a non-asymptotic bound on the last iterate of \Cref{alg:oce-minimization}. 
\begin{theorem}\label{theorem:oce-sgd}
Suppose the assumptions of \Cref{theorem-oce-gradient} and \cref{assumption:u-prime-variance,assumption:F-gradient-bound-II,as:F-smooth,assumption:u-F-variance-bound} are satisfied. Suppose the OCE objective $\ocet(\cdot)$ is $\mu$-strongly convex on $\Theta$ for some $\mu>0$. Run \Cref{alg:oce-minimization} with $\alpha_k = \frac{c}{k},m_k=k,\,\forall k$, where $c>\frac{1}{\mu}$. Then,
    \begin{align*}
        \Exp&\left[\norm{\theta_n - \theta_*}_2^2\right] \leq\frac{\Exp\left[\norm{\theta_0 - \theta_*}_2^2\right]+K_1}{(n+1)^{2}} + \frac{K_2 \Exp\left[ \norm{\theta_0 - \theta_*}_2\right]}{(n+1)^{3/2}} + \frac{K_3}{n+1},
    \end{align*}
    where the constants $K_1,K_2$ and $K_3$ are independent of $n$.
\end{theorem}
\begin{proof}
We split the proof into three parts. In Part I, we derive some intermediate results that are applied in the later parts of the proof. In part II, we derive an MAE bound on the last iterate of the SG algorithm, whereas in part III, we derive an MSE bound on the last iterate of the SG algorithm.
\paragraph{Part I.} Recall that the objective function $\ocet(\cdot)$ is $\mu$-strongly convex and $L$-smooth. $\theta_0$ is chosen arbitrarily and $\left\{\theta_1,\theta_2,\ldots,\theta_n\right\}$ are the random iterates of the SG algorithm given in \Cref{alg:oce-minimization}. Using the shorthand notation $z_k \triangleq \theta_k - \theta_*$, the following holds w.p. $1$.
\begin{align*}
    \norm{z_{n-1} - \alpha_n \nabla \octh{n-1}}_2^2  &= \norm{z_{n-1}}_2^2 + \alpha_n^2 \norm{\nabla \octh{n-1}}_2^2 - 2 \alpha_n \left\langle z_{n-1}, \nabla \octh{n-1}\right\rangle \\
    \numberthis \label{eq:local-reference-1a}
    &\leq (1 + \alpha_n^2L^2)\norm{z_{n-1}}_2^2 - 2 \alpha_n \left\langle z_{n-1}, \nabla \octh{n-1}\right\rangle.
\end{align*}The above inequality follows from the theorem condition: $\nabla \octh{*} = 0$ and because $\oce(\cdot)$ is $L$-smooth, which implies that $\nabla \oce$ is $L$-Lipschitz. Since $\ocet$ is continuously differentiable and $\mu$-strongly convex function $\oce(\cdot)$, by Definition 2.1.2 of \citet{nesterov_introductory_2004}, we have $\octh{1} \geq \octh{2} + \left\langle \nabla \octh{2}, \theta_2 - \theta_1 \right\rangle + \frac{\mu}{2}\norm{\theta_1 - \theta_2}_2,$ for every $\theta_1,\theta_2 \in \Theta$. Putting $\theta_1 = \theta_{n-1},\theta_2 = \theta_*$ in the identity and using the condition: $\nabla \octh{*} = 0$, we have $\octh{*} - \octh{n-1} \leq \frac{-\mu}{2}\norm{z_{n-1}}_2^2$. Furthermore, by putting $\theta_1 = \theta_*$ and $\theta_2 = \theta_{n-1}$ in the identity, we have
\begin{align*}
    -\left\langle z_{n-1}, \nabla \octh{n-1}\right\rangle \leq \octh{*} - \octh{n-1}  - \frac{\mu}{2}\norm{z_{n-1}}_2^2 \leq -\mu \norm{z_{n-1}}_2^2.
\end{align*}
Substituting the above result back in (\ref{eq:local-reference-1a}), we have
\begin{equation}\label{eq:local-reference-1b}
    \norm{z_{n-1} - \alpha_n \nabla \octh{n-1}}_2^2 \leq (1 -2\alpha_n \mu + \alpha_n^2L^2)\norm{z_{n-1}}_2^2.
\end{equation}
Recall that $Q_{\theta_{k-1}}^{m_k}(\Z^k)$ is the gradient estimator constructed at iteration $k$, using iterate $\theta_{k-1}$ and samples $\Z^k$ of size $m_k$. We use the following shorthand notation $J^k = Q_{\theta_{k-1}}^{m_k}(\Z^k)$. For the case when $\theta$ is fixed (deterministic) and $m \in \mathcal{N}$, the gradient estimator $Q_\theta^m(\cdot)$ satisfies estimation bounds given by
\begin{equation*}
    \Exp\left[\norm{Q^m_\theta(\mathbf{Z}) - \nabla \octh{}}_2\right] \leq \frac{C_1}{\sqrt{m}}, \;\;\text{and\;\;}
        \Exp\left[\norm{Q^m_\theta(\mathbf{Z}) - \nabla \octh{}}_2^2\right] \leq \frac{C_2}{m},
\end{equation*}
where $C_1$ and $C_2$ can be determined from Lemma \ref{lemma:oce-gradient-estimator-I}. For the case when inputs are stochastic, such as the random iterates $\{\theta_k\}_{k \in \mathcal{N}}$, we now show that the above bounds can be extended to the gradient estimators $\{J^k\}_{k \in \mathcal{N}}$ that are constructed using the random iterates. Define the following shorthand notation $\zeta_k \triangleq J^{k} - \nabla \octh{k-1}$. Define filtration $\mathcal{F}_0 = \sigma(\theta_{0})$ and $\mathcal{F}_{k} = \sigma\left(\theta_0, \mathbf{Z}^1, \mathbf{Z}^2,\ldots,\mathbf{Z}^k\right), \forall k \in \mathcal{N}$. From the procedure given in \Cref{alg:oce-minimization}, we deduce that  $\theta_{k-1}$ is $\mathcal{F}_{k-1}$ measurable, and $\mathbf{Z}^k \perp \mathcal{F}_{k-1}$. Then by the Lemma 2.3.4 (Independence Lemma) of \citet{shreve_stochastic_2004}, following holds for every $k \in \mathbb{N}$:
\begin{align}\label{eq:sg-independence-lemma}
    \Exp\left[ \left. \norm{\zeta_k}_2  \right| \mathcal{F}_{k-1} \right] \leq \frac{C_1}{\sqrt{m_k}}, \;\; \mathrm{ and } \;\;
    \Exp\left[ \left. \norm{\zeta_k}_2^2  \right| \mathcal{F}_{k-1} \right] \leq \frac{C_2}{m_k}.
\end{align}
\paragraph{Part II.} Next, we derive MAE bounds on the last iterate of the SG algorithm. For each iteration $n\ \in \mathrm{N}$ of the SG update, we have $z_n = \Pi_\theta\left(\theta_{n-1} - \alpha_n \left(\nabla \octh{n-1}+\zeta_n\right)\right) - \theta_*$. Note that $\theta_* \in \Theta$ holds, and therefore, $\theta_* = \Pi_\Theta(\theta_*)$. Using this identity along with the non-expansive property of the projection operator, we have the following w.p. $1$.
\begin{align*}
    \norm{z_n}_2 &\leq \norm{z_{n-1} - \alpha_n \left(\nabla \octh{n-1}+\zeta_n\right)}_2 \\
    &\leq \norm{z_{n-1} - \alpha_n \nabla \octh{n-1}}_2 + \alpha_n \norm{\zeta_n}_2 \leq \sqrt{1 - 2\alpha_n \mu + \alpha_n^2L^2}\norm{z_{n-1}}_2 + \alpha_n \norm{\zeta_n}_2,
\end{align*}
where the last inequality follows from \cref{eq:local-reference-1b}. Here, the square root is well-defined because $1 - 2\alpha_k \mu + \alpha_k^2L^2$ is non-negative for every $k$. Indeed, $(1 - 2\alpha_k \mu + \alpha_k^2L^2) \geq (1 - 2\alpha_k \mu + \alpha_k^2\mu^2) \geq 0$, where we used the fact that for a $L$-smooth and $\mu$-strongly convex function, $L \geq \mu$ holds. After unrolling the above inequality, the following holds w.p. $1$:
\begin{align*}
    \norm{z_n}_2 &\leq \norm{z_0}_2 \left(\prod_{k=1}^n \sqrt{1 - 2\alpha_k \mu + \alpha_k^2L^2}\right) + \sum_{k=1}^n \left[\left(\alpha_k \norm{\zeta_k}_2\right) \right] \left(\prod_{j=k+1}^n \sqrt{1 - 2\alpha_j \mu + \alpha_j^2L^2}\right) \\
    \numberthis \label{eq:local-reference-4}
    &= \norm{z_0}_2  \sqrt{\prod_{k=1}^n \left( 1 - 2\alpha_k \mu + \alpha_k^2L^2\right)} + \sum_{k=1}^n \left[\left(\alpha_k \norm{\zeta_k}_2\right) \right] \sqrt{\prod_{j=k+1}^n \left( 1 - 2\alpha_j \mu + \alpha_j^2L^2\right)}.
\end{align*}
Note that if $0\leq a_j \leq b_j,\forall j$ then $\Pi_j a_j \leq \Pi_j b_j$. Let $a_j = (1 - 2\alpha_j \mu + \alpha_j^2L^2)$ and $b_j = \exp{\left(2\alpha_j \mu + \alpha_j^2L^2\right)}$. Then, we apply the identity: $1+x \leq e^x,\forall x \in \Rel$ to infer that $a_j\leq b_j,\forall j$. Then, we have
\begin{equation}\label{eq:product-1-plus-x}
    \prod_{j=k+1}^n \left( 1 - 2\alpha_j \mu + \alpha_j^2L^2\right) \leq \sum_{j=k+1}^n \exp{\left(- 2\alpha_j \mu + \alpha_j^2L^2\right)}.
\end{equation}
Substituting the above result in (\ref{eq:local-reference-4}), we have
\begin{align*}
    \norm{z_n}_2 & \leq \norm{z_{0}}_2 \exp{\left(\sum_{k=1}^n\left(- \alpha_k \mu + \frac{\alpha_k^2L^2}{2}\right)\right)} + \sum_{k=1}^n\exp{\left(\sum_{j=k+1}^n\left(- \alpha_j \mu + \frac{\alpha_j^2L^2}{2}\right)\right)}\alpha_k \norm{\zeta_k}_2 \\
    \numberthis \label{eq:local-reference-2}
    &\leq \exp{\left(\sum_{j=1}^n \frac{\alpha_j^2L^2}{2}\right)} \left[\norm{z_{0}}_2  \exp{\left(\sum_{k=1}^n - \alpha_k \mu\right)} + \sum_{k=1}^n\exp{\left(\sum_{j=k+1}^n - \alpha_j \mu\right)}\alpha_k \norm{\zeta_k}_2\right].
\end{align*}
For the term: $\exp{\left(\sum_{j=1}^n \frac{\alpha_j^2L^2}{2}\right)}$, using $\alpha_k = \frac{c}{k}$ and applying simple calculus to have the following bound:
\begin{align*}
    \exp{\left(\sum_{j=1}^n \frac{\alpha_j^2L^2}{2}\right)} &= \exp{\left(\frac{c^2L^2}{2}\sum_{j=1}^n \frac{1}{j^{2}}\right)} \leq \exp{\left(\frac{c^2L^2\pi^2}{12}\right)}.
\end{align*}
In a similar manner, for the other term: $\exp{\left(\sum_{j=k+1}^n - \alpha_j \mu\right)}$, we have
\begin{align*}
    &\exp{\left(\sum_{j=k+1}^n - \alpha_j \mu\right)} =  \exp{\left( \mu c \sum_{j=k+1}^n \frac{-1}{j}\right)} \\
    &\leq \exp{\left( \mu c \int_{j=k+1}^{n+1} \frac{-1}{j}dj\right)} = \exp{\left( \mu c \left[-\log(x)\right]^{n+1}_{k+1}\right)} = \left(\frac{k+1}{n+1}\right)^{\mu c} \leq 2^{\mu c}\left(\frac{k}{n+1}\right)^{\mu c}.
\end{align*}
Substituting the bounds for the above two terms back in (\ref{eq:local-reference-2}), we have
\begin{equation*}
    \norm{z_n}_2  \leq \exp{\left(\frac{c^2L^2\pi^2}{12}\right)} \left[\frac{\norm{z_{0}}_2}{(n+1)^{\mu c}} + \frac{2^{\mu c}c}{(n+1)^{\mu c}} \sum_{k=1}^n k^{\mu c - 1} \norm{\zeta_k}_2\right].
\end{equation*}
Taking expectation on both sides, we have
\begin{align*}
    \Exp\left[\norm{z_n}_2\right] 
    \leq \exp{\left(\frac{c^2L^2\pi^2}{12}\right)} \left[\frac{\Exp\left[\norm{z_{0}}_2\right]}{(n+1)^{\mu c}} + \frac{2^{\mu c}cC_1}{(n+1)^{\mu c}} \sum_{k=1}^n k^{\mu c - 1.5}\right],
\end{align*}
where the above inequality follows from \cref{eq:sg-independence-lemma} after applying the law of total expectation. The theorem assumption $c > \frac{1}{\mu}$ implies that the condition : $\mu c-1.5 > -1$ is satisfied, and the summation above is bounded by a finite integral given below.
\begin{align*}
    \sum_{k=1}^n k^{\mu c - 1.5} \leq \int_{k=1}^{n+1} k^{\mu c - 1.5}dk \leq \frac{(n+1)^{1+\mu c - 1.5}}{1+\mu c - 1.5} \leq 2 (n+1)^{\mu c - 0.5}.
\end{align*}
Then,
\begin{equation}\label{eq:sg-theorem-mae-bound}
    \Exp\left[\norm{z_n}_2\right] \leq \exp{\left(\frac{c^2L^2\pi^2}{12}\right)} \left[\frac{\Exp\left[\norm{z_{0}}_2\right]}{n+1} + \frac{2^{\mu c+1}cC_1}{\sqrt{n+1}}\right].
\end{equation}
This concludes the MAE bound on the last iterate of SG algorithm. 
\paragraph{Part III.} We now derive the MSE error bound on the last iterate $\theta_n$. With probability $1$, we have
\begin{align*}
    \norm{z_n}_2^2 &\leq \norm{z_{n-1} - \alpha_n \nabla \octh{n-1} - \alpha_n \zeta_n }_2^2 \\
    &= \norm{z_{n-1} - \alpha_n \nabla \octh{n-1}}_2^2 - 2 \alpha_n \langle z_{n-1} - \alpha_n \nabla \octh{n-1}, \zeta_n\rangle + \alpha_n^2 \norm{\zeta_n}_2^2 \\
    &\leq (1 - 2\alpha_n \mu + \alpha_n^2L^2)\norm{z_{n-1}}_2^2 + 2\alpha_n \norm{z_{n-1} - \alpha_n \nabla \octh{n-1}}_2 \norm{\zeta_n}_2 + \alpha_n^2 \norm{\zeta_n}_2^2 \\
    &\leq (1 - 2\alpha_n \mu + \alpha_n^2L^2)\norm{z_{n-1}}_2^2 + 2\alpha_n \sqrt{1 - 2\alpha_n \mu + \alpha_n^2L^2} \norm{z_{n-1}}_2 \norm{\zeta_n}_2 + \alpha_n^2 \norm{\zeta_n}_2^2 \\
    &\leq (1 - 2\alpha_n \mu + \alpha_n^2L^2)\norm{z_{n-1}}_2^2 + 2\alpha_n  \left(1 + c L\right) \norm{z_{n-1}}_2 \norm{\zeta_n}_2 + \alpha_n^2 \norm{\zeta_n}_2^2,
\end{align*}
where the second and third inequalities follow from (\ref{eq:local-reference-1b}), while the last inequality follows from: $\sqrt{1 - 2\alpha_n \mu + \alpha_n^2L^2} < 1+\alpha_n L \leq 1 + \alpha_1 L = 1+cL$. Unrolling the above equation, we have
\begin{align*}
    \norm{z_n}_2^2 &\leq \norm{z_0}_2^2 \prod_{k=1}^n \left( 1 - 2\alpha_k \mu + \alpha_k^2L^2\right) + \sum_{k=1}^n \left[\left(2 \left(1+c L\right)\alpha_k \norm{z_{k-1}}_2 \norm{\zeta_k}_2 + \alpha_k^2 \norm{\zeta_k}_2^2 \right) \prod_{j=k+1}^n \left( 1 - 2\alpha_j \mu + \alpha_j^2L^2\right) \right] \\
    &\leq \frac{\norm{z_0}_2^2}{(n+1)^{2\mu c}} + \sum_{k=1}^n \left[\left(2 \left(1+c L\right)\alpha_k  \norm{z_{k-1}}_2 \norm{\zeta_k}_2 + \alpha_k^2 \norm{\zeta_k}_2^2 \right)\left(\frac{k+1}{n+1}\right)^{2 \mu c} \right].
\end{align*}
The last inequality follows from (\ref{eq:product-1-plus-x}). Taking expectations on both sides, we have
\begin{equation}\label{eq:local-reference-3}
    \Exp\left[\norm{z_n}_2^2\right] \leq \frac{\Exp\left[\norm{z_0}_2^2\right] + 2^{2 \mu c}\sum_{k=1}^n \left[\left(2\left(1+c L\right) \alpha_k  \Exp\left[ \norm{z_{k-1}}_2 \norm{\zeta_k}_2\right] + \alpha_k^2 \Exp\left[\norm{\zeta_k}_2^2\right] \right) k^{2 \mu c} \right]}{\left(n+1\right)^{2 \mu c}} 
\end{equation}
Next, we have for all $k \in \{1,2,\ldots,n\}$,
\begin{align*}
    \Exp\left[ \norm{z_{k-1}}_2 \norm{\zeta_k}_2\right] = \Exp\left[ \left.\Exp\left[ \norm{z_{k-1}}_2 \norm{\zeta_k}_2 \right| \mathcal{F}_{k-1}\right]\right] =\Exp\left[ \norm{z_{k-1}}_2 \left.\Exp\left[  \norm{\zeta_k}_2\right| \mathcal{F}_{k-1}\right]\right] \leq \frac{C_1}{\sqrt{m_k}}\Exp\left[ \norm{z_{k-1}}_2\right].
\end{align*}
The first equality is the law of total expectation, while the second equality follows because $\theta_{k-1}$ is $\mathcal{F}_{k-1}$-measurable. The last inequality follows from the first bound in \cref{eq:sg-independence-lemma}. Substituting this back into \cref{eq:local-reference-3} and applying the second bound from \cref{eq:sg-independence-lemma}, we have
\begin{align*}
    \Exp\left[\norm{z_n}_2^2\right] &\leq \frac{\Exp\left[\norm{z_0}_2^2\right]}{(n+1)^{2\mu c}} + \frac{2^{2 \mu c}}{\left(n+1\right)^{2 \mu c}}\sum_{k=1}^n \left[ \frac{2\left(1+c L\right)C_1}{\sqrt{m_k}} \alpha_k k^{2 \mu c} \Exp\left[ \norm{z_{k-1}}_2\right]  + \alpha_k^2 \frac{C_2}{m_k}k^{2 \mu c} \right].
\end{align*}
Substituting $m_k=k$ and $\alpha_k=c/k$, we have
\begin{align}
    \nonumber
    \Exp\left[\norm{z_n}_2^2\right] &\leq \frac{\Exp\left[\norm{z_0}_2^2\right]}{(n+1)^{2\mu c}}+ \frac{2^{2 \mu c }c^2C_2}{\left(n+1\right)^{2 \mu c}}\sum_{k=1}^n k^{2 \mu c - 3} + \frac{2^{2 \mu c+1}\left(1+c L\right)cC_1}{\left(n+1\right)^{2 \mu c}}\sum_{k=1}^n k^{2\mu c-1.5}\Exp\left[ \norm{z_{k-1}}_2\right] \\
    \label{eq:local-reference-5}
    &\leq \frac{\Exp\left[\norm{z_0}_2^2\right]}{(n+1)^{2}}+ \frac{2^{2 \mu c }c^2C_2}{(2\mu c-2)\left(n+1\right)^{2}} + \frac{2^{2 \mu c+1}\left(1+c L\right)cC_1}{\left(n+1\right)^{2 \mu c}}\sum_{k=1}^n k^{2\mu c-1.5}\Exp\left[ \norm{z_{k-1}}_2\right].
\end{align}
The last inequality follows from the assumption $c > \frac{1}{\mu}$ and by replacing the first summation with an integral. For bounding the last summation in \cref{eq:local-reference-5}, we make the following claim: for every $k \geq 1$, 
\begin{align*}
    \Exp\left[ \norm{z_{k-1}}_2\right] \leq \exp{\left(\frac{c^2L^2\pi^2}{12}\right)} \left[\frac{\Exp\left[\norm{z_{0}}_2\right]}{k} + \frac{2^{\mu c+1}cC_1}{\sqrt{k}}\right].
\end{align*}
For $k>1$, the claim follows from the MAE bound in \cref{eq:sg-theorem-mae-bound}, while for the case $k=1$ it holds trivially. Substituting the above inequality back in \cref{eq:local-reference-5}, we have
\begin{align*}
    &\Exp\left[\norm{z_n}_2^2\right] \leq \frac{\Exp\left[\norm{z_0}_2^2\right]}{(n+1)^{2}}+ \frac{2^{2 \mu c }c^2C_2}{(2\mu c-2)\left(n+1\right)^{2}} \\
    &+ \frac{2^{2 \mu c+1}\left(1+c L\right)cC_1}{\left(n+1\right)^{2 \mu c}}\exp{\left(\frac{c^2L^2\pi^2}{12}\right)} \sum_{k=1}^n \left(k^{2\mu c-2.5}\Exp\left[ \norm{z_{0}}_2\right] + 2^{\mu c+1}cC_1k^{2\mu c -2}\right) \\
    &\leq \frac{\Exp\left[\norm{z_0}_2^2\right]}{(n+1)^{2}}+ \frac{2^{2 \mu c }c^2C_2}{(2\mu c-2)\left(n+1\right)^{2}} + \frac{2^{2 \mu c+2}\left(1+c L\right)cC_1}{n+1}\exp{\left(\frac{c^2L^2\pi^2}{12}\right)}\left[ \frac{\Exp\left[ \norm{z_{0}}_2\right]}{\sqrt{n+1}} + 2^{\mu c}cC_1\right] \\
    &=\frac{\Exp\left[\norm{z_0}_2^2\right]+K_1}{(n+1)^{2}} + \frac{K_2 \Exp\left[ \norm{z_{0}}_2\right]}{(n+1)^{3/2}} + \frac{K_3}{n+1},
\end{align*}
where $K_1=\frac{2^{2 \mu c }c^2C_2}{(2\mu c-2)}, K_2 = 2^{2 \mu c+2}\left(1+c L\right)cC_1 \exp{\left(\frac{c^2L^2\pi^2}{12}\right)}$ and $K_3=2^{\mu c}cC_1K_2$.
The last inequality follows by bounding the summation with a finite integral. This completes the proof for the bound on convergence in parameter, i.e., $\theta_n \to \theta_*$ for the strongly-convex case. For the bound on convergence in value ($\octh{n} \to \octh{*}$), similar bounds may be obtained using the fact that $\ocet$ is $L$-smooth.
\end{proof}

\subsection{Convex case}
When $\ocet(\cdot)$ is convex, we bound the difference in function value $\Exp\left[\ocet(\overline{\theta}_n)-\octh{*}\right]$ for an average iterate given by $\overline{\theta}_n \triangleq \frac{1}{n} \sum_{k=1}^n \theta_{k-1}$. This bound is  \horder{1/\sqrt{n}} and is presented in the result below.
\begin{theorem}\label{theorem:oce-sgd-average-iterates}
Suppose the assumptions of \Cref{theorem-oce-gradient} and \cref{assumption:u-prime-variance,assumption:F-gradient-bound-II,as:F-smooth,assumption:u-F-variance-bound} are satisfied. Suppose $\ocet(\cdot)$ is convex on $\Theta$. Run \Cref{alg:oce-minimization} with $\alpha_k = \frac{1}{L\sqrt{k}},m_k=k,\,\forall k$, then we have
    \begin{align*}
        \Exp\left[\ocet(\overline{\theta}_n)-\octh{*}\right] \leq\frac{K_4}{\sqrt{n}},
    \end{align*}
    where the constant $K_4$ depends on the initial errors $\Exp\left[\norm{\theta_0-\theta_*}_2\right],\Exp\left[\norm{\theta_0-\theta_*}_2^2\right]$ and contains the terms $\ln{n}$ and $\ln^2{n}$.
\end{theorem}
\begin{proof}
    Note that $\ocet$ is convex and $L$-smooth, and $\Theta$ is convex. Then by Theorem 2.1.5 of \cite{nesterov_introductory_2004}, for every $\theta_2,\theta_2 \in \Theta$, we have
    \begin{align}\numberthis \label{eq:grad-product-bound}
        \left<\nabla \octh{2} - \nabla \octh{1},\theta_2-\theta_1\right> \geq \frac{1}{L}\norm{\nabla \octh{2} - \nabla \octh{1}}^2.
    \end{align}
    Since $\ocet(\cdot)$ is continuously differentiable and convex, by Definition 2.1.2 of \cite{nesterov_introductory_2004}, we have
    \begin{equation}\label{eq:oce-convexity}
        \octh{2} \geq \octh{1} + \nabla \octh{1}^T (\octh{2}-\octh{1}), \;\; \forall \theta_1,\theta_2 \in \B.
    \end{equation}
    Let $\{z_k\}_{k\geq 0}$ and $\{\zeta_k\}_{k\geq 1}$ be as defined in the proof of \Cref{theorem:oce-sgd}. Then, we have the following w.p. $1$.
    \begin{align*}
        \norm{z_{n-1} - \alpha_n \nabla \octh{n-1}}_2^2  &= \norm{z_{n-1}}_2^2 + \alpha_n^2 \norm{\nabla \octh{n-1}}_2^2 - 2 \alpha_n \left\langle z_{n-1}, \nabla \octh{n-1}\right\rangle \\
        &\leq \norm{z_{n-1}}_2^2 + (\alpha_n^2 - \frac{2 \alpha_n}{L}) \norm{\nabla \octh{n-1}}_2^2 \leq \norm{z_{n-1}}_2^2 \\
        \numberthis \label{eq:oce-cvx-zn-bound}
        &\implies \norm{z_{n-1} - \alpha_n \nabla \octh{n-1}}_2 \leq \norm{z_{n-1}}_2.
    \end{align*}
    This implies that w.p. $1$, the following holds.
    \begin{align*}
        \norm{z_n}_2 &\leq \norm{z_{n-1} - \alpha_n \left(\nabla \octh{n-1}+\zeta_n\right)}_2 \leq \norm{z_{n-1} - \alpha_n \nabla \octh{n-1}}_2 + \alpha_n \norm{\zeta_n}_2 \leq \norm{z_{n-1}}_2 + \alpha_n \norm{\zeta_n}_2 \\
        &\implies \norm{z_n}_2 \leq \norm{z_0}_2 + \sum_{k=1}^n \alpha_k \norm{\zeta_k}_2,
    \end{align*}
    where the first inequality follows from \cref{eq:grad-product-bound} and the second inequality follows from the step size assumption: $\alpha_n \leq 1/L$.
    Next, we note that the following holds w.p. $1$.
    \begin{align*}
        \norm{z_n}_2^2 &\leq \norm{z_{n-1} - \alpha_n \nabla \octh{n-1}}_2^2 - 2 \alpha_n \langle z_{n-1} - \alpha_n \nabla \octh{n-1}, \zeta_n\rangle + \alpha_n^2 \norm{\zeta_n}_2^2 \\
        &\leq \norm{z_{n-1}}_2^2 + \left(\alpha_n^2-\frac{\alpha_n}{L}\right) \norm{\nabla \octh{n-1}}_2^2 - \alpha_n (\octh{n-1}-\octh{*}) + 2 \alpha_n \norm{z_{n-1}}_2\norm{\zeta_n}_2 + \alpha_n^2 \norm{\zeta_n}_2^2 \\
        &\leq \norm{z_{n-1}}_2^2 - \alpha_n (\octh{n-1}-\octh{*}) + 2 \alpha_n \left(\norm{z_0}_2 + \sum_{k=1}^{n-1} \alpha_k \norm{\zeta_k}_2\right)\norm{\zeta_n}_2 + \alpha_n^2 \norm{\zeta_n}_2^2 \\
        \implies &\alpha_n (\octh{n-1}-\octh{*}) \leq \norm{z_{n-1}}_2^2 - \norm{z_{n}}_2^2 + 2 \alpha_n \norm{\zeta_n}_2 \norm{z_0}_2   +2 \alpha_n \norm{\zeta_n}_2 \left[\sum_{k=1}^{n-1} \alpha_k \norm{\zeta_k}_2\right] + \alpha_n^2 \norm{\zeta_n}_2^2.
    \end{align*}
    Taking the telescopic sum, we have the following w.p. $1$. 
    \begin{align}\label{eq:avg-temp-Y}
        \sum_{k=1}^n \alpha_k (\octh{k-1}-\octh{*}) &\leq \norm{z_{0}}_2^2 + 2 \norm{z_0}_2 \sum_{k=1}^n \alpha_k \norm{\zeta_k}_2 + 2 \sum_{k=1}^n \alpha_k \norm{\zeta_k}_2 \sum_{j=1}^{k-1} \alpha_j \norm{\zeta_j}_2 +\sum_{k=1}^n \alpha_k^2 \norm{\zeta_k}_2^2.
    \end{align}
    By monotonicity of $\alpha_k$, i.e., $\alpha_n\leq \alpha_k,\forall k \in 1.\ldots,n$, we have the following w.p. $1$.
    \begin{align*}
        \sum_{k=1}^n \alpha_n (\octh{k-1}-\octh{*}) &\leq \norm{z_{0}}_2^2 + 2 \norm{z_0}_2 \sum_{k=1}^n \alpha_k \norm{\zeta_k}_2 + 2 \sum_{k=1}^n \alpha_k \norm{\zeta_k}_2 \sum_{j=1}^{k-1} \alpha_j \norm{\zeta_j}_2 +\sum_{k=1}^n \alpha_k^2 \norm{\zeta_k}_2^2.
    \end{align*}
    By convexity of $\ocet(\cdot)$, we have the following w.p. $1$.
    \begin{align*}
        \frac{\sqrt{n}}{L} \left(\ocet(\overline{\theta}_n)-\octh{*}\right) &\leq \norm{z_{0}}_2^2 + 2 \norm{z_0}_2 \sum_{k=1}^n \frac{\norm{\zeta_k}_2}{L\sqrt{k}} + 2 \sum_{k=1}^n \frac{\norm{\zeta_k}_2}{L\sqrt{k}} \sum_{j=1}^{k-1} \frac{\norm{\zeta_j}_2}{L\sqrt{j}} +\sum_{k=1}^n \frac{\norm{\zeta_k}_2^2}{L^2 k}.
    \end{align*}
    Taking expectation on both sides, and with identical arguments as made in the proof in \Cref{theorem:oce-sgd} (see \cref{eq:sg-independence-lemma}), we have 
    \begin{align*}
        \frac{\sqrt{n}}{L} \Exp\left[\ocet(\overline{\theta}_n)-\octh{*}\right] &\leq \Exp\left[\norm{z_{0}}_2^2\right] + \frac{2 C_1\Exp\left[\norm{z_0}_2\right]}{L} \sum_{k=1}^n \frac{1}{k} + 2 \frac{C_1^2}{L^2}\sum_{k=1}^n \frac{1}{\sqrt{k}} \sum_{j=1}^{k-1} \frac{1}{\sqrt{j}} + \frac{C_2}{L^2} \sum_{k=1}^n \frac{1}{k^2}.
    \end{align*}
    Applying $\sum_{k=1}^n \frac{1}{k} \leq 1 + \ln{n}$ and rearranging, we have
    \begin{align*}
        \Exp\left[\ocet(\overline{\theta}_n)-\octh{*}\right] &\leq \frac{1}{\sqrt{n}}\left[L\Exp\left[\norm{z_{0}}_2^2\right] + 2 C_1\Exp\left[\norm{z_0}_2\right](1+ \ln{n}) + \frac{2C_1^2(1+ \ln{n})^2}{L} + \frac{C_2\pi^2}{6L}\right].
    \end{align*}
\end{proof}

\subsection{Non-convex case}
The final result for OCE optimization does not require convexity, and as is standard in non-convex optimization literature, establishes a bound on the OCE gradient norm.
\begin{theorem}\label{theorem:oce-sgd-non-convex}
Suppose the assumptions of \Cref{theorem-oce-gradient} and \cref{assumption:u-prime-variance,assumption:F-gradient-bound-II,as:F-smooth,assumption:u-F-variance-bound} are satisfied. Run \Cref{alg:oce-minimization} with $\alpha_k = \frac{1}{L\sqrt{k}},m_k=k,\,\forall k$, then we have
    \begin{align*}
        \Exp&\left[\min_{k \in [1,n]} \norm{\nabla \octh{k-1}}_2^2\right] \leq\frac{K_5}{\sqrt{n}},
    \end{align*}
    where the constant $K_5$ depends on the initial errors $\Exp\left[\norm{\theta_0-\theta_*}_2\right],\Exp\left[\norm{\theta_0-\theta_*}_2^2\right]$ and contains the terms $\ln{n}$ and $\ln^2{n}$.
\end{theorem}
\begin{proof}
    Recall the terms $\{z_k\}_{k\geq 0}$ and $\{\zeta_k\}_{k\geq 1}$ given in \Cref{theorem:oce-sgd}. The following holds w.p. $1$.
    \begin{align*}
        \norm{z_n}_2^2 &\leq \norm{z_{n-1}}_2^2 + \left(\alpha_n^2-\frac{2\alpha_n}{L}\right) \norm{\nabla \octh{n-1}}_2^2 + 2 \alpha_n \norm{z_{n-1}}_2\norm{\zeta_n}_2 + \alpha_n^2 \norm{\zeta_n}_2^2 \\
        &\leq  \norm{z_{n-1}}_2^2 - \frac{\alpha_n}{L} \norm{\nabla \octh{n-1}}_2^2 + 2 \alpha_n \left(\norm{z_0}_2 + \sum_{j=1}^{n-1} \alpha_j \norm{\zeta_j}_2\right)\norm{\zeta_n}_2 + \alpha_n^2 \norm{\zeta_n}_2^2 \\
        \implies &\frac{\alpha_n}{L} \norm{\nabla \octh{n-1}}_2^2 \leq \left(\norm{z_{n-1}}_2^2-\norm{z_{n}}_2^2\right) + 2 \alpha_n \left(\norm{z_0}_2 + \sum_{j=1}^{n-1} \alpha_j \norm{\zeta_j}_2\right)\norm{\zeta_n}_2 + \alpha_n^2 \norm{\zeta_n}_2^2.
    \end{align*}
    Taking summation from $1$ to $n$, we have (i) $\frac{1}{L} \sum_{k=1}^n \alpha_k \norm{\nabla \octh{k-1}}_2^2 \leq Y$ a.s., where $Y=\norm{z_{0}}_2^2 + 2 \norm{z_0}_2 \sum_{k=1}^n \alpha_k \norm{\zeta_k}_2 + 2 \sum_{k=1}^n \alpha_k \norm{\zeta_k}_2 \sum_{j=1}^{k-1} \alpha_j \norm{\zeta_j}_2 +\sum_{k=1}^n \alpha_k^2 \norm{\zeta_k}_2^2$. From \Cref{theorem:oce-sgd-average-iterates}, we know that $Y$ is the same term as that on the r.h.s. of \cref{eq:avg-temp-Y} and we know that (ii) $\Exp\left[Y\right] \leq \frac{K_4}{L}$. From the monotonicity of $\alpha_k$, we know that 
    \begin{align*}
        \sum_{k=1}^n \alpha_k \norm{\nabla \octh{k-1}}_2^2 \geq \alpha_n \sum_{k=1}^n \norm{\nabla \octh{k-1}}_2^2 \geq n \alpha_n \min_{k \in 1,\ldots,n} \norm{\nabla \octh{k-1}}_2^2.
    \end{align*}
    Combining the above inequality with (i), we have $\frac{\sqrt{n}}{L} \min_{k \in 1,\ldots,n} \norm{\nabla \octh{k-1}}_2^2 \leq Y$ a.s. Then, taking the expectation on both sides and invoking (ii), we have
    \begin{equation*}
        \Exp\left[\min_{k \in 1...n} \norm{\nabla \octh{k-1}}_2^2\right] \leq \frac{K_4}{\sqrt{n}}.
    \end{equation*}
\end{proof}
The $\ln{n}$ terms in \cref{theorem:oce-sgd-average-iterates,theorem:oce-sgd-non-convex} are unavoidable as \Cref{alg:oce-minimization} is an any-time algorithm. If we modify the step sizes to $\alpha_k=\frac{1}{L\sqrt{N}}$ or modify the batch sizes to $m_k=N$, then the $\ln{n}$ terms are reduced to a constant.

\section{Experiments}
\label{sec:suppl-experiments}
In this section, we provide a detailed description of our simulation experiments on OCE estimation and optimization. Our work is foundational, and we introduce a gradient-based approach to optimize the OCE criterion and provide convergence guarantees. The goal of the experiments is to compare the OCE-optimal decision models against well-known benchmarks, including the risk-neutral variant. The experiments report comparisons across different problem instances, different OCE instances, and a wide range of evaluation metrics. Therefore, a comprehensive hyperparameter analysis is out of scope for this work and could be an interesting research direction for future work. In a similar spirit, analysis of the hyperparameters that make up the utility functions is left to future work, as the choice of utility functions is subjective and depends on the decision-maker's risk preferences. Thus, there are no 'optimal' hyperparameters in the context of utility functions.  
\paragraph{Hardware:}The experiments on uncertainty quantification and classification were run on \textit{Google Colab} with 'Python 3 Google Compute Engine backend' having 12.7 GB of RAM. The remaining experiments were run on a standard 16 GB laptop. 
\subsection{Portfolio Optimization}
\paragraph{Setup.}
Our experiment setup is based on the \texttt{skfolio} Python library. We test the aforementioned algorithms on three different datasets: 'Standard and Poor's 500 (S\&P500)', 'Financial Times Stock Exchange (FTSE100)', and 'Nasdaq'. The 'S\&P 500' dataset is composed of the daily prices of 20 assets from the 'S\&P 500' composition starting from 1990-01-02 up to 2022-12-28. The 'FTSE100' dataset contains daily prices for 64 assets in the 'FTSE100' index, from 2000-01-04 to 2026-05-26. The 'Nasdaq' dataset contains daily prices for 1455 assets in the 'Nasdaq' index from 2018-01-02 to 2026-05-26. For each of the above datasets, we run the OCE-SG algorithms and evaluate several risk measures, including popular choices such as entropic, monotone mean-variance, and quartic risk.
\paragraph{Implementation details.}
We denote the dataset size by $M$ and the number of assets by $d$. Thus, values of $(M,d)$ for 'S\&P 500', 'FTSE100', and 'Nasdaq' are $(8312,20),(5959,64),$ and $(1361,1455)$, respectively. Thus, the OCE-SG algorithm finds a $d$-dimensional vector $\theta$ that corresponds to the portfolio weights, with OCE as the optimization criteria. We run the OCE-SG algorithm for different choices of utility functions, given in \cref{example:entropic,example:cvar,example:lcvar,example:mean-variance,example:quartic,example:monotone-mean-variance,example:scvar}. For each dataset and each aforementioned example, we run the OCE-SG algorithm for $N=15000$ iterations, where we use $\alpha_k=\frac{1}{\sqrt{M}}$ and $m_k=M$ for each iteration. We use the \textit{pyproximal} Python library for the projection step in the gradient update step of the algorithm.  

\begin{figure}[ht]
  \centering
  \includegraphics[width=0.95\columnwidth]{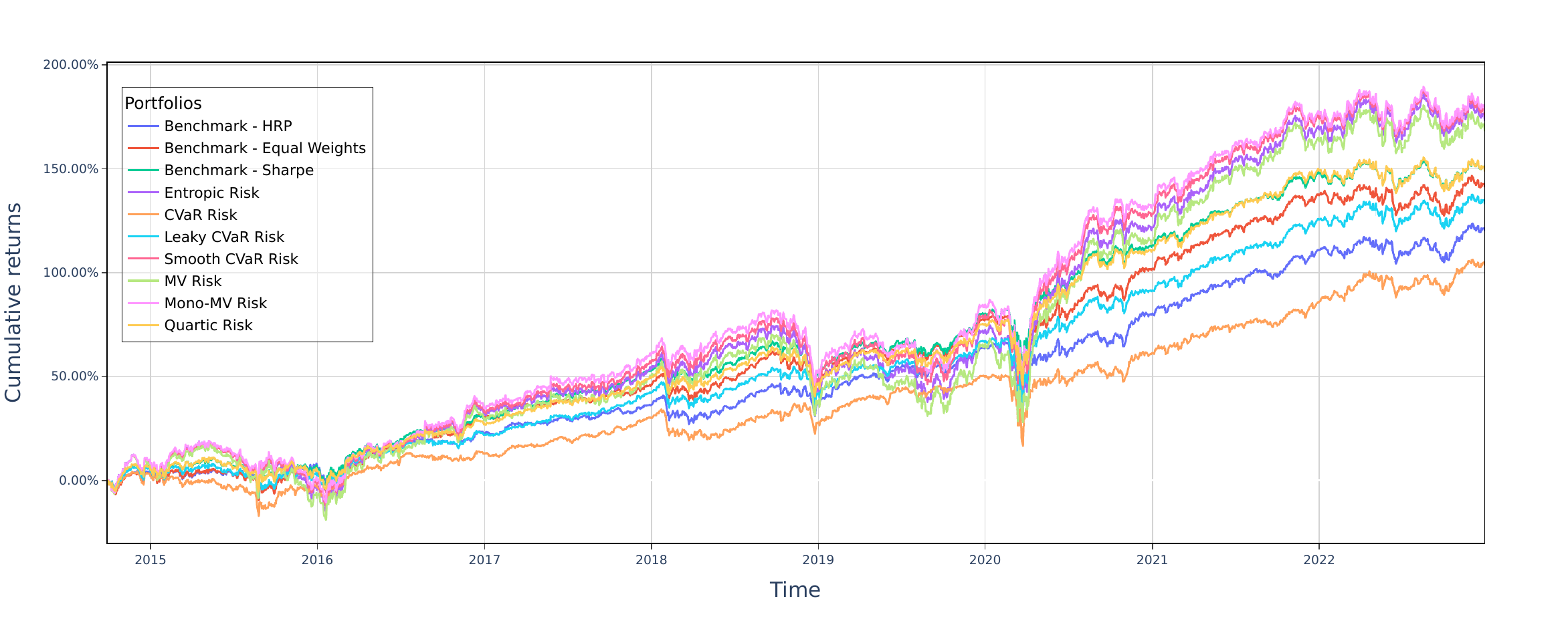}\\
  \caption{The performance of the OCE-SG algorithm for a variety of OCE risk measures in a portfolio optimization application sourced from the S\&P stock market data.}\label{figure:port-opt-snp}
\end{figure}

\begin{table}[ht]
  \centering
  \caption{The performance of the OCE-SG algorithm on the S\&P stock market dataset, compared against three benchmark portfolios. The table reports the portfolio evaluations across a wide range of risk metrics.}
  \label{tab:portfolio_performance}
  \begin{tabular}{l *{7}{r}}
    \toprule
    \textbf{Portfolio} & \textbf{Annualized Returns} & \textbf{Gini MD} & \textbf{CVaR} & \textbf{EVaR} & \textbf{CDaR} & \textbf{Sharpe} & \textbf{Sortino} \\
    \midrule
    Benchmark - HRP           & 0.1462 & \textbf{0.0102} & 0.0251 & 0.0541 & \textbf{0.1302} & 0.8663 & 1.1988 \\
    Benchmark - Equal Weights & 0.1717 & 0.0115 & 0.0274 & 0.0565 & 0.1448 & 0.9241 & 1.2885 \\
    Benchmark - Sharpe        & 0.1814 & 0.0115 & 0.0270 & 0.0542 & 0.1409 & \textbf{0.9685} & \textbf{1.3725} \\
    \midrule
    Entropic Risk             & \textbf{0.1924} & 0.0128 & 0.0297 & 0.0579 & 0.1697 & \textbf{0.9445} & \textbf{1.3308} \\
    CVaR Risk                 & 0.1272 & 0.0099 & \textbf{0.0239} & \textbf{0.0487} & 0.1528 & 0.7792 & 1.0881 \\
    Leaky CVaR Risk           & 0.1603 & \textbf{0.0106} & 0.0254 & 0.0526 & \textbf{0.1315} & 0.9225 & 1.2924 \\
    Smooth CVaR Risk          & \textbf{0.1937} & 0.0134 & 0.0304 & 0.0563 & 0.1868 & 0.9150 & 1.2925 \\
    MV Risk                   & \textbf{0.1921} & 0.0125 & 0.0294 & 0.0588 & 0.1644 & \textbf{0.9593} & \textbf{1.3451} \\
    Mono-MV Risk              & 0.1874 & 0.0125 & 0.0291 & 0.0570 & 0.1644 & \textbf{0.9403} & \textbf{1.3242} \\
    Quartic Risk              & 0.1683 & 0.0113 & 0.0267 & 0.0539 & 0.1432 & 0.9206 & 1.2954 \\
    \bottomrule
  \end{tabular}
\end{table}


\begin{figure}[ht]
  \centering
  \includegraphics[width=0.95\columnwidth]{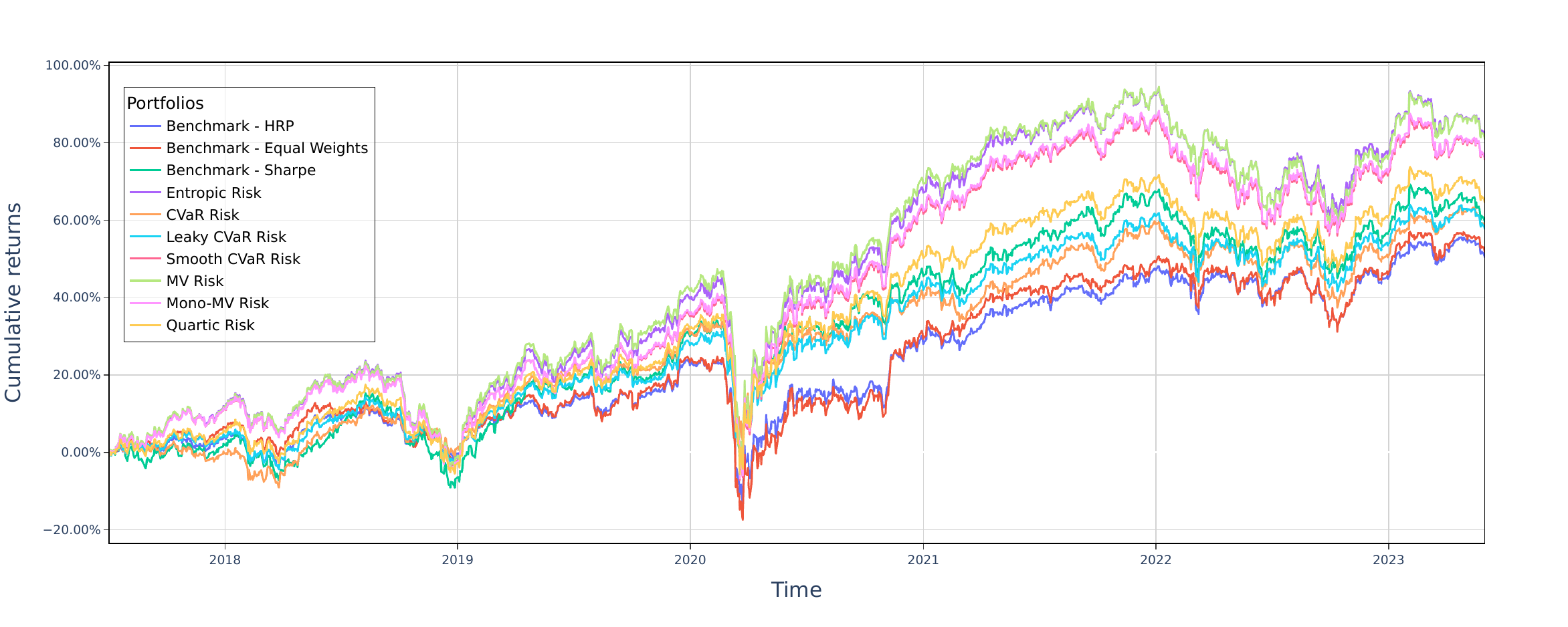}\\
  \caption{The figure shows the performance of the OCE-SG algorithm for a variety of OCE risk measures in a portfolio optimization application sourced from the FTSE stock market data.}\label{figure:port-opt-ftse}
\end{figure}

\begin{table}[ht]
  \centering
  \caption{The performance of the OCE-SG algorithm on the FTSE stock market dataset, compared against three benchmark portfolios. The table reports the portfolio evaluations across a wide range of risk metrics.}
  \label{tab:portfolio_performance_appendix}
  \begin{tabular}{l *{7}{r}}
    \toprule
    \textbf{Portfolio} & \textbf{Annualized Returns} & \textbf{Gini MD} & \textbf{CVaR} & \textbf{EVaR} & \textbf{CDaR} & \textbf{Sharpe} & \textbf{Sortino} \\
    \midrule
    Benchmark - HRP           & 0.0857 & \textbf{0.0103} & 0.0249 & 0.0516 & 0.1992 & 0.5311 & 0.7198 \\
    Benchmark - Equal Weights & 0.0879 & 0.0114 & 0.0276 & 0.0568 & 0.2367 & 0.4892 & 0.6673 \\
    Benchmark - Sharpe        & 0.1010 & 0.0120 & 0.0266 & 0.0450 & 0.2287 & 0.5579 & 0.7709 \\
    \midrule
    Entropic Risk             & \textbf{0.1252} & 0.0139 & 0.0310 & 0.0580 & 0.3026 & 0.5866 & 0.8071 \\
    CVaR Risk                 & 0.0990 & \textbf{0.0102} & \textbf{0.0229} & \textbf{0.0406} & \textbf{0.1864} & \textbf{0.6401} & \textbf{0.8777} \\
    Leaky CVaR Risk           & 0.0979 & \textbf{0.0107} & 0.0244 & 0.0452 & 0.1934 & 0.5989 & 0.8174 \\
    Smooth CVaR Risk          & \textbf{0.1120} & 0.0132 & 0.0297 & 0.0550 & 0.2877 & 0.5554 & 0.7622 \\
    MV Risk                   & \textbf{0.1229} & 0.0140 & 0.0311 & 0.0581 & 0.3065 & 0.5738 & 0.7895 \\
    Mono-MV Risk              & \textbf{0.1125} & 0.0132 & 0.0297 & 0.0551 & 0.2859 & 0.5563 & 0.7633 \\
    Quartic Risk              & \textbf{0.1119} & 0.0120 & 0.0272 & 0.0497 & 0.2435 & 0.6078 & \textbf{0.8365} \\
    \bottomrule
  \end{tabular}
\end{table}

\begin{figure}[ht]
  \centering
  \includegraphics[width=0.95\columnwidth]{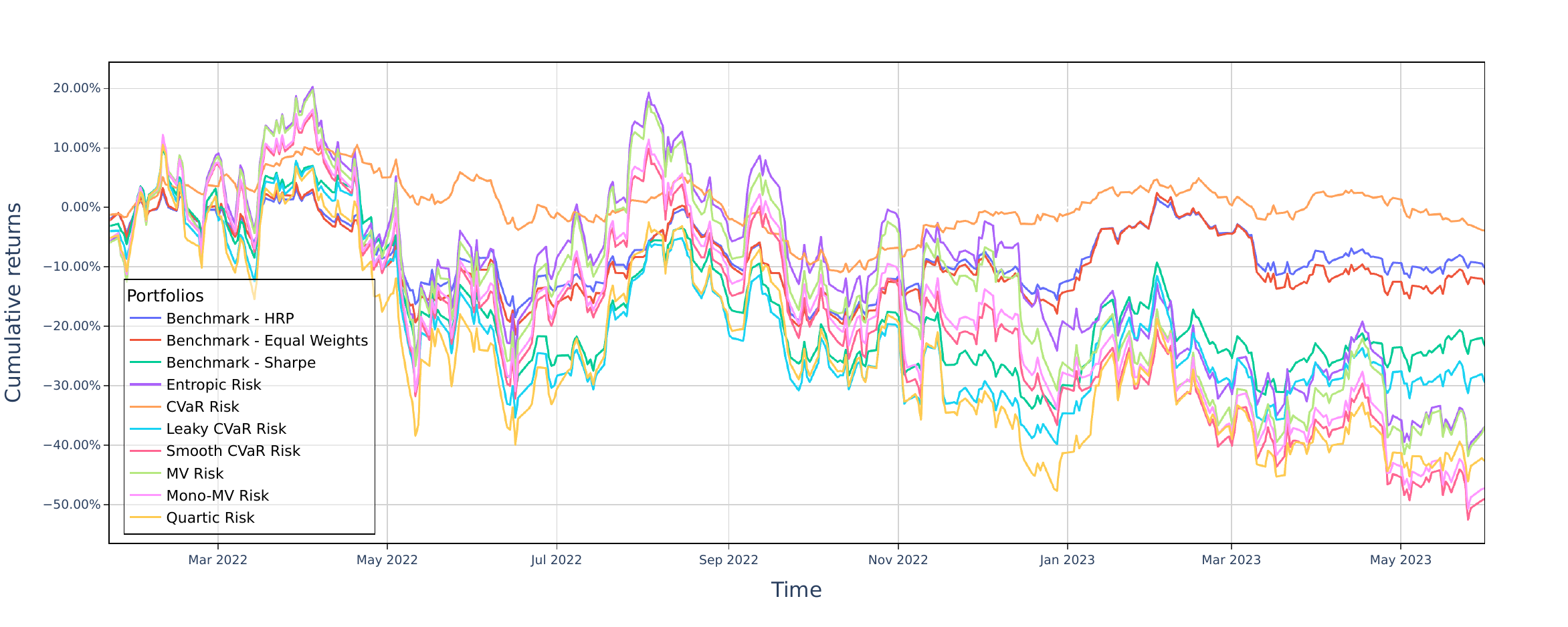}\\
  \caption{The figure shows the performance of the OCE-SG algorithm for a variety of OCE risk measures in a portfolio optimization application sourced from the NASDAQ stock market data.}\label{figure:port-opt-nasdaq}
\end{figure}

\begin{table}[ht]
  \centering
  \caption{The performance of the OCE-SG algorithm on the NASDAQ stock market dataset, compared against three benchmark portfolios. The table reports the portfolio evaluations across a wide range of risk metrics.}
  \label{tab:appendix_portfolio_metrics}
  \begin{tabular}{l *{7}{r}}
    \toprule
    \textbf{Portfolio} & \textbf{Annualized Returns} & \textbf{Gini MD} & \textbf{CVaR} & \textbf{EVaR} & \textbf{CDaR} & \textbf{Sharpe} & \textbf{Sortino} \\
    \midrule
    Benchmark - HRP           & -0.0747 & 0.0153 & 0.0279 & 0.0333 & 0.2121 & -0.3471 & -0.4924 \\
    Benchmark - Equal Weights & -0.0949 & 0.0181 & 0.0325 & 0.0379 & 0.2345 & -0.3739 & -0.5308 \\
    Benchmark - Sharpe        & -0.1711 & 0.0246 & 0.0465 & 0.0526 & 0.4183 & -0.4930 & -0.6945 \\
    \midrule
    Entropic Risk             & -0.2722 & 0.0379 & 0.0695 & 0.0768 & 0.5772 & -0.5064 & -0.7333 \\
    CVaR Risk                 & \textbf{-0.0315} & \textbf{0.0099} & \textbf{0.0191} & \textbf{0.0270} & \textbf{0.2033} & \textbf{-0.2220} & \textbf{-0.3103} \\
    Leaky CVaR Risk           & -0.2171 & 0.0292 & 0.0553 & 0.0608 & 0.4735 & -0.5270 & -0.7469 \\
    Smooth CVaR Risk          & -0.3620 & 0.0371 & 0.0686 & 0.0753 & 0.6369 & -0.6872 & -0.9975 \\
    MV Risk                   & -0.2727 & 0.0390 & 0.0706 & 0.0769 & 0.5866 & -0.4945 & -0.7172 \\
    Mono-MV Risk              & -0.3483 & 0.0371 & 0.0685 & 0.0753 & 0.6256 & -0.6616 & -0.9606 \\
    Quartic Risk              & -0.3150 & 0.0341 & 0.0633 & 0.0685 & 0.5525 & -0.6520 & -0.9399 \\
    \bottomrule
  \end{tabular}
\end{table}

\paragraph{Observations.}From \cref{figure:port-opt-ftse,figure:port-opt-nasdaq,figure:port-opt-snp} and \cref{tab:portfolio_performance_appendix,tab:appendix_portfolio_metrics,tab:portfolio_performance} we conclude that the portfolios given by the OCE-SG algorithm are either comparable to the benchmarks or outperform the benchmarks. Depending on market conditions and the decision-maker's preferences, different OCE risk measures emerge as the preferred choices. The CVaR indeed outperforms when the stock market returns fall, indicating that it is a viable risk-averse choice. On the other hand, when the stock market returns are rising, then the OCE risks: entropic, mean-variance, monotone mean-variance, smooth-cvar, and quartic risks produce higher returns, with better Sharpe and Sortino ratios. In all three datasets, the portfolio performance on the CVaR metric shows that the OCE-SG algorithm for utility given in \Cref{example:cvar} indeed optimizes the CVaR criterion, and validates the convergence guarantees of the OCE-SG algorithm.

\clearpage
\subsection{Classification}
\paragraph{Setup.}We used the \textit{UCI Heart Disease dataset} and \textit{Breast Cancer Detection dataset} for the classification problem given in \Cref{example:classification}. In this experiment, we evaluate neural network (NN) models obtained by minimizing the OCE criterion using our OCE-SG algorithm. We use \textit{PyTorch} to train these models and benchmark them against three standard classification algorithms: Logistic Regression, Support Vector Machines (SVM), and Random Forest, as well as a risk-neutral NN variant obtained by minimizing Binary Cross-Entropy (BCE) loss. These algorithms are the standard benchmarks on clinical classification tasks due to their generalization, interpretability, and performance across a wide range of metrics, going beyond 'Accuracy'. on the chosen datasets. We evaluate these models across four well-established metrics: Accuracy, F1-score, ECE (Expected Calibration Error), and AUROC (Area Under the Receiver Operating Characteristic).
\paragraph{Implementation details.}For simplicity, we choose $2$ layer NN models with hidden sizes of $32$ and $16$ respectively, for all NN-based algorithms. The following procedure is identical for the OCE variants and the risk-neutral variant. We train each algorithm for $500$ epochs, using the Adam optimizer with a learning rate of $\alpha_k=0.003$. In each epoch, the entire dataset is iterated using a fixed batch size of $m_k=64$. The remaining algorithms are implemented using the \textit{scikit-learn} library.

\begin{table}[ht]
\centering
\caption{Model Performance Comparison}
\label{tab:model_comparison}

\begin{subtable}[ht]{0.49\textwidth}
\centering
\caption{Performance comparison on the Breast Cancer Detection dataset}
\label{tab:breast-cancer-dataset}
\resizebox{\linewidth}{!}{%
\begin{tabular}{lcccc}
\toprule
\textbf{Model} & \textbf{Acc.} & \textbf{F1} & \textbf{ECE} & \textbf{AUROC} \\
\midrule
NN (OCE) Entropic    & 0.9591 & 0.9668 & \textbf{0.0213} & 0.9950 \\
NN (OCE) CVaR        & \textbf{0.9766} & \textbf{0.9813} & 0.3358 & 0.9974 \\
NN (OCE) LCVaR       & \textbf{0.9766} & \textbf{0.9815} & 0.0497 & 0.9963 \\
NN (OCE) SCVaR       & \textbf{0.9883} & \textbf{0.9907} & \textbf{0.0239} & 0.9963 \\
NN (OCE) Monotone MV & 0.9649 & 0.9717 & \textbf{0.0263} & 0.9956 \\
NN (OCE) Quartic     & 0.9708 & 0.9767 & \textbf{0.0279} & 0.9966 \\
\midrule
NN BCE Loss          & 0.9532 & 0.9619 & 0.0437 & 0.9908 \\
Logistic Regression  & \textbf{0.9883} & \textbf{0.9907} & 0.0324 & 0.9981 \\
SVC (Baseline)       & \textbf{0.9766} & \textbf{0.9813} & 0.0412 & 0.9978 \\
Random Forest        & 0.9357 & 0.9488 & 0.0430 & 0.9913 \\
\bottomrule
\end{tabular}%
}
\end{subtable}%
\hfill 
\begin{subtable}[ht]{0.49\textwidth}
\centering
\caption{Performance comparison on the UCI Heart Disease dataset}
\label{tab:heart_disease}
\resizebox{\linewidth}{!}{%
\begin{tabular}{lcccc}
\toprule
\textbf{Model} & \textbf{Acc.} & \textbf{F1} & \textbf{ECE} & \textbf{AUROC} \\
\midrule
NN (OCE) Entropic    & 0.8556 & 0.8395 & \textbf{0.1089} & \textbf{0.9504} \\
NN (OCE) CVaR        & \textbf{0.8778} & \textbf{0.8736} & 0.3721 & 0.9444 \\
NN (OCE) LCVaR       & 0.8333 & 0.8193 & 0.2873 & \textbf{0.9464} \\
NN (OCE) SCVaR       & 0.8556 & 0.8395 & 0.1244 & 0.9449 \\
NN (OCE) Monotone MV & \textbf{0.8667} & \textbf{0.8500} & 0.1192 & \textbf{0.9479} \\
NN (OCE) Quartic     & 0.8444 & 0.8250 & 0.2580 & \textbf{0.9464} \\
\midrule
NN BCE Loss          & 0.8556 & 0.8395 & \textbf{0.0980} & 0.9415 \\
Logistic Regression  & 0.8444 & 0.8250 & \textbf{0.1050} & 0.9439 \\
SVC (Baseline)       & 0.8333 & 0.8052 & \textbf{0.0989} & 0.9340 \\
Random Forest        & 0.8556 & 0.8395 & 0.1198 & 0.9293 \\
\bottomrule
\end{tabular}%
}
\end{subtable}

\end{table}

\paragraph{Conclusions.}From \Cref{tab:heart_disease}, we observe that the OCE models achieve higher accuracy, F1, and AUROC on the UCI Heart Disease dataset. From \Cref{tab:breast-cancer-dataset}, we observe that the OCE models achieve the lowest ECE on the Breast Cancer Detection dataset and exhibit comparable performance on the remaining metrics. 

\subsection{CVaR Optimization}
We use the same setup as in the portfolio optimization experiment, including the implementation details. This experiment analyzes the OCE variants associated with CVaR and compares them against a benchmark CVaR solver from the \textit{skfolio} library. Let the solver compute portfolio weights $\theta_\text{cvs}$ that minimize CVaR. In \Cref{fig:portfolio_plot_SnP}, we observe that the OCE variant for CVaR not only converges in value to the solver variant ($\octh{k} \to \octh{\text{cvs}}$), but also in the portfolio composition ($\theta_k \to \theta_\text{cvs}$). This provides empirical evidence of the convergence guarantees established in the paper for the OCE-SG algorithm. From \Cref{tab:portfolio_metrics_cvar_SnP}, we observe that the other variants, L-CVaR and S-CVaR, outperform the standard CVaR on some important metrics, indicating that these proposed variants could serve as viable alternatives to the standard CVaR. The \cref{tab:dataset_1,tab:dataset_2} provides a similar performance comparison for the FTSE dataset and the Nasdaq dataset, respectively.  

\begin{figure}[ht]
  \centering
  \begin{minipage}[c]{0.42\textwidth}
    \centering
    \includegraphics[width=\textwidth]{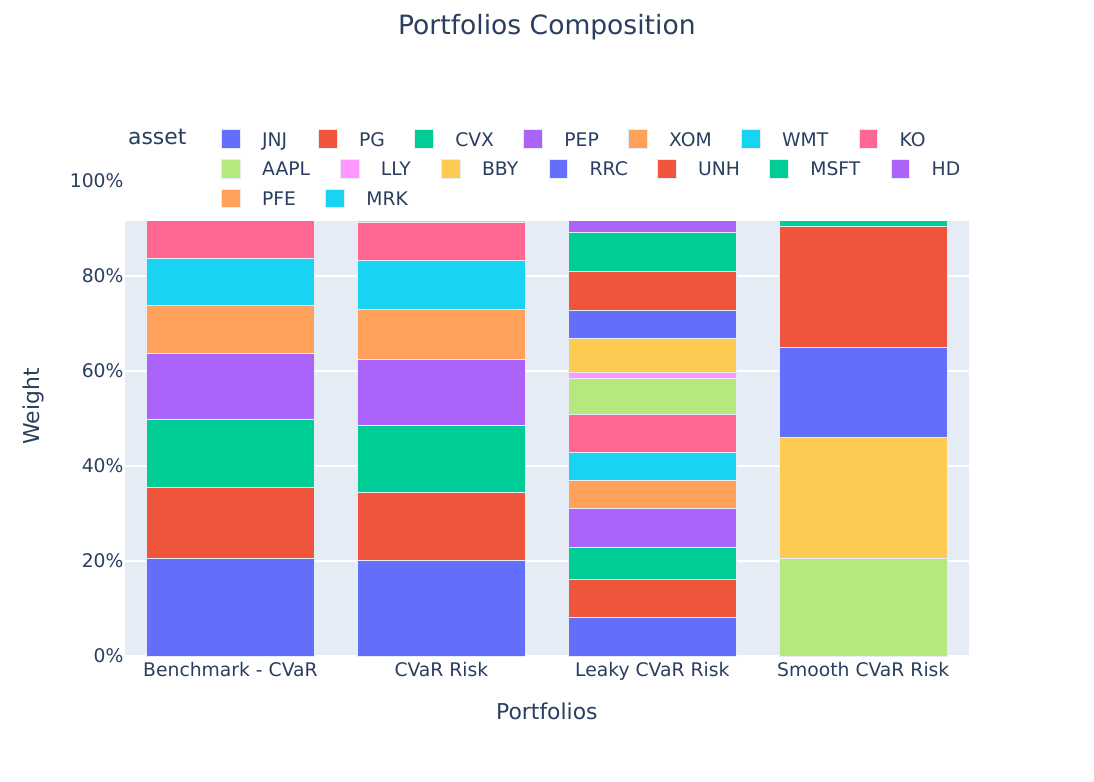}
    \captionof{figure}{Portfolio compositions of CVaR variants given by the OCE-SG algorithm, in comparison to the benchmark CVaR portfolio given by the \textit{skfolio} library.}
    \label{fig:portfolio_plot_SnP}
  \end{minipage}
  \hfill
  \begin{minipage}[c]{0.54\textwidth}
    \centering
    \small 
    \begin{tabular}{l *{4}{r}}
      \toprule
      \textbf{Portfolio} & \textbf{An. Ret} & \textbf{CVaR} & \textbf{CDaR} & \textbf{Sharpe} \\
      \midrule
      Benchmark - CVaR & 0.1250 & \textbf{0.0239} & 0.1561 & 0.7639 \\
      CVaR Risk        & 0.1261 & \textbf{0.0239} & 0.1544 & 0.7713 \\
      Leaky CVaR Risk  & 0.1619 & 0.0258 & \textbf{0.1350} & \textbf{0.9209} \\
      Smooth CVaR Risk & \textbf{0.2135} & 0.0348 & 0.2681 & 0.8466 \\
      \bottomrule
    \end{tabular}
    \captionof{table}{Comparison of CVaR variants against a benchmark solution on 'S\&P' dataset, and evaluated on popular risk metrics}
    \label{tab:portfolio_metrics_cvar_SnP}
  \end{minipage}
\end{figure}

\begin{table}[ht]
  \centering
  \begin{minipage}[c]{0.45\textwidth}
    \centering
    \small 
    \begin{tabular}{l *{4}{r}}
      \toprule
      \textbf{Portfolio} & \textbf{An. Ret} & \textbf{CVaR} & \textbf{CDaR} & \textbf{Sharpe} \\
      \midrule
      Benchmark - CVaR & 0.0987 & \textbf{0.0229} & 0.1870 & 0.6378 \\
      CVaR Risk        & 0.0991 & \textbf{0.0229} & \textbf{0.1858} & \textbf{0.6407} \\
      Leaky CVaR Risk  & 0.0980 & 0.0244 & 0.1950 & 0.5993 \\
      Smooth CVaR Risk & \textbf{0.1284} & 0.0305 & 0.2938 & 0.6101 \\
      \bottomrule
    \end{tabular}
    \caption{Comparison of CVaR variants against a benchmark solution on 'FTSE' dataset, and evaluated on popular risk metrics}
    \label{tab:dataset_1}
  \end{minipage}
  \hfill
  \begin{minipage}[c]{0.45\textwidth}
    \centering
    \small
    \begin{tabular}{l *{4}{r}}
      \toprule
      \textbf{Portfolio} & \textbf{An. Ret} & \textbf{CVaR} & \textbf{CDaR} & \textbf{Sharpe} \\
      \midrule
      Benchmark - CVaR & -0.0283 & \textbf{0.0190} & 0.2051 & -0.2000 \\
      CVaR Risk        & \textbf{-0.0281} & \textbf{0.0190} & \textbf{0.2016} & \textbf{-0.1988} \\
      Leaky CVaR Risk  & -0.2168 & 0.0553 & 0.4734 & -0.5263 \\
      Smooth CVaR Risk & -0.3619 & 0.0686 & 0.6368 & -0.6869 \\
      \bottomrule
    \end{tabular}
    \caption{Comparison of CVaR variants against a benchmark solution on 'NASDAQ' dataset, and evaluated on popular risk metrics}
    \label{tab:dataset_2}
  \end{minipage}
\end{table}

\clearpage
\subsection{Uncertainty Quantification}
Deep learning models have achieved superior performance on ML prediction problems, but they often fail to generalize well to out-of-distribution (OOD) data. A crucial aspect of any machine learning (ML) application is understanding what the model does not know. Uncertainty quantification tackles this problem by quantifying the model's uncertainty. A popular method for capturing a model's uncertainty is to use a mean-variance estimation (MVE) network \citep{nix_estimating_1994,mean-variance-estimation-network}, which serves as a building block for many uncertainty quantification algorithms in a supervised learning setting. 

\paragraph{Setup.}The MVE framework assumes that every pair of input $z=\langle x,y\rangle$ comes from a Gaussian distribution: $y \sim \mathcal{N}(\mu(x),\sigma(x))$, and the model learns to outputs mean and variance estimates for the given input, by minimizing the negative likelihood loss (NLL) of a Gaussian distribution. Precisely, it solves the problem : find $\theta_* = \argmin_{\theta \in \Theta} \Exp\left[F(\theta,\xi)\right]$, where
\begin{equation*}
    F(\theta, z) = F(\theta, \langle x,y\rangle) \triangleq \frac{1}{2}\log\left(\sigma^2_\theta(x)\right) + \frac{1}{2}\frac{(y-\mu_\theta(x))^2}{\sigma_\theta^2(x)}.
\end{equation*}

\paragraph{Implementation details.}In this experiment, we use a simple heterogeneous regression dataset. We use two benchmarks: the original MVE solution and the state-of-the-art Deep ensemble solution. The implementations of both the benchmarks and the aforementioned dataset are available in the \textit{lightning-uq-box} Python library. We modify the MVE's loss by replacing the expectation with OCE risk criterion $\ocet(\cdot)$, where $u$ is the exponential utility given in \Cref{example:entropic}. We run the OCE-SG for $1000$ epochs with a batch size of $175$ and a fixed learning rate of $0.03$. 

From \Cref{figure:uq-full-plot}, we observe that our approach significantly outperforms both benchmarks.
\begin{figure}[ht]
  \centering
  \begin{subfigure}{\columnwidth}
    \centering
    \includegraphics[width=\linewidth]{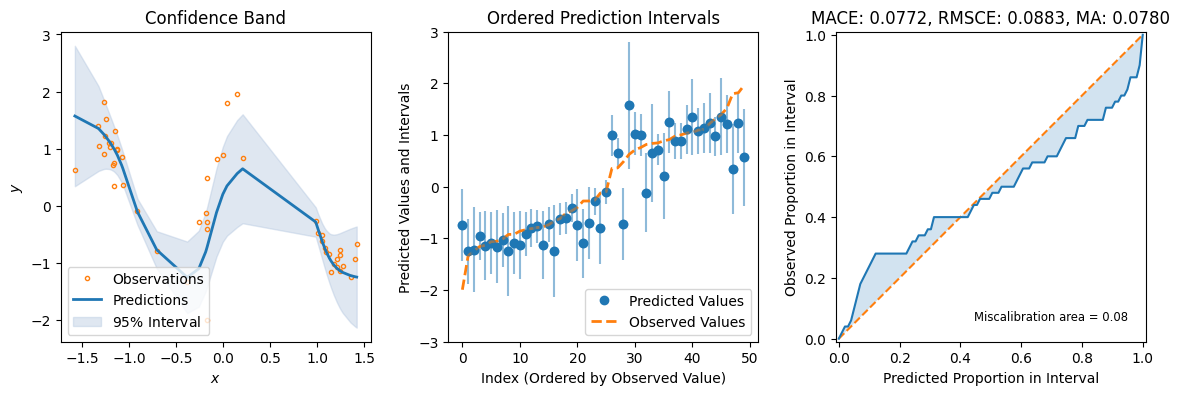}
    \caption{Deep Ensembles}
    \label{fig:perf1}
\end{subfigure}
\begin{subfigure}{\columnwidth}
    \centering
    \includegraphics[width=\linewidth]{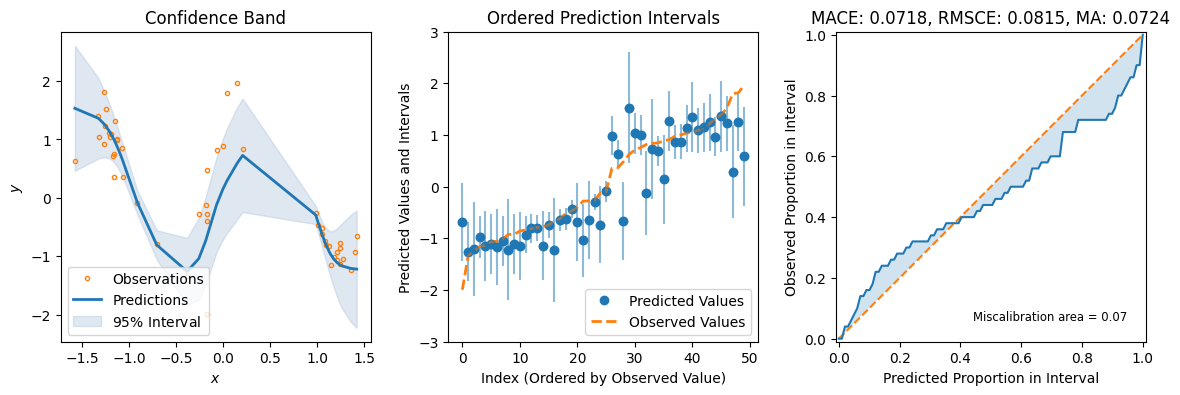}
    \caption{Mean Variance Estimation}
    \label{fig:perf2}
\end{subfigure}
\begin{subfigure}{\columnwidth}
    \centering
    \includegraphics[width=\linewidth]{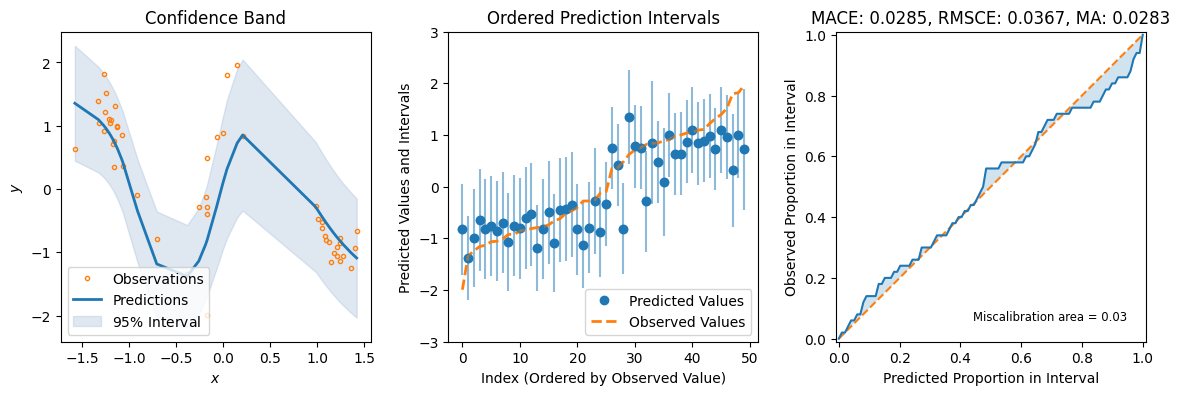}
    \caption{Mean Variance Estimation with OCE risk}
    \label{fig:perf3}
\end{subfigure}
  \caption{Performance of a Mean-Variance Estimation (MVE) model with OCE criterion against standard MVE and Deep Ensembles.}
  \label{figure:uq-full-plot}
\end{figure}

\clearpage
\subsection{Entropic Risk - Estimation and Optimization}

\paragraph{Entropic risk estimation.}
In this experiment, we assume $X \sim \mathcal{N}(\mu, \sigma^2)$ denotes losses, with mean $\mu=-1$ and variance $\sigma^2 = 4$. Under this assumption, the value of the entropic risk measure $\rho_e(X)$ for some $\beta>0$, is given by: 
\begin{equation}\label{eq:entropic-risk}
    \rho_e(X) = \frac{1}{\beta} \log\left(\Exp\left[e^{\beta X}\right] \right) = \mu + \frac{\beta \sigma^2}{2}.
\end{equation}

\begin{figure}[ht]
  \centering
  \begin{subfigure}[b]{0.48\linewidth}
    \centering
    \includegraphics[width=\linewidth]{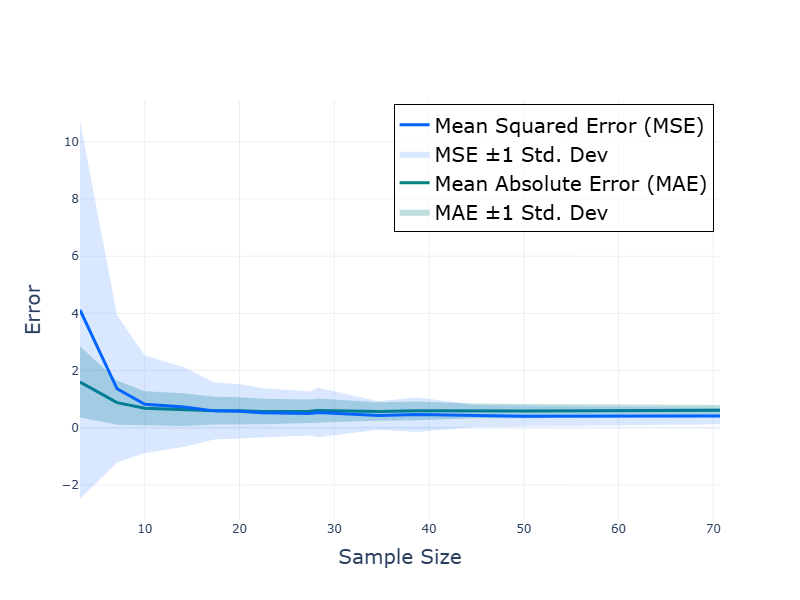} 
    \caption{Estimation errors for sample-based estimation of entropic risk using OCE.}
    \label{fig:entropic-estimation}
  \end{subfigure}
  \hfill
  \begin{subfigure}[b]{0.48\linewidth}
    \centering
    \includegraphics[width=\linewidth]{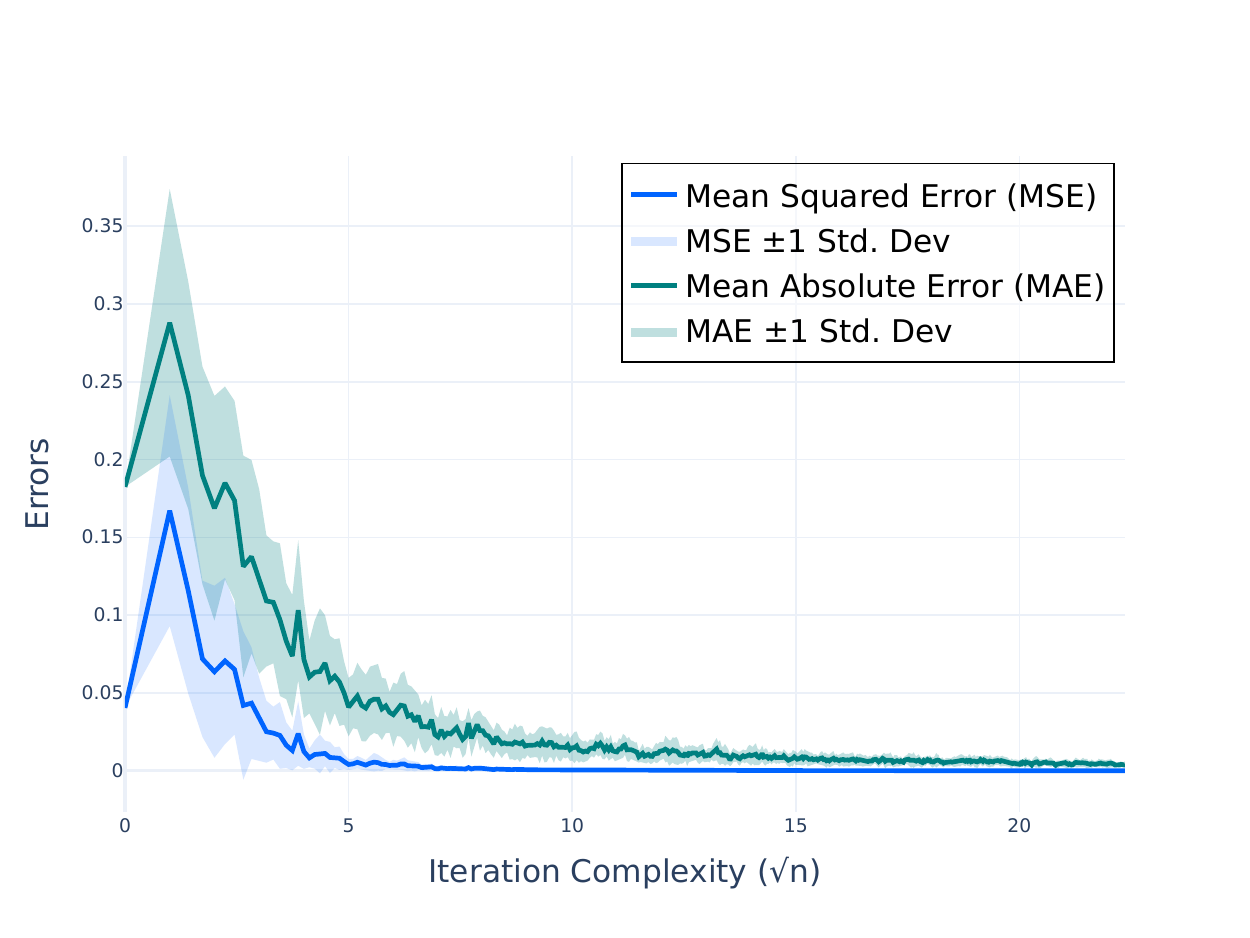} 
    \caption{Convergence of iterates of the OCE-SG algorithm for the entropic risk criterion.}
    \label{fig:entropic-optimization}
  \end{subfigure}
\end{figure}

We use \cref{eq:oce-as-entropic-risk} with $\beta=0.5$ in our experiment and employ \cref{alg:oce_estimation} to estimate $\oce$ using $m$ samples of $X$. The associated MAE and MSE bounds on the estimation error, for varying choices of sample size $m$ are given in \Cref{fig:entropic-estimation}. For each choice of $m$, we repeat the simulation $N=1000$ times and compute the error mean and its spread (standard error) by averaging across the $N$ simulations. The figure shows that our proposed estimator converges rapidly\footnote{Owing to fast convergence, we choose to plot these errors versus $\sqrt{m}$ instead of $m$, in order to make the error decrease discernible.}.

\paragraph{Entropic risk optimization.}
\label{subsection:entropic-risk-minimization}
We consider a portfolio optimization application (\Cref{example:portfolio-optimization}) with entropic risk as the objective.
In particular, we consider a $d$-dimensional random vector $\xi$ denoting a random vector of asset-wise market returns, which follows a multivariate normal distribution with mean $\mu$ and covariance matrix $\Sigma$, where $\Sigma$ is positive-definite. The decision space $\Theta$ is an $d$-dimensional simplex. Given $\theta \in \Theta$, we are interested in the problem of optimizing the quantity $\theta^T\xi$ over $\Theta$. A risk-neutral criterion would be to consider the average returns, i.e., maximize $\Exp\left[\theta^T\xi\right]$ over $\theta$, or equivalently, minimize $-\theta^T\xi$ over $\theta$. 

A popular risk-sensitive criterion is the mean-variance risk, defined as follows. Let $\beta>0$, then the \textit{mean-variance optimization} problem is posed as
\begin{equation}\label{eq:mean-variance-optimization}
    \text{find }\; \theta^* \triangleq \argmin_{\theta \in \Theta} \left[-\theta^T \mu + \frac{\beta}{2} \theta^T \Sigma \theta\right]. 
\end{equation}
\paragraph{}
For our experiments, we choose entropic risk as the optimization criterion and relate it to the mean-variance criterion. Consider the objective function defined as $F(\theta,\xi) \triangleq -\theta^T\xi$. Since $\xi \sim \mathcal{N}(\mu,\Sigma)$, we have $F(\theta,\xi) \sim \mathcal{N}(-\theta^T\mu, \theta^T\Sigma\mu), \forall \theta \in \Theta$. Replacing $X$ in (\ref{eq:entropic-risk}) with $F(\theta,\xi)$, we redefine the entropic risk as the function of $\theta$ as follows. Define $\rho_E:\Theta \to \Rel$, where $\rho_E(\theta) = \rho_e(F(\theta,\xi)), \forall \theta \in \Theta$, where $\rho_e$ is defined in (\ref{eq:entropic-risk}). Then, by (\ref{eq:entropic-risk}), it follows that for every $\theta \in \Theta$,
\begin{equation}\label{eq:entropic-risk-theta}
    \rho_E(\theta) = -\theta^T\mu + \frac{\beta \theta^T\Sigma\theta}{2}. 
\end{equation}
Comparing, (\ref{eq:mean-variance-optimization}) and (\ref{eq:entropic-risk-theta}), it is easy to see that $\theta^*$ is also the minimizer of $\rho_E(\cdot)$. 

\paragraph{Experiment setup.}In our setup, we set $d=5$. Using an arbitrary vector $\mu \in \Rel^d$ and arbitrary, positive-definite matrix $\Sigma \in \Rel^d \times \Rel^d$, we define $\xi \sim \mathcal{N}(\mu,\Sigma)$. The choices for $\mu,\Sigma$ are governed by a distribution underlying the \texttt{make\_spd\_matrix} function of \texttt{scikit-learn} python package.
To find $\theta^*$ in (\ref{eq:mean-variance-optimization}), we employed the convex optimization solver in the \texttt{pyportfolioopt} Python package. 

We evaluate the performance of the OCE-SG algorithm by choosing the utility function as per \Cref{example:entropic}, due to which, $\oc{F(\theta,\xi)}$ coincides with $\rho_E(\theta)$. We run Algorithm \ref{alg:oce-minimization} with $\alpha_k=1/\sqrt{k}$ and $m_k=k$ for $500$ iterations. We repeat this simulation $20$ times and \Cref{fig:entropic-optimization} shows the MAE and MSE errors on the iterates produced by the algorithm, averaged across the $20$ runs. From the figure, we observe that the iterates indeed converge to the optima, i.e., $\theta_k \to \theta^*$.


\end{document}